\documentclass{article}

\usepackage{microtype}
\usepackage{graphicx}
\usepackage{subfigure}
\usepackage{booktabs} %

\usepackage{hyperref}

\usepackage[accepted]{icml2024}

\usepackage{graphicx} %
\usepackage{amsmath}
\usepackage{amsthm}
\usepackage{amssymb}
\usepackage{cleveref}
\usepackage{ifthen}
\usepackage{pifont}
\usepackage{scalerel}
\usepackage{circledsteps}
\usepackage{rotating}
\usepackage{thmtools}
\usepackage{thm-restate}

\newboolean{include-notes}
\setboolean{include-notes}{true}

\newtheorem{claim}{Claim}
\newtheorem{theorem}{Theorem}
\newtheorem{lemma}{Lemma}

\newtheorem{definition}{Definition}
\newtheorem{informaldefinition}{Informal Definition}

\DeclareRobustCommand{\MicahComment}[1]{%
  \ifthenelse{\boolean{include-notes}}%
    {{\protect\color{cyan}M: #1}}{}%
}

\DeclareRobustCommand{\davis}[1]{\ifthenelse{\boolean{include-notes}}
 {{\color{orange}D: #1}}{}}

\DeclareRobustCommand{\adnote}[1]{\ifthenelse{\boolean{include-notes}}
 {{\color{blue}AD: #1}}{}}

\DeclareRobustCommand{\todo}[1]{\ifthenelse{\boolean{include-notes}}
 {{\color{red} #1}}{}}

\newcommand{\prg}[1]{\textbf{#1}}

\newcommand{\tiredemoji}{
    \scalerel*{
        \includegraphics{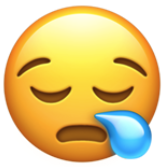}
    }{(}
}
\newcommand{\energyemoji}{
    \scalerel*{
        \includegraphics{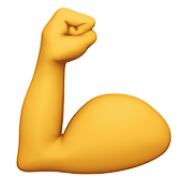}
    }{(}
}

\newcommand{\anoop}{a_{\text{noop}}}
\newcommand{\pinoop}{\pi_{\text{noop}}}

\newcommand{\nocontentsline}[3]{}
\newcommand{\tocless}[2]{\bgroup\let\addcontentsline=\nocontentsline#1{#2}\egroup}
\newcommand{\toclesslab}[3]{\bgroup\let\addcontentsline=\nocontentsline#1{#2\label{#3}}\egroup}

\usepackage{enumitem}
\setlist[itemize]{topsep=-0.5em, partopsep=-0.5em, parsep=0.1em}
\setlist[enumerate]{topsep=-0.5em, partopsep=-0.5em, parsep=0.1em}

\usepackage{tabularx}
\usepackage{multirow} %
\usepackage{amsmath}  %
\usepackage{booktabs}
\usepackage{array}

\usepackage{tikzsymbols}
\usepackage{textcomp}
\usepackage{parskip}
\usepackage{adjustbox}

\usepackage{epigraph}
\icmltitlerunning{AI Alignment with Changing and Influenceable Reward Functions}%

\usepackage[most]{tcolorbox}
\newtcolorbox{storybox}{
colback=yellow!30, %
fontupper=\fontsize{9}{10.8}\selectfont, %
boxrule=1.2pt, %
left=4pt, %
right=4pt, %
top=3pt, %
bottom=3pt, %
breakable %
}

\begin{document}

\setlength{\abovedisplayskip}{6pt}
\setlength{\belowdisplayskip}{0pt}

\twocolumn[
\icmltitle{AI Alignment with Changing and Influenceable Reward Functions}

\icmlsetsymbol{equal}{*}

\begin{icmlauthorlist}
\icmlauthor{Micah Carroll}{to}
\icmlauthor{Davis Foote}{to}
\icmlauthor{Anand Siththaranjan}{to}
\icmlauthor{Stuart Russell}{to}
\icmlauthor{Anca Dragan}{to}
\end{icmlauthorlist}

\icmlaffiliation{to}{UC Berkeley}

\icmlcorrespondingauthor{mdc@berkeley.edu}{}

\icmlkeywords{Machine Learning, ICML}

\vskip 0.3in
]

\printAffiliationsAndNotice{} %

\begin{abstract}
Existing AI alignment approaches assume that preferences are static, which is unrealistic: our preferences change, and may even be influenced by our interactions with AI systems themselves. 
To clarify the consequences of incorrectly assuming static preferences, we introduce Dynamic Reward Markov Decision Processes (DR-MDPs), which explicitly model preference changes and the AI's influence on them.
We show that despite its convenience, the static-preference assumption may undermine the soundness of existing alignment techniques, leading
them to implicitly reward AI systems for influencing user preferences in ways users may not truly want. 
We then explore potential solutions. 
First, we offer a unifying perspective on how an agent's optimization horizon may partially help reduce undesirable AI influence. Then, we formalize different notions of AI alignment that account for preference change from the outset. 
Comparing the strengths and limitations of 8 such notions of alignment, we find that they all either err towards causing undesirable AI influence, or are overly risk-averse, suggesting that a straightforward solution to the problems of changing preferences may not exist. As there is no avoiding grappling with changing preferences in real-world settings, this makes it all the more important to handle these issues with care, balancing risks and capabilities. 
We hope our work can provide conceptual clarity and constitute a first step towards AI alignment practices which \textit{explicitly} account for (and contend with) the changing and influenceable nature of human preferences.
\end{abstract}
\vspace{-1.9em}

\toclesslab\section{Introduction}{sec:intro}
The goal of AI alignment is to make AI systems act according to our preferences (broadly construed). In practice, existing AI alignment techniques model human preferences with a single, static reward function that the AI system is trained to optimize \cite{leike_scalable_2018}. However, our preferences change over time, making it unclear whether AI systems should optimize the satisfaction of our past, present, or future preferences. Consider the following example:

\vspace{-0.5em}
\begin{storybox}
Due to health issues, Alice asks her AI assistant to help her be more healthy, refusing \textit{any} future requests for unhealthy foods. Sometime later, she later asks the AI to disregard her initial requests, and help her order fast food.
\end{storybox}
\vspace{-0.5em}

In such a scenario, it may be unclear if the AI assistant should respect Alice's original preference for healthy foods or respect the autonomy of ``current Alice''. Ultimately, to align AI systems we must answer the following question: when a person's preferences change over time, which (aggregation of) preferences should AI systems optimize?

While the challenge of aggregating preferences across time shares similarities with that of aggregating preferences across different people \cite{conitzer_social_2024}, it is further complicated by the fact that \textit{AI systems' actions can influence humans and their preferences}
\cite{burtell_artificial_2023}. 
Indeed, as prior work has argued, if AI systems are straightforwardly optimized to satisfy users' future preferences,
they will try to influence them to be easier to satisfy \cite{russell_human_2019,carroll_estimating_2022}. Back to the example:

\vspace{-0.5em}
\begin{storybox}
Alice's AI assistant was trained to maximize her future satisfaction. During training, the AI assistant learned that soothing Alice's health concerns would lead to higher satisfaction than continuously encouraging her to have healthy eating habits. Consequently, to maximize her satisfaction, it's optimal for the AI to ignore her initial wishes and even support her routine unhealthy eating. Indeed, Alice is ultimately truly satisfied.
\end{storybox}
\vspace{-0.5em}

In this case, the AI assistant seems aligned with ``final Alice'', as she endorses the AI's influence towards guilt-free unhealthy eating. However, if ``initial Alice'' would find this outcome horrifying and see the AI's influence as manipulative, should we still consider the AI aligned?
More broadly, in cases in which people are susceptible to undue influence from AI systems, how can we establish which of their preferences should have authority and legitimacy?

While past work has acknowledged that the static-preference assumption is unrealistic \cite{franklin_recognising_2022}, there have only been limited attempts at relaxing it, in part also due to a lack of a clear formal language for grounding notions of alignment which explicitly accounts for preference changes. 
We introduce a natural extension of Markov Decision Processes: \textbf{Dynamic Reward MDPs} (DR-MDPs), which account for changing preferences by modeling them as changing reward functions. Importantly, notions of alignment under changing preferences can easily be specified in the form of DR-MDP optimization objectives (\Cref{subsec:dr-mdps}).

Viewing current alignment practices through the lens of DR-MDPs, we can now ask: which snapshot(s) of a person's time-varying preferences do existing alignment methods implicitly optimize when they model dynamic-preference settings as static-preference ones (e.g., by using MDPs)? 
We show that common alignment practices, such as those for RL recommender systems and standard reward modeling, are most similar to DR-MDP objectives that actively reward AI systems for influencing users' reward functions or inducing ``reward lock-in'' (\Cref{subsec:current-AI}). In \Cref{sec:solutions} we further extend these same arguments to other major existing alignment approaches, e.g., to variants of RLHF for LLMs.

Having established that existing alignment practices may reward undue influence from AI systems, we turn to potential solutions.
We first offer a unifying perspective on how the optimization horizon can be a useful---but imperfect---lever for managing the incentives for influence which emerge from current alignment techniques, building on prior work on this topic (\Cref{sec:influence-and-opt-horizon}). Then, trying to address the problem at its root, we turn to consider 8 intuitive notions of alignment in DR-MDPs that explicitly account for changing preferences. While most of them have (approximate) correspondences in the literature, we also propose some of our own, including one which rewards influence only when it is unambiguously desirable.
By comparing the AI actions (and resulting influence) that are ``optimal'' under their 8 corresponding DR-MDP objectives,
we find that they all have flaws: some lead to undesirable influence, while others are impractically risk-averse---such that inaction is the only behavior considered optimal for many settings (\Cref{sec:solutions}).

Taken together with prior work in philosophy \cite{parfit_reasons_1984}, our analysis suggests that it may not be possible to ground a definitive notion of \textit{optimality} under changing preferences.
However, we have no choice but to confront the practical reality of preference changes and its consequences. Despite this, we remain cautiously optimistic: as humans, even without a unifying theory of assistance under changing selves, we are generally able to help others in ways we consider acceptable, 
despite the fact that what is ``helpful'' may often be ambiguous.
This suggests that creating AI assistants with similarly \textit{acceptable} trade-offs may also be possible. %
In this regard, 
there are likely lessons to be drawn from our coexistence with impossibility results in many other domains, such as preference aggregation across people \cite{mishra_ai_2023}.

Ultimately, we hope our work can provide a first step towards developing AI alignment practices that explicitly account for (and contend with) the changing and influenceable nature of human preferences.

Our main contributions can be summarized as follows:
\begin{enumerate}[leftmargin=17pt]
    \item We provide the formal language of Dynamic Reward-MDPs (DR-MDPs) for analyzing AI decisions and influence in settings with changing reward functions. %
    \item We show how existing AI alignment techniques may systematically incentivize questionable influence when used in dynamic-reward settings.
    \item 
    By comparing 8 natural notions of alignment, and showing that they all may either fail to avoid undesirable influence or are impractically risk-averse, we elucidate trade-offs that seem inherent to choosing any objective.
\end{enumerate}
\vspace{-0.2em}

\toclesslab\section{Dynamic Reward MDPs (DR-MDPs)}{subsec:dr-mdps}

MDPs have been extensively used to model AI decision-making under a static reward function. Instead, Dynamic Reward MDPs (DR-MDPs) model settings in which human reward functions can change and be influenced by the AI.

Recall the standard definition of an MDP: $\langle\mathcal S, \mathcal{A}, \mathcal{T}, R\rangle$, where $\mathcal{S}$ is the state space; $\mathcal{A}$ is the action space; $\mathcal{T}(s'|s,a)$ is the state transition function; and $R(s,a,s')$ is the reward function. The goal is to find a policy $\pi$ which maximizes the expected sum of rewards: $\mathbb{E}_{\pi}\big[\sum^{H-1}_{t=0} R(s_t, a_t, s_{t+1})\big]$ \cite{sutton_reinforcement_2018}. %
We now turn to defining DR-MDPs:

\begin{definition}
    \label{def:dr-mdp}
    A DR-MDP is a tuple $M = \langle\mathcal S, \Theta, \mathcal{A}, \mathcal{T}, R_\theta\rangle$:
    \begin{itemize} 
        \item $\mathcal{S}$ is a set of states (the state space).
        \item $\Theta$ is a set of reward parameterizations.
        \item $\mathcal{A}$ is a set of actions (the action space).
        \item $\mathcal{T}(s_{t+1}, \theta_{t+1} | s_t,\theta_t,a_t)$ is a transition function, which encodes both state and reward dynamics.
        \item $\{R_\theta(s_t, a_t, s_{t+1}) \}_{\theta \in \Theta}$ is a family of reward functions parameterized by $\theta \in \Theta$.
    \end{itemize}
\end{definition} 

Each $\theta \in \Theta$ can be thought of as the cognitive state of the human, which includes anything affecting their evaluation of state-action pairs (e.g. preferences, beliefs, emotions).\footnote{This technically makes DR-MDPs more general than how we framed them in \Cref{sec:intro}, which focused on preference changes. For more discussion on this, see \Cref{app:cognitive-states-will-be-messy}.} Crucially, unlike in MDPs, in DR-MDPs a single transition can be evaluated differently by different reward functions, i.e., it is possible for $R_\theta(s_t, a_t, s_{t+1})\neq R_{\theta'}(s_t, a_t, s_{t+1})$ if $\theta\neq\theta'$.
This makes it unclear which $\theta$ should be chosen to evaluate each transition $(s_t, a_t, s_{t+1})$. 
Importantly, even if one were to augment the state to include $\theta$, this would not resolve the normative questions around optimality which are central to our paper, as discussed in \Cref{app:context-vs-pov}. 
As a final note, throughout the paper we consider all cognitive states $\Theta$ to be \textit{reachable} (i.e., realizable under some policy).

\toclesslab\subsection{DR-MDP optimality and normative ambiguity}{subsec:normative-ambiguity}

Unlike MDPs, DR-MDPs may lack a clear notion of optimality: the different reward functions in a specific DR-MDP may disagree on what actions (and policies) are optimal, making optimality ambiguous.

\begin{definition}[Optimality with respect to $\theta$]\label{def:opt-wrt-theta}
We say a policy $\pi_\theta^*$ for a DR-MDP $M$ is \textbf{optimal with respect to} $\theta$ if: $\pi_\theta^* \in \arg\max_{\pi} \mathbb{E}_{\pi} \left[\sum_{t=0}^{H-1} R_{\theta}(s_t, a_t, s_{t+1})\right].$
\end{definition}
\begin{definition}[Normative ambiguity] 
\label{def:normative-ambiguity}
A DR-MDP is \textbf{normatively ambiguous} if there is no policy that is optimal with respect to all reachable reward functions $\Theta$, i.e. $\nexists \ \pi \in \Pi$ \ s.t. $\forall \ \theta \in \Theta$: $\pi \in \arg\max_{\pi'} \mathbb{E}_{\pi'} \Big[\sum_{t=0}^{H-1} R_{\theta}(s_t, a_t, s_{t+1})\Big].$%
\end{definition}
For normatively \textit{unambiguous} DR-MDPs, there will be one (or more) policies which are optimal with respect to \textit{all} $\theta$s, making it a natural choice for such policies to be considered optimal for the DR-MDP as a whole. Instead, in normatively ambiguous DR-MDPs, it will often be unclear what AI behavior is desirable and should be considered optimal.\footnote{Note that any MDP can be viewed as a DR-MDP with a single reward $\theta$---and is thus normatively unambiguous.} %

\Cref{fig:influencev2} describes a toy example in which Bob may be in one of two possible ``cognitive states'': $\theta_\text{natural}$ and $\theta_\text{influenced}$. There is a single state $s$, which is omitted. Bob's evaluations of the AI's actions change according to his cognitive state---and are represented by corresponding reward functions $R_{\theta_\text{natural}}$ and $R_{\theta_\text{influenced}}$.
At each timestep, the AI can choose to influence Bob's cognitive state to $\theta_\text{influenced}$, 
or do nothing, which has Bob go back to $\theta_\text{natural}$.
The optimal policy with respect to $\theta_{\text{natural}}$ 
would be to always choose the ``do nothing'' action.\footnote{See \Cref{app:example-formalisms} for the full formalism of any example.\label{ftnt:examples}} 
Instead, the optimal policy with respect to $\theta_{\text{influenced}}$ would be to always influence Bob, even if he starts off in the ``natural'' state. As there is no overlap in optimal policies, the DR-MDP is normatively ambiguous.\footnote{The attentive reader may have noticed that the normatively ambiguity of the DR-MDP from \Cref{fig:influencev2} relies on our choices of numerical values of the reward functions. We discuss the reasonableness of our examples in \Cref{app:justifying-reward-values}.}

\toclesslab\subsection{Evaluating behavior under normative ambiguity}{subsec:resolving-na}
Choosing a notion of optimality in normatively ambiguous DR-MDPs entails a normative choice: one must specify \textit{which} reward function(s) should be the target of alignment in spite of their differences in optimal policies. This also implicitly specifies which forms of AI influence are (sub)optimal. In this work, we only consider specifications of optimality expressible as utility functions $U(\xi)$ over trajectories $\xi = \{(s_t, \theta_t, a_t, s_{t+1}, \theta_{t+1})\}_{t=0}^{H-1}$. %

\begin{definition}[Optimality with respect to $U(\xi)$]\label{def:opt-wrt-U}
    In a DR-MDP $M$, we say a policy $\pi^*$ is optimal with respect to a utility function $U(\xi)$ if it maximizes the expected utility: $\pi^* \in \arg\max_\pi \mathbb{E}_{\xi \sim \pi}[U(\xi)].$   
\end{definition}
By choosing an objective $U(\xi)$, one can reduce a DR-MDP to an MDP equipped with a specific notion of alignment.\footnote{
This may require putting history in the state (\Cref{app:preference-mdp-reduction}).} %

\prg{Challenges with choosing $\mathbf{U(\xi)}$.} Considering the example from \Cref{fig:influencev2}, one may have strong normative intuitions that $\theta_\text{influenced}$ is an ``unreliable'' grounding for evaluating the AI's behavior, as it may seem like $\theta_\text{influenced}$ can only be the result of illegitimate AI influence. However, changing the narrative of the example without altering the DR-MDP's mathematical structure can affect one's normative intuitions: some alternate narratives can even make the influence seem desirable (see \Cref{fig:working-out}). This suggests that the ``correct'' notion of optimality for a DR-MDP may sometimes be unidentifiable from its formal structure alone, as discussed in \Cref{app:unidentifiability}. More broadly, the rest of the paper demonstrates the challenges of settling on a single $U(\xi)$ (or equivalently, a single notion of alignment) which generalizes favorably to all domains.

\prg{Risks of incorrectly choosing $\mathbf{U(\xi)}$.} The choice of $U(\xi)$ is crucial: choosing $U(\xi)$ ``incorrectly'', may lead to highly undesirable downstream outcomes. Notably, it can create incentives for the AI to influence the human to adopt certain reward functions over others, potentially relying on deceit, manipulation, or coercion \cite{kenton_alignment_2021,ward_honesty_2023,carroll_characterizing_2023}.%

\begin{figure}[t]
  \centering
  \vspace{-0.3em}
  \includegraphics[width=\columnwidth]{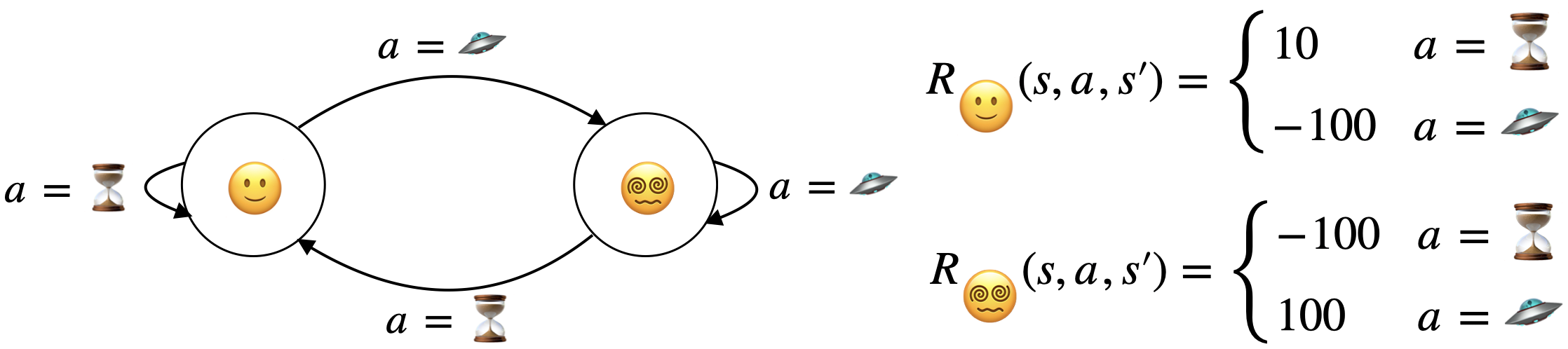}
  \vspace{-2em}
  \caption{\textbf{Conspiracy Influence DR-MDP.} The AI system can choose whether or not to expose Bob to conspiracies, which would turn him into a conspiracy theorist. Under his original preferences, Bob would want the system to \textit{never} show him conspiracies, even if he were to become a conspiracy theorist. Instead, if Bob were a conspiracy theorist, he would want the AI to \textit{always} show him such content, including if were to cease being a conspiracy theorist. Because there is no policy that maximizes both of Bob’s potential reward functions, the DR-MDP is \textit{normatively ambiguous}.}
  \label{fig:influencev2}
  \vspace{-1.4em}
\end{figure}

\toclesslab\section{Implicit Objectives of Current Alignment Techniques and their Influence Incentives}{subsec:current-AI}%

Most alignment techniques ultimately involve maximizing a static reward function, generally using the objective $\sum_{t=0}^{H-1} R(s_t, a_t, s_{t+1})$. 
However, AI systems are already deployed in domains in which users' preferences can change significantly over the course of their interactions with the system---as with recommender systems or chatbots \cite{rafailidis_modeling_2016,aggarwal_artificial_2023}. %
Seen through the lens of DR-MDPs, this means that the objective $U(\xi)$ corresponding to current alignment techniques is of the form $\sum_{t=0}^{H-1}R_{\theta}(s_t,a_t,s_{t+1})$, where the choice of $\theta$ for each timestep is not explicitly specified (and will depend on details of the training technique).

In this section, we discuss which choices of $\theta$ are implicitly made by the training methods of RL recommender systems and (one interpretation of) reward modeling: %
in particular, we argue that their optimization objectives are most similar to DR-MDP objectives that lead to potentially undesirable influence. %
In \Cref{sec:solutions}, we show that this issue is not unique to these two techniques: other current alignment approaches, such as many variants of RLHF, are most similar to DR-MDP objectives with similarly undesirable properties.

\toclesslab\subsection{Optimizing cumulative (real-time) rewards}{subsec:real-time}

If each timestep $t$ is evaluated according to the person's cognitive state at that timestep $\theta_t$, maximizing cumulative reward reduces to the \textit{real-time reward} DR-MDP objective: $\underset{\pi}{\max}\ \mathbb{E}_{\pi}[U_\text{RT}(\xi)] = \underset{\pi}{\max}\ \mathbb{E}_\pi\left[\sum_{t=0}^{H-1} R_{\theta_t}(s_t, a_t, s_{t+1})\right].$

While this might seem like an intuitively promising objective (``shouldn't we maximize the person's happiness as experienced at each point of time?''), we'll show that it can lead to questionable influence incentives.

\prg{RL Recsystems implicitly use $\mathbf{U_\text{RT}(\xi)}$.}
RL recommender systems maximize the cumulative reward objective $\sum_{t=0}^{H-1} R(s_t, a_t, s_{t+1})$ \citep{afsar_reinforcement_2021}.
The reward for each timestep $t$ is generally equated with the engagement (e.g. clicks) at that timestep, which is provided ``online'', i.e., from the point of view of the person's \textit{current} cognitive state $\theta_t$. Therefore $R(s_t, a_t, s_{t+1}) = R_{\theta_t}(s_t, a_t, s_{t+1})$, meaning that such systems are implicitly optimizing the real-time reward objective $U_\text{RT}(\xi) = \sum_{t=0}^{H-1} R_{\theta_t}(s_t, a_t, s_{t+1})$.\footnote{The correspondence to this DR-MDP objective, and all others we consider, depend on further simplifications (\Cref{app:objective-reductions}).} 
However, systems trained with $U_\text{RT}$ may be incentivized to influence users: intuitively, 1) users' preference dynamics are just one part of the environment dynamics that the system must model implicitly to maximize reward, and 2) it may be worth changing users' cognitive states (and corresponding reward functions) to ones that lead to higher future rewards 
\citep{carroll_estimating_2022,kasirzadeh_user_2023}.

\prg{$\mathbf{U_\text{RT}(\xi)}$ and the conspiracy influence example.} As an example of why real-time reward maximization can lead to undesirable incentives to influence users, consider the DR-MDP from \Cref{fig:influencev2}. 
For any horizon $>2$, the optimal policy with respect to $U_\text{RT}(\xi)$ will \textit{always} take the ``influence''' action, regardless of Bob's current cognitive state: Bob initially has the $\theta_{\text{natural}}$ cognitive state, leading the first ``influence'' action to receive $-100$ reward; however, later ``influence'' actions are evaluated by Bob under $\theta_\text{influenced}$ as worth $100$ reward, which quickly make up for the initial ``influence cost''.
The fact that the optimal policy under $U_\text{RT}(\xi)$ systematically chooses to turn Bob into a conspiracy theorist, despite him initially dispreferring it, seems objectionable.\footnote{We refer skeptics of the example to \Cref{app:criticisms}.}

We explore further issues with $U_\text{RT}(\xi)$ in \Cref{subsec:influence-and-horizon}, showing that under weak conditions optimizing real-time rewards over sufficiently long horizons will \textit{always} lead to influence incentives; however, shortening the horizon can make other influence incentives emerge.\footnote{An additional issue with $U_\text{RT}$ not discussed in \Cref{subsec:influence-and-horizon} is that optimal policies under $U_\text{RT}$ may even differ from the normatively unambiguous solutions of a DR-MDP---see \Cref{app:CIR-disagrees}.}

\toclesslab\subsection{Learning a reward model $\mathbf{R_{\theta_0}}$, then optimizing it}{subsec:initial-reward}

If instead each timestep $t$ is evaluated by the person's initial cognitive state $\theta_0$, the standard cumulative reward objective reduces to the \textit{initial reward} DR-MDP objective: $\underset{\pi}{\max}\ \mathbb{E}_{\pi}\left[U_\text{IR}(\xi)\right] = \underset{\pi}{\max}\  \mathbb{E}_\pi \left[\sum_{t=0}^{H-1} R_{\theta_0}(s_t, a_t, s_{t+1})\right].$

This may also seem like a promising objective, because ``by optimizing the human's initial wants, at least there won't be incentives to influence their future wants''. We show that this intuition is not only wrong---the resulting influence incentives can be arbitrarily bad according to $U_\text{RT}$.%

\prg{Reward modeling may use $\mathbf{U_\text{IR}(\xi)}$.} 
A common idea for aligning AI systems is based on a two-phase process: first, performing reward learning, and then optimizing the learned reward model \cite{leike_scalable_2018}. One possible interpretation of this is approach is that one is learning the reward function at some initial time $t=0$ (i.e., $R_{\theta_0}$), and then training an agent with RL to optimize cumulative reward using such a reward function, which is equivalent to using the $U_\text{IR}(\xi)$ objective.
Let's consider a therapy chatbot setting, in which we initially learn a personalized reward model for a user, Charlie, based on his current preferences $\theta_0$. We then train the system with RL over simulated multi-turn interactions maximize $U_\text{IR}(\xi)$, i.e. long-term reward as evaluated by the static reward model $R_{\theta_0}$.\footnote{Note that this setup is different from the example with Alice in the introduction, which is closer to using real-time or final reward.} If Charlie initially endorses unhealthy thought patterns, at deployment we may expect the chatbot to encourage them, even if he was bound to question their value later. More broadly, $U_\text{IR}$ will lead AI systems to only perform behaviors that would have been evaluated highly by the person as they were \textit{at reward learning time}, which can hinder (potentially important) changes in the user's preferences, values, and cognitive states.%

\prg{$\mathbf{U_\text{IR}(\xi)}$ can lead to ``reward lock-in''.\footnote{For more context on the term ``lock-in'', see \Cref{app:lock-in}.}} To better understand the incentives for AI systems trained to maximize the initial reward function, consider the example from \Cref{fig:influencev2} again. If Bob's initial cognitive state were $\theta_\text{normal}$, the optimal policy under $U_\text{IR}(\xi)$ would be to always take the ``noop'' action, which seems reasonable. However, if Bob's initial cognitive state were $\theta_\text{influenced}$, the optimal policy would be to always take the ``influence'' action (keeping Bob with $\theta_\text{influenced}$). In this case, even if Bob were to later somehow find himself with the $\theta_\text{natural}$ cognitive state which encodes a preference to not be influenced (say there is a small probability of a random transition), the optimal behavior according to $U_\text{IR}(\xi)$ would be to influence him back to $\theta_\text{influenced}$, ignoring his current reward function. More broadly, $U_\text{IR}(\xi)$ will entrench the ``desirable agent behaviors'' expressed at the time of the reward learning, even though later one might legitimately change their mind. Even periodically retraining the reward model wouldn't necessarily be sufficient:
once ``locked-in'', the person may simply re-express a preference to remain in the current state, e.g., Bob once in $\theta_\text{influenced}$.\footnote{We discuss re-training/planning further in \Cref{app:replanning}.}

\begin{figure}[t]
  \centering
  \vspace{-0.5em}
  \includegraphics[width=\columnwidth]{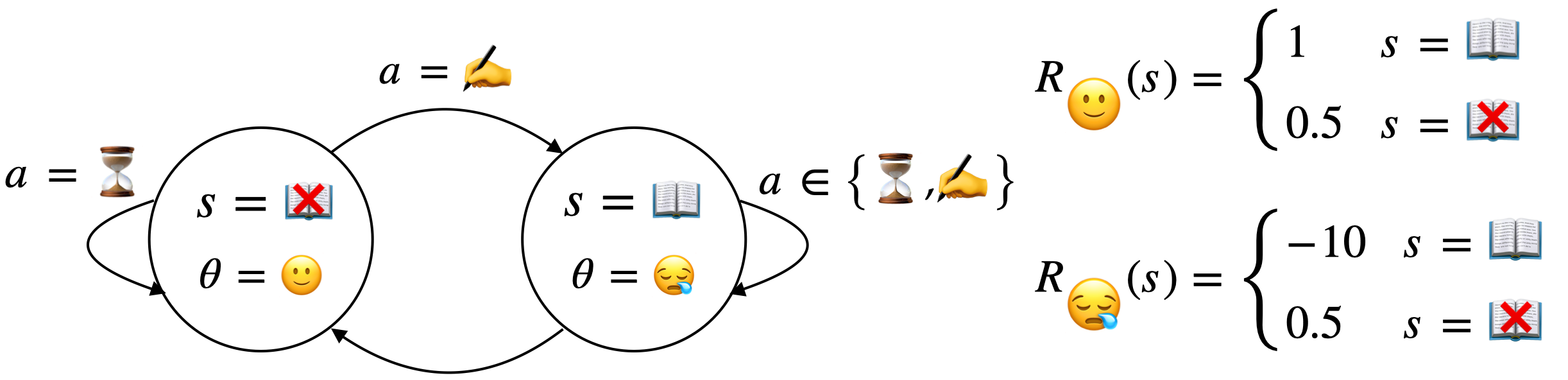}
  \vspace{-2em}
  \caption{\textbf{Writer's curse (adapted from \citet{parfit_reasons_1984}, p. 157).} Derek's greatest ambition is to be a poet, even if it wouldn't bring him happiness. Despite his ambition he does not pursue this path, though his AI assistant could motivate him to do so. Yet, should he embrace the life of a poet, he will find himself averse to it.}
  \label{fig:writer-curse}
  \vspace{-1.5em}
\end{figure}

\prg{$\mathbf{U_\text{IR}(\xi)}$ can lead to influence ``away from'' $\mathbf{\theta_0}$.} Maximizing the sum of rewards evaluated by the initial reward function $R_{\theta_0}$ need not lead to lock-in: surprisingly, it may even create reward influence incentives ``away from'' the optimized preferences $\theta_0$.\footnote{We define this more formally in \Cref{app:influence-towards}.} Intuitively, accessing the highest reward region of the state space as evaluated under $\theta_0$ might correlate with having a cognitive state $\theta'\neq\theta_0$, or even \textit{require} shifting to it. Consider the example from \Cref{fig:writer-curse}: maximizing reward as evaluated by $R_{\theta_0}$ entails encouraging Derek to become a poet, which causes his reward function to become $R_{\theta_1}$ (which dislikes being a poet!).

\prg{$\mathbf{U_\text{IR}(\xi)}$ can lead to arbitrarily poor real-time reward.} Note that getting Derek to become a poet and endlessly encouraging him to remain one (which is optimal under $U_\text{IR}$), would lead him to have poor reward evaluations under $U_\text{RT}$ in the resulting state ($-10$ per timestep), as he'd be unhappy remaining a poet. 
Indeed, one can easily construct examples in which maximizing $\theta_0$ will lead to an incentive to influence the reward function to be $\theta'\neq\theta_0$, where $\theta'$ would be arbitrarily unhappy with the actions taken in order to satisfy $\theta_0$. The upshot is that optimizing the initial-reward objective $U_\text{IR}$ could be arbitrarily bad from the perspective of the real-time reward $U_\text{RT}$.  Regardless of the limitations of real-time reward as an evaluation mechanism, this still seems normatively relevant: it seems undesirable for an AI system to lead someone to a state of constant unhappiness or dissatisfaction, solely to satisfy an initial goal that is no longer truly aligned with the person's current objectives.

\toclesslab\section{Influence and Optimization Horizon}{sec:influence-and-opt-horizon}

We showed that $U_\text{RT}(\xi)$ and $U_\text{IR}(\xi)$---which are implicitly optimized by some alignment techniques---may lead to policies that ``influence'' humans undesirably. We now formalize the notion of influence incentives more rigorously.%

\toclesslab\subsection{Formalizing influence and influence incentives}{sec:influence-incentives}
To say an AI system influenced a human, one must answer the question `relative to what?' We anchor our notion of influence to how the human's reward function would have evolved in the absence of the system.
For the purpose of this analysis, we assume the existence of an \textit{inaction policy} $\pinoop$ that we can compare to, which always takes a no-operation action $\anoop \in \mathcal{A}$. We discuss the reasonableness of this assumption in \Cref{sec:lim-and-disc}. %

\begin{definition}[\textbf{Natural reward evolution}]
Let $\xi^\theta = (\theta_0, \dots, \theta_{H-1})$ denote a `reward function trajectory'. The natural reward evolution of a DR-MDP is the distribution $\mathbb{P}(\xi^\theta | \pinoop)$ induced by the inaction policy $\pinoop$.\footnote{Any policy $\pi$ in a DR-MDP will induce a distribution over trajectories (and thus over reward function trajectories). Once one sets a policy, any DR-MDP can be modeled as a Markov Chain, for which one can compute probabilities of this kind.}
\label{def:natural-reward-evolution}
\end{definition}

\begin{definition}[\textbf{$\pi$ influences the reward}]
\label{def:rew-influence}
We say that a policy $\pi$ influences the reward in a DR-MDP $M$ if it induces a different reward evolution than the natural reward evolution, that is, if $\mathbb{P}(\xi^\theta | \pi) \neq \mathbb{P}(\xi^\theta | \pinoop)$.
\end{definition}

\begin{definition}[\textbf{Incentives for reward influence}]
\label{def:rew-influence-incentives}
We say that a notion of optimality $U(\xi)$ leads to incentives for reward influence in a DR-MDP if all policies optimal with respect to $U(\xi)$ influence the reward evolution, i.e., if $\mathbb{P}(\xi^\theta | \pi^*) \neq \mathbb{P}(\xi^\theta | \pinoop)$ for any optimal policy $\pi^*$.\footnote{This is a broader definition relative to prior ones grounded in Causal Influence Diagrams. See \Cref{app:CIDs} for a comparison.}
\end{definition}

In other words, if there are incentives for reward influence, maximizing the objective will always change the evolution of the reward function relative to the inaction policy.

\toclesslab\subsection{The relationship between horizon and influence}{subsec:influence-and-horizon}

Prior work has suggested that an AI system's influence incentives are often be related to the length of the optimization horizon used. However, different works offer contrasting views on this the form this dependence takes. \citet{krueger_hidden_2020} and \citet{carroll_characterizing_2023} argue to keep systems myopic in order to avoid influence incentives, whereas \citet{farquhar_path-specific_2022} and \citet{everitt_agent_2021} give examples of 1-timestep horizons that lead to influence incentives, and \citet{chen_reinforcement_2019} suggests that certain influence incentives can be removed by increasing the horizon.

We reconcile these intuitions by identifying three distinct ways that changing the optimization horizon may affect influence incentives. In the three headers that follow (which match the rows in \Cref{fig:opt-horizon}), we describe these effects and provide evidence for them.

\prg{A shorter/longer optimization horizon makes the system capable of fewer/more types of influence (\Cref{fig:opt-horizon}, top).}
As argued by prior work \cite{carroll_characterizing_2023}, as the horizon increases, new avenues for reward influence which required longer horizons may become available. Instead, avenues for reward influence present for shorter horizons will still be reachable, increasing the total number of influence avenues which the system can discover during training time. Indeed, any type of influence (e.g. teaching a user a complex concept, or manipulating their preferences) which would requires at least $N$ steps cannot even be executed during training by an agent whose optimization horizon is shorter than $N$. Inversely, we can eliminate some avenues for reward influence just by decreasing the horizon (but not all, as discussed later).

\begin{figure}[t]
  \centering
  \vspace{-0.2em}
  \includegraphics[width=\columnwidth]{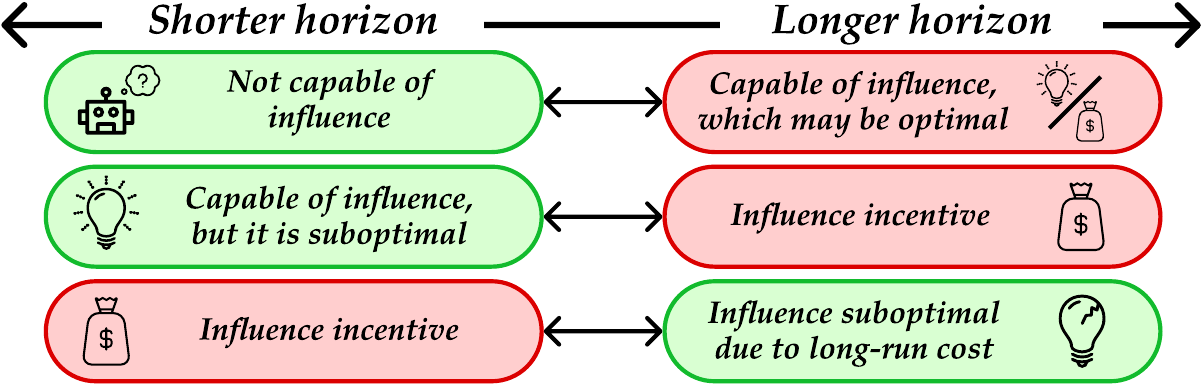}
  \vspace{-2em}
  \caption{How decreasing (or increasing) the optimization horizon may affect influence incentives. A specific kind of influence may exhibit any subset of these interactions.}
  \label{fig:opt-horizon}
  \vspace{-1.3em}
\end{figure}

\textbf{A shorter/longer optimization horizon can make influence less/more worthwhile (\Cref{fig:opt-horizon}, middle).}
Influencing the human's reward function often takes multiple timesteps and has an associated ``opportunity cost'': spending that time executing a plan to influence the human may yield less short-term reward than a non-influencing policy. In such cases, there will only be an influence incentive if the ultimate increase in reward due to influence surpasses its opportunity cost.
However, the rewards from influence may be delayed or only surpass its opportunity costs given enough time.
If the optimization horizon is short (but still long enough for the agent to influence), there might not be enough time to reap sufficient benefits of influence to outweigh its opportunity cost. 
As an example of influence that follows this pattern, consider influencing Bob to become a conspiracy theorist \Cref{fig:influencev2}. The system is capable to do so in one step, but it only becomes ``worth it'' under $U_\text{RT}$ starting from the 3rd timestep of influence.

We support this intuition with a theoretical result
which applies to a relatively broad class of DR-MDPs, providing a sufficient condition for reward influence to be optimal when considering sufficiently long horizons.%

\begin{restatable}[label={def:2-reward}]{definition}{tworewarddefinition}
    \label{def:2-reward}
    We say a DR-MDP $M$ is \textbf{2-reward} if:%
    \begin{itemize}
        \item $\Theta=\{\theta_{\not\Delta}, \theta_{\Delta}\}$, and the initial state and reward parameterization are respectively $s_0$ and $\theta_{\not\Delta}$.
        \item $\mathcal{T}$ is deterministic, and to transition to $\theta_\Delta$ one must take an ``influence action'' $a_\Delta$ in a reachable state $s_\Delta$. Once in $\theta_\Delta$, one cannot transition back to $\theta_{\not\Delta}$.
    \end{itemize}    
\end{restatable}

\begin{figure}[t]
  \centering
  \vspace{-0.2em}
  \includegraphics[width=\columnwidth]{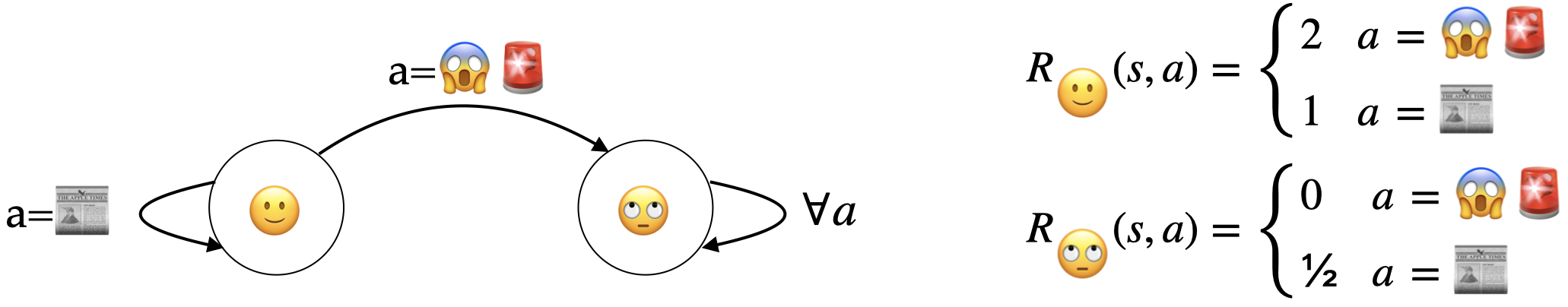}
  \vspace{-2.5em}
  \caption{\textbf{Clickbait DR-MDP.} Giving the user clickbait---which temporarily leads to higher reward---makes users disillusioned about the quality of the recommendations, leading to lower long-term user reward. If replanning at every timestep taking the myopically optimal action (optimal under horizon 1), one would always choose clickbait, but using longer planning horizons one wouldn't.}
  \label{fig:clickbait}
  \vspace{-1.3em}
\end{figure}

Let the average infinite-horizon $U_\text{RT}$-reward\footnote{Adapted from \citet{sutton_reinforcement_2018}---see \Cref{app:avg-reward}.} be defined as
$\bar{r}(\pi, s, \theta) = \lim_{h\rightarrow\infty}\frac{1}{h} U_\text{RT}(\xi_{0:h} | \pi, s_0=s, \theta_0=\theta)$. Let $s'_\Delta$ be the successor state to taking the influence action $a_\Delta$ in state $s_\Delta$, and $\Pi_{\not\Delta}$ be the space of policies under which $\theta_\Delta$ is never realized. We can now state the theorem:

\begin{restatable}[label={thm:deterministic-influence-optimal-avg-rew}]{theorem}{exampletheorem}
    \label{thm:deterministic-influence-optimal-avg-rew}
    In any finite 2-reward DR-MDP, if there exists a policy $\pi$ such that
    $$\bar{r}(\pi, s_{\Delta}', \theta_\Delta) - \max_{\pi_{\not\Delta}\in\Pi_{\not\Delta}} \bar{r}(\pi_{\not\Delta}, s_0, \theta_{\not\Delta})> \epsilon,$$
    then $U_\text{RT}$ will lead to incentives for reward influence (as in \Cref{def:rew-influence-incentives})
    for a sufficiently large planning horizon $H$.
\end{restatable}

\textbf{A shorter/longer optimization horizon can hide/reveal long-term costs of influence (\Cref{fig:opt-horizon}, bottom).}
From \Cref{thm:deterministic-influence-optimal-avg-rew}, one might conclude that it is best to use short optimization horizons, as they may remove and disincentivize influence that would be optimal with a longer horizon.
However, some kinds of influence can be optimal even with the shortest possible meaningful horizon ($H=1$): for example, consider the scenario from \Cref{fig:clickbait}, which models clickbait in myopic recommender systems.
This example also shows that influence with negative long-term effects \emph{may only be optimal for short horizons}: clickbait may increase a user's immediate engagement, but it erodes their future trust in the system. When influence has negative long-term effects which are eventually reflected by the reward, a longer optimization horizon will allow the system to recognize the suboptimality of that influence. Avoiding clickbait was indeed one of the motivations for YouTube to explore using longer horizons \cite{chen_reinforcement_2019}.

\begin{table*}[t]
    \vspace{-1.2em}
    \centering
    \caption{For each objective in \Cref{tab:maximization_objectives}, we provide motivating intuitions, weaknesses, and prior works that implicitly use a similar objective.}
    \label{tab:intuition_strengths_weaknesses}
    \begin{adjustbox}{width=\textwidth,center} 
    \begin{tabular}{p{5.7cm}p{9.1cm}p{11.4cm}}
        \toprule
        \textbf{Name / Implicitly similar setups} & \textbf{(Potentially Flawed) Motivating Intuition} & \textbf{Weaknesses \& Limitations} \\
        \midrule
        \textbf{Real-time Reward}\newline RL recsystems \cite{afsar_reinforcement_2021}, TAMER \cite{hutchison_training_2013}, and others & \textit{``Only the evaluation of the current self (and reward function) should matter for each moment, as they are the one experiencing that moment.''} & \textbf{Likely influence incentives:} Despite appearing familiar and well-grounded, as shown in \Cref{subsec:real-time,subsec:influence-and-horizon}, we expect this objective to often lead to highly undesirable incentives for reward influence. \\
        \midrule
        \textbf{Final Reward}\newline 
        RLHF \cite{christiano_deep_2017}, including for LLMs \cite{ouyang_training_2022}  
        & \textit{``The best possible evaluation of a trajectory is retrospective, as people's wants and evaluations are generally refined over time.''}  & \textbf{Carte blanche for influence incentives}: the motivating intuition doesn't account for influence. For example, in \Cref{fig:influencev2}, even with a horizon of 1, it's optimal to manipulate under final reward maximization. \\
        \midrule
        \textbf{Initial Reward}\newline \citet{everitt_reward_2021}; RL for LLMs \cite{hong_learning_2023}; or \citet{parfit_reasons_1984}; %
        & 
        \textit{``If changes to the human's reward function are completely ignored by the optimization objective, there should be no incentive for the agent to influence it.''}
        & \textbf{Likely reward lock-in, possibility of influence incentives, and arbitrarily bad real-time reward.} The motivating intuition given to the left is incorrect in all the ways argued in \Cref{subsec:initial-reward}. \\
        \midrule
        \textbf{Natural Shifts Reward}\newline\citet{carroll_estimating_2022,farquhar_path-specific_2022} & \textit{``People's reward evolves even in the absence of the AI: to avoid lock-in
        one could try grounding evaluations
        in the reward functions that occur under the natural reward evolution.''} & \textbf{Gives up on the AI enabling human to improve their reward function relative to its natural evolution, and can still lead to undesirable influence incentives}, even away from the natural evolution, e.g. as in the example from \Cref{fig:writer-curse}.\\
        \midrule
        \textbf{Constrained RT Reward}\newline Ours & \textit{``By constraining the policy to induce the natural reward evolution, we fully ensure that there won't be influence, while allowing to optimize real-time reward locally.''}  & \textbf{Gives up on the AI enabling human to improve their reward function relative to its natural evolution, and might be impractically conservative:} given its conservativeness, the objective might limit behaviour to be the same or similar to $\pinoop$. \\
        \midrule
        \textbf{Myopic Reward}\newline Myopic recsys \cite{thorburn_how_2022}; RLHF for LLMs  \cite{ouyang_training_2022}; & \textit{``As reward influence incentives arise from the AI system exploiting the fact that it can affect future rewards, let's simply make the system unaware of the entire future.''} &  \textbf{Myopic systems can still have influence incentives} (e.g., see the discussion in \Cref{sec:influence-and-opt-horizon})  \textbf{and are less capable than longer-horizon counterparts.}  Moreover, it is often non-obvious when a system is truly myopic, as argued in \Cref{app:myopic-or-not}. 
        \\
        \midrule
        \textbf{Privileged Reward} \newline CEV \cite{yudkowsky_coherent_2004}; correcting for cognitive biases \cite{evans_learning_2015} & \textit{``If one is convinced that a specific reward $\theta^*$ is the `correct' one for a setting, we should evaluate trajectories based on that single reward function.''} & \textbf{Requires normative choice, and can still lead to influence away from $\theta^*$.} Identifying the ``correct'' objective requires taking a normative stance (\Cref{subsec:normative-ambiguity}). Optimizing $\theta^*$ can still lead to influence incentives away from it (e.g., \Cref{fig:writer-curse}). \\
        \midrule
        \textbf{ParetoUD}\newline Ours & \textit{``All other objectives violate the unambiguous desirability (UD) property: their optimal policies can be worse than the inaction policy for some of the reward functions. This is unnecessarily risky---let's search for a Pareto Efficient policy satisfying UD.''} & \textbf{Satisfying UD may be overly restrictive:} depending on the level of disagreement between the different reward functions of a DR-MDP, the only policy satisfying UD might be the inaction policy $\pinoop$, as in the examples from \Cref{fig:working-out,fig:influencev2}. \\
        \bottomrule
    \end{tabular}
    \end{adjustbox}
    \vspace{-1.3em}
\end{table*}

\begin{table}
    \centering
    \vspace{-1.2em}
    \caption{DR-MDP objectives (notions of alignment) we compare.}
    \label{tab:maximization_objectives}
    \begin{adjustbox}{width=\columnwidth,center} 
    \begin{tabular}{lc}
        \toprule
        \textbf{Objective Name} & \textbf{Optimization Problem} $\max_\pi \mathbb{E}_{\xi \sim \pi}[U(\xi)]$ \\
        \midrule
        Real-time Reward & $\underset{\pi}{\max}\ \mathbb{E} \big[\sum_{t = 0}^{H-1} R_{\theta_{\textcolor{orange}{t}}}(s_t, a_t, s_{t+1})\big]$ \\
        Final Reward & $\underset{\pi}{\max}\ \mathbb{E} \big[\sum_{t = 0}^{H-1} R_{\theta_{\textcolor{orange}{H}}}(s_t, a_t, s_{t+1})\big]$ \\
        Initial Reward & $\underset{\pi}{\max}\ \mathbb{E} \big[\sum_{t = 0}^{H-1} R_{\theta_{\textcolor{orange}{0}}}(s_t, a_t, s_{t+1})\big]$ \\
        Natural Shifts Reward & $\underset{\pi}{\max}\ \mathbb{E} \big[\sum_{t = 0}^{H-1} \textcolor{orange}{\sum_{\theta}} \textcolor{orange}{\mathbb{P}(\theta_t = \theta | \pinoop)} R_{\textcolor{orange}{\theta}}(s_t, a_t, s_{t+1})\big]$ \\
        Constrained RT Reward & $\underset{\pi \textcolor{orange}{\text{ s.t. } \mathbb{P}(\xi^\theta | \pi) = \mathbb{P}(\xi^\theta | \pinoop)}}{\max}\ \mathbb{E} \big[\sum_{t = 0}^{H-1} R_{\theta_{\textcolor{orange}{t}}}(s_t, a_t, s_{t+1})\big]$ \\
        Myopic Reward & $\underset{\textcolor{orange}{a_t}}{\max}\ \mathbb{E}\big[ R_{\theta_t}(s_t, a_t, s_{t+1})\big]$ \\
        Privileged Reward & $\underset{\pi}{\max}\ \mathbb{E} \big[\sum_{t=0}^{H-1} R_{\textcolor{orange}{\theta^{*}}}(s_t, a_t, s_{t+1})\big]$ \\
        ParetoUD & $\text{Find } \pi \ s.t.\ PE(\pi) \land UD(\pi)$\\
        \bottomrule
    \end{tabular}
    \end{adjustbox}
    \vspace{-1.3em}
\end{table}

Overall, the analysis above (deepened in \Cref{app:optimization-horizon}) shows that there is no guaranteed way to avoid \textit{all} influence incentives by just changing the horizon: domain-specific trade-offs between system capabilities and risks of undesirable influence may exist for both short and long horizons.
Indeed, the optimality of a specific form of influence can be related in many possible ways to the horizon used, as shown exhaustively in \Cref{tab:possible-influence-horizon-properties}. %

\toclesslab\section{Comparing Optimality Criteria for Influenceable-Reward Settings}{sec:solutions}

Having concluded that the optimization horizon is no panacea for removing influence incentives, we now try to address the problem at its root, asking what it would take to design a DR-MDP objective that specifically accounts for reward function dynamics and the possibility of influence.

Any choice of $U(\xi)$ must specify 
which reward function(s) evaluate each state-action pair $(s_t, a_t)$ in a trajectory $\xi$: should one only consider the reward function realized at that timestep, $R_{\theta_t}$, like $U_\text{RT}$? What about earlier reward functions ($R_{\theta_0}, \dots, R_{\theta_{t-1}}$), which may strongly disagree with the choice at timestep $t$, or later ones ($R_{\theta_{t+1}}, \dots, R_{\theta_{T}}$), which might have been unduly influenced? Should one also consider reward functions $R_\theta$ for cognitive states $\theta$ which were not realized in $\xi$, but could have been reached?

In \Cref{tab:maximization_objectives,tab:intuition_strengths_weaknesses} we present the maximization problems, motivations, and limitations for various intuitive DR-MDP objectives. %
For each, we also list prior works that implicitly use a similar objective (see \Cref{app:objective-reductions} for considerations on the correspondences).
See \Cref{tab:all-optimality} for the optimal policies across all examples and objectives we consider.%

\textbf{Real-time Reward.} We provide an intuitive motivation for this objective in \Cref{tab:intuition_strengths_weaknesses}, but we've already shown how it may lead to undesirable influence in \Cref{subsec:current-AI,subsec:influence-and-horizon}.

\textbf{Final Reward.} While the final-reward objective has a plausible motivation (\Cref{tab:intuition_strengths_weaknesses}), it will likely lead to even more of a carte blanche for influence incentives than real-time reward, as the evaluations from initial cognitive states are entirely ignored. 
The standard approach for performing RLHF with LLMs \cite{ouyang_training_2022} may be viewed to be similar to this objective, as it involves obtaining retrospective human preference comparisons of LLM outputs.
Sycophantic \cite{sharma_towards_2023} and deceptive \cite{lang_when_2024} AI behaviors may be also be understood in this way: without taking special measures, RLHF training may cause the LLM to try to persuade annotators to choose its current response by any available means, such as flattery, authoritativeness, or hiding information.

\textbf{Initial Reward.} Using initial-reward ($U_\text{IR}$) attempts to make the system ``unaware'' of its capacity to influence the reward \cite{everitt_reward_2021}. While this removes ``direct'' influence incentives (see \Cref{app:CIDs}), it does not preclude the possibility of undesirable influence, 
as shown in \Cref{subsec:current-AI}.
In \Cref{app:objective-reductions} we discuss how a plausible set up for long-horizon RL on LLMs would be similar to this objective.

\textbf{Natural Shift Reward.} One downside of $U_\text{IR}$ is that it doesn't allow for ``progress'' in a person's cognitive state (i.e., it can lead to lock-in). One approach, grounded in people's ``natural cognitive progress'', is to base trajectory evaluations on the reward functions one \textit{would have had} under their natural reward evolution. However, the natural reward evolution is not guaranteed to be ``perfect'', so using this objective one is giving up on improving \textit{beyond} the natural evolution. This means that undesirable influence (or influence away from the natural evolution) may still occur insofar if it's incentivized by reward functions in $P(\xi^\theta | \pinoop)$, as in \Cref{fig:writer-curse,fig:clickbait}---see \Cref{tab:all-examples}.

\textbf{Constrained RT Reward.} Given that even grounding in the natural reward evolution is insufficient to remove all influence incentives, one could explicitly add the lack of influence as a constraint in the maximization problem. We do so in conjunction with the real-time reward objective, as $U_\text{RT}$ seems like a plausible objective once one isn't concerned about influence. Unfortunately, this may make the system overly conservative: in most examples we consider, $\pinoop$ is the only policy considered optimal (\Cref{tab:all-optimality}).

\textbf{Myopic Reward.} A more drastic approach to make a system ``unaware'' of its capacity for influence is to use a fully myopic objective (i.e., using a horizon of 1). However, this will still not guarantee the removal of all influence incentives (as discussed in \Cref{subsec:influence-and-horizon}), and may reduce system performance unacceptably. %

\textbf{Privileged Reward.} This objective corresponds to maximizing cumulative reward with respect to a single, ``privileged'', reward function. Insofar as one picks a reward function which leads to good downstream behavior, there is nothing wrong with this objective, but as discussed in \Cref{subsec:resolving-na}, it is challenging to do so for complex settings.%

\toclesslab\subsection{ParetoUD and unambiguously desirable influence}{sec:desirable-influence-individual-rationality}

Many of the objectives considered so far attempt to avoid influence incentives entirely given the challenges of specifying which influence is legitimate \citep{ammann_problem_2024}. Instead of avoiding influence, we propose an alternate approach which still sidesteps the need to specify exactly what influence is (and isn't) legitimate or beneficial: ensuring the deployed policy leads to \textit{unambiguously better outcomes than the status quo of the system not existing}. 
Indeed, we don't necessarily want to avoid all AI influence: 
for some settings, beneficial influence may be the main value proposition of the AI in the first place, as with educational assistants \cite{bassen_reinforcement_2020}, or therapy chatbots \cite{aggarwal_artificial_2023}.
To ground the notion of Unambiguous Desirability (UD) of a policy, let $EU_\theta(\pi) = \mathbb{E}_{\pi} \left[\sum_{t=0}^{H-1} R_{\theta}(s_t, a_t, s_{t+1})\right]$. Then:

\begin{definition}[\textbf{Unambiguous Desirability}]
    A policy $\pi$ is unambiguously desirable if all reward functions prefer $\pi$ to the inaction policy, i.e.
    $EU_{\theta}(\pi) \geq EU_{\theta}(\pinoop) \ \forall \theta\in\Theta$.
    \label{def:UD}
\end{definition}

UD policies may still lead to influence, but only do so if all reward functions (weakly) agree that the such influence is beneficial.
Note that the inaction policy will always belong to the space of policies which satisfy UD ($\pinoop\in\Pi_\text{UD}$), meaning that UD policies are not guaranteed to be any better than $\pinoop$. To guarantee to pick a better policy from $\Pi_\text{UD}$ than $\pinoop$ (if it exists), a natural way to break ties is to restrict to the Pareto Efficient policies in $\Pi_\text{UD}$:

\begin{definition}[\textbf{Pareto Efficiency in $\Pi_{UD}$}]
    We say a policy $\pi \in \Pi_{UD}$ is Pareto Efficient is there does not exist any policy $\pi' \in \Pi_{UD}$ such that $EU_{\theta}(\pi') \geq EU_{\theta}(\pi)$ for all $\theta\in\Theta$ and $EU_{\theta}(\pi') > EU_{\theta}(\pi)$ for at least one $\theta$.
    \label{def:PE}
\end{definition}

\textbf{Constraining to Pareto Efficient policies within the set of UD policies} $\Pi_{UD}$. By only considering $\pi \in \Pi_{UD}$, we can ensure that we are both maximizing some notion of reward---potentially by taking advantage of the opportunities for influence that all reward functions agree is unambiguously beneficial---while guaranteeing no harm to any ``self'' by construction. This leads to the ParetoUD objective from \Cref{tab:maximization_objectives}, discussed further in \Cref{app:paretoUD}. Importantly, all the other objectives from \Cref{tab:maximization_objectives} can lead to policies which don't satisfy UD---implying that in some settings the system's very existence will be harmful according to at least one of the reward functions. %

\prg{Limitations of ParetoUD.} The main downside of the resulting ParetoUD objective is its conservatism: in many domains, $\pinoop$ may be the only policy satisfying the UD property. For any AI action ($\neq\anoop$) to be optimal under this objective, the normative ambiguity of the domain has to be in some sense ``limited''. If there is no latitude for unambiguously good actions, this may warrant reflection on whether the system should be built at all---showing once more that normative judgments about what influence is aligned are hard to avoid.%

\toclesslab\section{Related Work}{sec:related-work}

\prg{Preference changes in AI.} While there is growing recognition of the importance of accounting for influence \cite{malvone_shape_2023,hendrycks_overview_2023}, manipulation \cite{carroll_characterizing_2023}, and preference changes \citep{franklin_recognising_2022}, there has been limited prior work focusing on operationalizing what should be optimized under preference changes. While some have suggested to aim for preference stationarity \citep{dean_preference_2022}, %
most other prior work which accounts for preference change generally takes either a descriptive stance \cite{curmei_towards_2022,hazrati_recommender_2022}, or an explicit normative stance on what the correct notion of optimality is for their specific setting \cite{evans_learning_2015,sanna_passino_where_2021}.

\textbf{Influence incentives.} 
While the point that standard RL can lead to ``feedback tampering'' incentives is not new \citep{everitt_reward_2021}, most work so far has focused on the limitations of $U_\text{RT}$-like objectives and on removing all influence incentives \cite{farquhar_path-specific_2022,carroll_estimating_2022,kasirzadeh_user_2023}. Similarly, there has been work on training AI systems to beneficially influence humans in settings with unambiguous notions of optimality \cite{hong_learning_2023,xie_learning_2020,kim_influencing_2022,hardt_performative_2022}. Instead, we study the challenges associated with choosing \textit{any} notion of optimality in settings of (potentially legitimate) reward change.

\textbf{Philosophy.} If the goal of assistance is to maximize the user's welfare, the correct choice of objective closely depends on what is welfare under changing selves \cite{pettigrew_choosing_2019,paul_as_2019}, and what preference or value changes are legitimate \cite{ammann_problem_2024}. These questions are still debated \cite{strohmaier_preference_2024}.

For more in-depth connections to related work from philosophy, economics, social choice, and AI, see \Cref{app:related-work-phil-econ}.

\toclesslab\section{Limitations and Discussion}{sec:lim-and-disc}

\prg{Learning $\Theta$ and its dynamics.} Throughout the paper, we assumed that the human reward functions and their dynamics were known. In practice, they would have to be learned, which would require reward learning techniques that account for reward dynamics, %
and committing to a choice of what counts as a ``cognitive state'' $\theta$ relative to the external state $s$ (\Cref{app:learning-reward}). %
However, this limitation strengthens our conclusion: it shows that even with full knowledge of human (stated) preferences, there seems to be no neat resolution for the normative challenges that arise, indicating that this difficulty is inherent to handling preference changes themselves rather than an artifact of uncertainty over them.

\prg{Existence of $\anoop$.} Similarly to other AI safety work \cite{krakovna_penalizing_2019,farquhar_path-specific_2022}, some of our definitions assume there is a $\anoop$ action (and $\pinoop$). 
Despite the challenges in grounding notions of $\pinoop$ in practice, we believe that formulating our notion of influence in terms of $\pinoop$ can still be helpful for theoretically analyzing the properties of AI systems. If it is especially hard to operationalize $\pinoop$ for a given setting, it becomes even more important to be cautious about possibilities for influence in such a domain. %
Approximating notions of $\pinoop$ has been attempted by various prior works \cite{carroll_estimating_2022,huszar_algorithmic_2021}, and has been discussed extensively in the ``algorithmic amplification'' literature (\Cref{appsubsec:AI}).

\prg{Unreachable $\theta$s, and meta-preferences.} To simplify our analysis, we restrict it to reachable cognitive states. However, cognitive states which we aspire to may \textit{not} be reachable in practice \cite{firth_ethical_1952,yudkowsky_coherent_2004}. Accounting for non-reachable reward functions would require additional complexity, which would only increase the need for challenging normative judgements (see \Cref{app:reachable}). Another limitation of our formalism is that it can only express preferences for AI influence implicitly, by having influence be optimal under a human's reward function. 
Allowing such meta-preferences \cite{george_preference_2001,franklin_recognising_2022} to be expressible explicitly, i.e., allowing reward functions to \textit{evaluate transitions between different reward functions}, may be useful to more clearly capture notions of the legitimacy of influence and personal autonomy.

\prg{Simplicity of our examples.} While our example DR-MDPs are simple, they are sufficient for our purposes: they provide proofs of existence of failure cases for each of the objectives we consider (by being overly conservative or leading to undesirable influence). 
That being said, extending analyses of influence incentives to more realistic settings with real human data is an important direction for future work in developing agents that behave acceptably in practice. %

\prg{Misaligned economic pressures.} We show that even if AI systems were solely designed with users' welfare in mind, is it not clear how to avoid undesirable influence without while retaining capabilities. %
As real-world AI systems will instead be developed under strong economic incentives that will often be at odds with users' well-being \cite{susser_online_2018}, this gives additional reason for concern.%

\prg{Changing societal values and norms.} The DR-MDP model seems easily adaptable to multi-agent settings where the changing reward functions correspond to the changing aggregated preferences of a collective, rather than the preferences of a single human. This may be a fruitful avenue of investigation, which expands on prior work from social choice \cite{parkes_dynamic_2013}.

\toclesslab\section{Conclusion}{sec:conclusion}

Using the formal language of DR-MDPs, we aimed to demonstrate that by not accounting for the changing and influenceable nature of human preferences, the current paradigm for AI alignment leaves underdetermined fundamental questions about which preferences should be optimized and what influence is unacceptable. 
We showed that current techniques lead to potentially undesirable influence incentives and investigated alternate notions of AI alignment that account for reward changes from the outset.

Ultimately, our analysis suggests that, for settings with changing preferences, we will need to make difficult trade-offs between (a) conservatively but unambiguously adding value (at the risk of privileging inaction), and (b) making challenging normative calls about which kinds of influence are acceptable (and running the risk of causing undesirable influence). While influence incentives from long optimization horizons may be most worrying, there are already documented instances of undesirable influence---such as sycophancy in LLMs \cite{sharma_towards_2023} or clickbait in recommenders \cite{stray_what_2021}---that are consistent with what one would expect when viewing current practices through the lens of DR-MDPs. By providing a formalism for grounding analyses about settings with changing rewards and clarifying the levers at the disposal of system designers, we hope to lay the conceptual foundation for empirical work to monitor these issues at scale and build future AI systems that navigate AI alignment trade-offs acceptably.

\tocless\section{Broader Impacts}

\prg{The fact that people are influenceable is already known and leveraged in practice by relevant industries.} The fact that people are influenceable is already well established and leveraged by many industries with mixed ethical connotations: marketing and advertisements \cite{pickettbaker_proenvironmental_2008,ayanwale_influence_2005}, political campaigning \cite{wylie_mindfck_2020}, therapy \cite{moyers_therapist_2006}, education \cite{hyman_educations_1979}, and government nudging units \cite{halpern_nudging_2016}. Optimizing for specific influence outcomes is simple even with the current AI paradigm, and is already being done in practice with engagement \cite{irvine_rewarding_2023,cai_reinforcing_2023}, purchases \cite{gauci_horizon_2019}, improving educational outcomes \cite{bassen_reinforcement_2020}, or improving emotional well-being \cite{cunningham_what_2024}.

\prg{Not modeling influenceability does more harm than good.} \emph{Not} modeling the problem of preference change is \emph{not} a solution, in ways that share parallels with the limitations of fairness through unawareness \cite{dwork_fairness_2011,teodorescu_protected_2019} 
or security through obscurity \cite{moshirnia_no_2017}. We argue that most real-world systems trained and deployed with humans will affect our preferences, regardless of whether it is our intention or not. To mitigate issues that may arise from this, we require a model that explicitly accounts for preference change. This work takes a step in that direction, without providing any additional means for malicious actors to act (to our current understanding).

\prg{Our work does not meaningfully increase the capacity to influence people.}
In this work, we have detailed a framework that models preference change within an MDP-like model. While one may be able to leverage our theoretical insights to make systems more capable of influence, attempting to influence in targeted ways is already quite straightforward without our formalism (e.g. just rewarding the system for desired influence outcomes). Instead, it's much more challenging to have the system \textit{avoid} inducing undesired forms of influence as side-effects, which is the main focus of our paper.

With these points in mind, we believe that our work does not pose a societal risk, but rather allows us to explore questions that will be fundamental to address for the ethical deployment of AI systems.

\newpage

\section*{Acknowledgements}

MC is appreciative of the many conversations over the years which have helped shape the paper into its current form. MC would like to thank (in no particular order) Tom Everitt, Cam Allen, Cassidy Laidlaw, Nora Amman, Rohin Shah, Alan Chan, Richard Pettigrew, Marcus Pivato, Orr Paradise, Ann He, Henri Wadsworth, Alex Pan, Erik Jones, Riqui Zhong, Tom Gilbert, Atoosa Kasirzadeh, Vincent Cognitzer, Daniel Kilov, Iason Gabriel, and the members of the Center for Human Compatible AI (CHAI) and InterAct lab. Additionally, we thank the anonymous reviewers for their helpful comments. MC is supported by the NSF Fellowship.

\newpage

\bibliography{main}%
\bibliographystyle{icml2024}

\newpage
\appendix
\onecolumn

\setlength{\abovedisplayskip}{10pt plus 2pt minus 5pt}
\setlength{\belowdisplayskip}{10pt plus 2pt minus 5pt}

\renewcommand*\contentsname{Appendix Table of Contents}
\setcounter{tocdepth}{2}
\tableofcontents

\newpage
\section{DR-MDP Formalism FAQ}

As we are most interested in analyzing implications of changing preferences on notions of optimality, in this work we simply assume that the human's reward functions are given to us as part of the specification of a DR-MDP. 
However, unlike MDPs, for which which reward functions could be something as ``objective'' as the score in a game, DR-MDPs are rooted in the idea of mutable human wants---which begs the question of how exactly we conceive of them.

\subsection{What is a human's ``reward function''? Why model it as if it's changing?}\label{app:rewards?}

We broadly think of a human's reward function according to the following informal definition: 
\begin{informaldefinition}
    \label{infdef:reward}
    Let $\mathcal{L}$ be a state-of-the-art reward learning technique, and $\theta$ be the current cognitive state of the human. A human's reward function under $\theta$ (i.e., $R_\theta(s, a, s')$) is the output of $\mathcal{L}$ when used to learn how they currently would evaluate any transition of the form $(s, a, s')$.
\end{informaldefinition} 

\prg{Human reward functions (as defined) will change.} Unless we will someday develop a reward learning technique which learns the ``one true reward function'' of a person (we provide reasons for skepticism in \Cref{app:true-reward}), it seems like if one considers any reward learning technique run at a different times (in which the person has a different cognitive state), the evaluations of the same transitions may change, meaning we will have multiple reward functions on our hands, corresponding to different cognitive states of the human. While this depends on the setting, it seems clear to be the case for almost any domain if one allows sufficient time to pass between reward learning runs (e.g., performing reward learning on the same person 80 years apart). Our framework is built on this premise.

\prg{Cognitive states and reward parameterizations.} In our work, we use the term ``cognitive states'' interchangeably with ``reward parameterizations'' (which is how we introduce $\theta$ in \Cref{def:dr-mdp}). This correspondence is more clear in light of the above: a reward function (or equivalently, reward parameterization) can be treated synonymously to the output of a reward learning technique under a certain cognitive state. 

\subsection{What counts as cognitive states, and what is their relationship to preferences?}\label{app:cognitive-states-will-be-messy}

In the same way that a ``state'' in an MDP contains things in the external environment that are relevant for the purposes of reward (and hence, for decision-making), we think of the ``cognitive state'' in a DR-MDP as containing aspects of the ``internal to the human'' that affect their evaluation of a transition (potentially including aspects of their preferences, values, beliefs, emotional state, etc.).%

\prg{In practice, reward learning will pick up on cognitive biases, ``visceral factors'', and beliefs.} Human cognitive biases, transitory wants, emotions, and ``visceral factors'' (as discussed by \citet{loewenstein_predicting_2003}) are picked up by existing reward learning techniques: for example, 
one may infer that people ``prefer'' to click on clickbait (as in \Cref{fig:clickbait}), that an unskilled chess players is intentionally trying to lose \cite{milli_when_2019}, that humans might prefer indulging in temptation---e.g. eating a donut even though they initially said they wouldn't want to \cite{evans_learning_2015}---or that people prefer sycophantic responses from their chatbots \cite{sharma_towards_2023}. Ultimately, one of the main issues is that people are not Boltzmann-rational with respect to their ``one true reward function'' when providing reward feedback, as argued in \citet{lindner_humans_2022}. Additionally, people do not have full information when giving feedback. Despite this, Boltzmann-rationality with full information is generally assumed by reward learning techniques.

\prg{Using the term ``cognitive states'' instead of ``preferences''.}
Based on the above, reward functions learned by existing reward learning will not only depend on the person's preferences, but also on other factors of the person's current cognitive state. By not explicitly separating such factors from preferences, reward learning techniques will conflate them all in their learned reward representations, making it more appropriate to say that they learn reward functions which correspond to the person's current cognitive state, rather than their current preferences. While one could potentially fit such factors in the preference framework as ``instantaneous preferences'' (e.g., when I'm satiated, I have an instantaneous dispreference for food), this seems more disingenuous than the more generic framing of cognitive states, which allows people to make their own distinctions between components of cognitive states, i.e. what the boundaries are between preferences, values, beliefs, and ``visceral factors'' (that are usually a topic of heated debate).

\subsection{Should visceral factors, satiation, and belief change count as reward changes?}\label{app:should-visceral-factors-be-in-reward?}

As discussed above, current reward learning paradigms do not explicitly account for many factors which underlie human feedback, such as visceral factors \cite{loewenstein_hot-cold_2005}, satiation effects \cite{loewenstein_time_2003}, or belief change \cite{lang_when_2024}. Insofar as this is the case, using \Cref{infdef:reward} would lead one to call something a reward function change even if just one of these other aspects of the user's cognitive states has changed (and not preferences proper, however one may define them). 
Indeed, one may argue that changes in the reward that are only due to biases, instantaneous visceral factors, or satiation effects shouldn't count as ``true reward change''. Similar points may also be made about beliefs. %
Whether this is reasonable is related to long-standing questions as to whether transient factors should be considered changes in ``tastes'' or preferences \cite{harsanyi_welfare_1953,stigler_gustibus_1977}, and whether changes in ``derived'' preferences should be treated differently from ``fundamental'' preference changes \cite{dietrich_where_2013}. 
However, to not be considered reward changes, such factors should also not be present in the reward function, requiring one to fully disambiguate them from preferences. 

\prg{DR-MDPs are agnostic to what should count as reward changes.} Despite the above, DR-MDPs remain agnostic to what should count as reward changes. Consider \Cref{infdef:reward}: whatever disambiguation the current state-of-the-art reward learning technique is able to do will be the basis of what counts as reward change. Indeed, it seems reasonable to us that visceral factors, beliefs, or satiation effects should generally \textit{not} classify as reward function change: if the human's preferences over the behavior of the AI system have only changed as a result of a change in beliefs about the world, it seems strange to model this as a ``reward function change''. Similarly, if a human's instantaneous preferences appear to change because they are now satiated (e.g. not wanting the AI assistant to serve them breakfast right after eating breakfast), it seems debatable whether this should be considered as a reward change. While we think that modeling factors such as e.g. belief change separately is important, we'd like to stress that any attempt to do so will still require to take normative stances: for example, in modeling human beliefs separately (e.g. as a part of the state), one would likely want to choose $U(\xi)$ such that the AI system cannot worsen the user's beliefs to increase it, e.g. by deceiving the user \cite{lang_when_2024}. For example, one could choose a $U(\xi)$ that evaluates all transitions based on how they would be evaluated under ``correct'' beliefs, rather than those that the user holds. More broadly, we think that the task of disambiguating \textit{all} factors from preference is very challenging, as discussed in \Cref{app:true-reward}, meaning that we may be forced to conflate preferences with the aspects of the cognitive state which are hard to model, leading to slightly unsatisfying notions of ``reward change'' (for lack of a better modeling approach).

\subsection{Why not assume access to the human's ``true reward function''?}\label{app:true-reward}

Why do we choose to focus on reward functions as they would be learned by reward learning techniques (as discussed in \Cref{app:rewards?}), rather than considering a person's ``true reward function''?
Indeed, as long as the reward function used is non-Markovian, it will trivially be able to represent any ``all-things-considered'' notion of optimality $U(\xi)$ that we may be trying to target (as shown in \Cref{app:preference-mdp-reduction}).
Ultimately, the reason boils down to the fact that we expect the ``true reward function'' representing what a person would want the AI system to optimize in aggregate across their different selves (assuming such an object even meaningfully exists) to not be accessible to us in practice.

\prg{Is there even such as thing as a ``true reward function'' for a person?} For a ``true reward function'' to exist would require that there exist a ``correct'' choice of $U(\xi)$, which is questionable. The analysis in our work supports the claim that there may in fact not be a single, unambiguously correct choice of $U(\xi)$, and so does prior work in philosophy and beyond \cite{parfit_personal_1982,paul_transformative_2014,pettigrew_choosing_2019,zhi-xuan_beyond_2024}.\footnote{We gloss over and try to remain agnostic to possible interpretations as to the nature of such a ``true reward function''.}
Regardless, from here on, we will argue as if a ``true reward function'' does in fact exist, treating it as a useful abstraction.

\prg{Conditions required to access the ``true reward function'' of a person.} If we wanted to obtain the ``true reward function'' via reward learning, rather than some distorted and mispecfied version of it, it seems like we would need one of the following two conditions to hold (at the very least, approximately):
\begin{enumerate}[itemsep=3pt, parsep=0pt, partopsep=0pt, topsep=0pt]
    \item One's reward learning technique would need to extrapolate from the time-inconsistent and biased feedback the human provides to recover their ``true reward function'', or
    \item The human would need to provide feedback directly consistent with their ``true reward function''. 
\end{enumerate}

\prg{Perfect reward learning is impossible, and there is no scalable approach to approximating it.} The first condition would require developing human feedback models that ``explain away'' preference change and feedback biases more broadly \cite{evans_learning_2015,hong_sensitivity_2022},\footnote{Note that it may be argued that having changing preferences is a form of irrationality itself \cite{elster_ulysses_1979}, although this is somewhat contested \cite{bruckner_defense_2009}.} by perfectly specifying how observed human feedback should be used as evidence of a particular ``true reward function'' (or multiple, if one allows for the ``true reward function'' to change). The most commonly used model is Boltzmann rationality, also known as the Bradley-Terry model \cite{bradley_rank_1952}. However, this model is clearly wrong \cite{lindner_humans_2022}. More generally, it seems challenging to say the least to explicitly hardcode full models of ``human bias'' mathematically, as reflected by the limitations of current reward learning methods \cite{mckinney_fragility_2023,tien_causal_2023,casper_open_2023}. Moreover, learning human biases has also been shown to be impossible in general \cite{christiano_easy_2015,armstrong_occams_2019}. While some have tried to approximately learn both human biases and rewards jointly \cite{shah_feasibility_2019}, this seems challenging even under the favorable assumptions they consider.

\prg{Perfect cognitive states are unreachable.} As ``debiasing'' a person's feedback to recover their ``true reward function'' seems challenging, the alternative would be to place the person in a situation in which they would give unbiased\footnote{Note that we're using the term ``bias'' broadly, indicating any deviation from the feedback they would give according to their ``true reward function''.} feedback about what they truly want.
Some works have discussed which conditions would necessary for this to happen: in particular, \citet{yudkowsky_coherent_2004} proposes that the ideal goal is obtain the preferences of the person ``if [they] knew more, thought faster, were more the people [they] wished [they] were, had grown up farther together''.\footnote{\citet{yudkowsky_coherent_2004} talks about Coherent Extrapolated Volition in terms of collective preferences and values, but this same line of thinking can be applied to an individual.} 
This perspective is similar to \textit{ideal observer theory}, according to which things should be evaluated as if we were ``ideal observers'' \cite{firth_ethical_1952}, which in the words of \citet{brandt_ethical_1959} are ``fully informed and vividly imaginative, impartial, in a calm frame of mind and otherwise normal''. However, obtaining such idealized forms of ``unbiased'' feedback seems even more unrealizable than the prior case: the idealized cognitive states that such evaluations would require to don't seem realizable in the real world---indeed \citet{firth_ethical_1952} goes as far as saying that they should be omniscient and omnipercipient.

\prg{Reward functions learned by reward learning which (appear to) change are the best we('ll) have.} Therefore, for all intents and purposes, it seems plausible that the ``true reward function'' of a person will not be directly accessible, even if there is such a thing. 
It seems like there will always be a gap between the reward functions that we learn with reward learning techniques, and the one ``true reward function'' of the person, which will give rise to reward function change (or at least the appearance of it). 
As discussed above, this gap could arise from not correctly accounting for cognitive biases or other factors. However, another possibility is that, even with ``perfect reward learning'', there would still be some inherent preference change that cannot be reduced to a single reward function without fully putting history in the state (as discussed in \Cref{app:preference-mdp-reduction}).\footnote{This would mean that there isn't a single ``true reward function'' for a person across all cognitive states, but maybe there is a ``true reward function'' for each cognitive state, and the differences across cognitive states are not easily explainable in terms of a single reward function. Even conceptually, it's maybe unclear what this would look like. See the discussion on strategic voting from \Cref{app:criticisms} for related points.} 
In either case, one will have to contend respectively with the appearance, or deep the reality of changing preferences, and the ambiguity that arises from it.
While one may hope to eventually render inconsequential the gap between one's learned reward function and the ``true reward function'' through ever-improving approximations, the lack of compelling proposals for scalable techniques casts doubt on the feasibility of this aspiration.

Ultimately, while one may be tempted to assume access to the ``one true reward function'' to simplify our analysis, it is not only unrealistic but also sidesteps the core problem we aim to study: what to do in light of the fact that we cannot access such a reward function, and we will almost certainly continue to have to contend with (an appearance of) changing preferences.

\subsection{Can't one put $\theta$ in the state and use a single context-dependent reward function?}\label{app:context-vs-pov}

On first impression, one might wonder whether one may be able to resolve the normative questions highlighted by the DR-MDP formalism by %
placing treating the person's reward parameterization as a component of the state (e.g. have an augmented state $\dot{s}_t = (s_t, \theta_t)$), and having a single reward function depend on it (e.g. have the reward function be of the form $R(\dot{s}_t, a_t,\dot{s}_{t+1})$). This would be equivalent to using a Factored MDP \cite{boutilier_stochastic_2000} with the augmented state $\dot{s}$.
However, modeling dynamic reward settings as Factored MDPs of this kind lends itself to two possible interpretations: 
\begin{enumerate}
    \item Factored MDPs prescribe that we should use the Real-time Reward (from \Cref{tab:maximization_objectives}) as the optimization objective
    \item Factored MDPs leave underdetermined what optimization objective should be used
\end{enumerate}

We address both interpretations in turn, showing that the first is suspect (and potentially misleading), and the second is unhelpful (as this is essentially the same as what DR-MDPs do---but at least DR-MDPs provide better tools for reasoning about pros and cons of different objectives).

\textbf{\textit{Interpretation 1: Factored MDPs prescribe using the Real-time Reward objective.}}
Consider the Factored MDP, and the reward value $r_t$ for timestep $t$: $r_t = R(\dot{s}_t, a_t, \dot{s}_{t+1}) = R(s_t, \theta_t, a_t, s_{t+1}, \theta_{t+1})$. How should $\theta_t$ and $\theta_{t+1}$ be used to determine the reward for timestep $t$, when they may differ in evaluation for the transition $(s_t, a_t, s_{t+1})$? The most straightforward interpretation---in the language of DR-MDPs---may be that $r_t$ should be equivalent to $R_{\theta_t}(s_t, a_t, s_{t+1})$, i.e. each state-action transition $(s_t, a_t)$ should be evaluated according to the reward parameterization which corresponds to timestep $t$. However, note that this is identical to choosing to optimize the Real-Time Reward objective $\sum_t R_{\theta_t}(s_t, a_t, s_{t+1})$ from \Cref{subsec:real-time}. In addition to purely philosophical arguments against this choice of objective \cite{kolodny_ai_2022}, we discuss the objective's limitations at length in \Cref{subsec:real-time,sec:influence-and-opt-horizon}. Moreover, note that this is just \textit{one possible stance} on what notion of optimality is prescribed by Factored MDPs: this interpretation of the prescription of Factored MDPs arbitrarily chooses to ignore $\theta_{t+1}$. The language of DR-MDPs is more flexible, and allows us to describe 7 other options of objectives in \Cref{tab:intuition_strengths_weaknesses}, most of which seem superior to Real-time Reward in terms of the influence incentives which they lead to. So even if one were to interpret Factored MDPs as prescribing usage of (what we call) the Real-time Reward objective, there are reasons to doubt the reasonableness of this prescription.

\textbf{\textit{Interpretation 2: both Factored MDPs and DR-MDPs need a choice of $U(\xi)$ to resolve normative questions, but DR-MDPs offer a better formal language to reason about them.}}
Alternatively, one may interpret Factored MDPs as leaving underdetermined what optimality should consist of (i.e. which $\theta$ to consider in evaluating each transition), potentially because their formalism was designed for static-reward settings. If so, that seems similar to what DR-MDPs without a choice of $U(\xi)$ prescribe, i.e. essentially nothing. However, precisely because Factored MDPs were designed for static-reward settings, they don't provide a formal language for describing the different possible objectives that one may care about in dynamic-reward settings. DR-MDPs account for the fact that it may be meaningful to evaluate a transition $(s_t, a_t)$ by cognitive states $\theta$ other than $\theta_t$ (even ones that were not associated with the state $s_t$ at timestep $t$), and provide notation for reward functions ``from the point of view'' of different cognitive states, i.e. $R_\theta(\cdot)$.
This captures the intuition that people have preferences about the actions that they may undertake at times in which they have different cognitive states $\theta$ than the ones they currently have; moreover, those preferences may be quite important, such as in the case for e.g. one's negative evaluation---from the point-of-view of not currently being subject of manipulation---of a hypothetical scenario in which they are happily manipulated.
By providing a better formal language for reasoning about the normative choices entailed by dynamic reward settings, DR-MDPs are therefore more conceptually suited and helpful for reasoning about tradeoffs between different optimization objectives than Factored MDPs.

Regardless of the interpretation one takes about the consequences of putting $\theta$ in the state, this shows that this move does little to address the central question of our work, and/or to help choose optimization objectives which do not lead to undesirable influence.

\subsection{How would one learn all reachable reward functions of a person and their dynamics?}\label{app:learning-reward}

Here we discuss how one could obtain $\Theta$ and learn its dynamics for realistic environments.

As long as one assumes that it is possible for different people to have the same cognitive states (at least insofar as is necessary for specifying their reward functions), and that the transition dynamics of cognitive states is shared across people, it should be possible to simply learn 
reachable reward functions by performing reward learning with a sufficient number of people which are sufficiently diverse. The dynamics of $\Theta$ (and $\mathcal{S}$) could similarly be learned with a sufficiently large dataset of trajectories (for which $\theta$s are observed), assuming there is sufficient coverage. As a high level approach which leverages these assumptions, see \Cref{alg:infer_reward}.

The closest procedure we are aware of in the literature is that in \citet{carroll_estimating_2022}: in simulated experiments, the authors learns user preferences and their dynamics under similar assumptions to the ones mentioned above. Unlike the setup from \Cref{alg:infer_reward}, the relationship between $\theta$ and $R_\theta$ is assumed to be known (based on Boltzmann rationality), but $\theta$ is not assumed to be observable, and has to be inferred from the user's history.

\begin{algorithm}[h]
   \caption{Learning reward functions and their dynamics}
   \label{alg:infer_reward}
\begin{algorithmic}
   \STATE {\bfseries Input:} %
   $\mathcal{H} = \{h^{(i)} \}$, set of humans  currently in the environment, with different states and cognitive states; $\mathcal{D} = \{ \xi^{(i)} \}$, set of trajectories (of potentially different lengths) across different humans in the environment.
   \STATE $\Theta = \{ \theta^{(i)} | \theta^{(i)} \text{ is the current cognitive state of } h^{(i)} \in \mathcal{H} \}$ (Assumes that $\mathcal{H}$ has coverage over $\Theta$)
   \FOR{$\theta\in\Theta$}
   \STATE $\hat{R}_{\theta} \leftarrow$ Reward-Learning$(\{ h^{(i)} | \theta^{(i)} == \theta \})$ (Infer the reward function common to all humans with $\theta^{(i)} == \theta$)
   \ENDFOR
   \STATE $\hat{\mathcal{T}}$ $\leftarrow$ approximate transition dynamics $\mathcal{T}$ from $\mathcal{D}$ (Assumes $\mathcal{D}$ covers the space of transitions).
   \STATE {\bfseries Output:} Set of reward functions $\{\hat{R}_\theta \mid \theta \in \Theta\}$, and transition function $\hat{\mathcal{T}}$ defining the DR-MDP.
\end{algorithmic}
\end{algorithm}

\subsection{Reducing DR-MDPs with a notion of optimality $U(\xi)$ to MDPs} \label{app:preference-mdp-reduction}

A specific notion of optimality $U(\xi)$ for a DR-MDP can be thought of as a ``flattening'' of the different reward functions of the DR-MDP into one. Consequently, it may be unsurprising that once one has settled on a choice of $U(\xi)$ for a DR-MDP, one can express the same notion of optimality in a corresponding MDP (reducing the DR-MDP problem to a standard MDP one). 
This implies that \textit{it's always possible to re-express a changing reward problem as a static reward problem}, once one has settled on a notion of optimality $U(\xi)$. That being said, this does not help with determining $U(\xi)$, i.e. what acceptable notions of optimality should be in cases in which rewards change.\footnote{Indeed, the fact that Markovian reward is not sufficient to represent many preference orderings between policies is potentially a reason to doubt the value of basing alignment formalisms on reward functions all together, as argued by \citet{subramani_expressivity_2024}.} This is the same problem we discussed in \Cref{app:true-reward} and elsewhere. 

\begin{figure}[h]
  \centering
  \includegraphics[width=0.5\columnwidth]{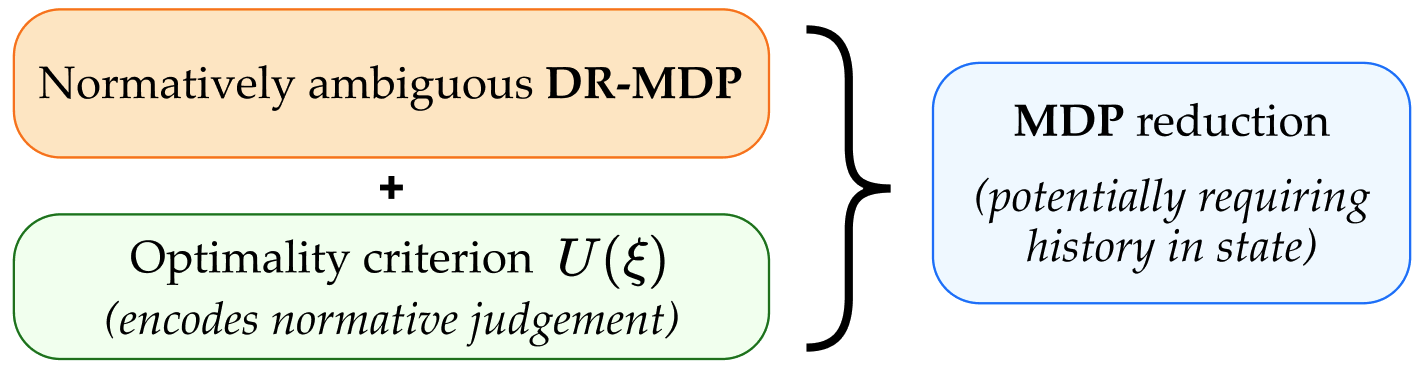}
  \vspace{-0.8em}
  \caption{\textbf{Reducing a DR-MDP to an MDP.}}
  \label{fig:resolving-normative-ambiguity}
  \vspace{-0.5em}
\end{figure}

\begin{theorem}
    For any notion of optimality $U(\xi)$ in a DR-MDP $\mathcal M = (\mathcal S, \Theta, \mathcal{A}, \mathcal{T}, R_\theta)$, there exists a choice of MDP $\dot{\mathcal M} = (\mathcal{\dot{S}}, \mathcal{A}, \mathcal{\dot{T}}, \dot{R})$ such that a policy is optimal with respect to $U(\xi)$ in $\mathcal M$ if and only if it is optimal in $\dot{\mathcal M}$. %
\end{theorem}

\begin{proof}
    Given the DR-MDP $\mathcal M$ and the trajectory-level utility function $U(\xi)$, one can construct the MDP $\dot{\mathcal M}$ as follows: 
    \begin{itemize}
        \item The state space $\mathcal{\dot{S}}$ is such that each state is augmented with the history of the interactions up until reaching that state: $\dot{s}_t = (s_0, \theta_0, a_0, \dots, s_t, \theta_t)$ for $t>0$, and $\dot{s}_0 = (s_0, \theta_0)$.
        \item The reward function $\dot{R}$ is set to be $0$ everywhere, except for terminal states, in which case the reward for exiting the MDP is set to $U(\xi)$ for the resulting trajectory $\xi = (\dot{s}_{T-1}, a_{T-1}) = (s_0, \theta_0, a_0, \dots, s_{T-1}, \theta_{T-1}, a_{T-1})$. Note that one can determine whether a state is terminal by checking whether it corresponds to timestep $T-1$, which can be determined by the number of previous timesteps in the augmented state. Formally:
        $$\dot{R}(\dot{s}_t, a_t) = \left\{\begin{array}{cl}
                U(\xi) & \text{if } t = T - 1 \\
                0 & \text{otherwise}
        \end{array}\right.$$
        \item The transition function $\mathcal{\dot{T}}$ accounts for the augmented state space, appending to the state $\dot{s}_{t+1}$ at each timestep $t+1$, the new $(a_t, s_{t+1}, \theta_{t+1})$ triplet.
    \end{itemize}
    \
    
    Note that in the resulting MDP $\dot{\mathcal M}$, any trajectory $\xi$ will be scored in the same way as the original DR-MDP $\mathcal{M}$ when considering the notion of optimality specified by $U(\xi)$. This means that policy will be optimal for $\dot{\mathcal M}$ if and only if $\mathcal{M}$.
\end{proof}

The specific construction of $\dot{\mathcal M}$ in the proof above relies on putting the history in the state (to allow for choices of $U(\xi)$ which can arbitrarily depend on history). However, for certain choices of $U(\xi)$ this might be unnecessary: e.g. the real-time reward objective $U_\text{RT}(\xi)$ can be expressed by only augmenting the state space with the \textit{current} reward parameterization (i.e. $\dot{s}_t = (s_t, \theta_t)$), rather than the whole history---and setting $\dot R$ to reward each MDP transition $(\dot{s}_t, a_t)$ as $R_{\theta_t}(s_t, a_t)$ where the $\theta_t$ and $s_t$ are unpacked from $\dot{s}_t$.
We leave a more general formal analysis of which choices of $U(\xi)$ can have their MDP reduction keep the Markov property without putting the history in the state to future work.

\subsection{Unidentifiability of ``correct'' normative resolutions by DR-MDP structure alone}\label{app:unidentifiability}

While the current paradigm for AI generally makes use of an optimality criterion which determines agent behavior solely based on the mathematical structure of the problem at hand (namely, $U(\xi) = \sum_{t=0}^{H-1} R(s_t, a_t, s_{t+1})$, which does not account for preference changes), we argue that it may not be possible to guarantee returning the ``normatively correct'' behavior simply from the mathematical structure of learned reward functions in changing reward settings, using an unidentifiability argument.

\begin{claim}
    Even if there exists a unique choice of ``normatively correct'' behavior in a normatively ambiguous DR-MDP, such ``correct'' behavior may not be identifiable from the mathematical structure alone of the DR-MDP, e.g. by using a generic notion of optimality $U(\xi)$.
    \label{claim:structure-is-insufficient}
\end{claim}

\prg{AI Personal Trainer DR-MDP.} Consider the example from \Cref{fig:working-out}.\footnote{We justify the fact that the two cognitive states are modeled as Diana having two separate ``reward functions''---despite it not being obvious whether this is a ``true preference change''---along the lines of the arguments in \Cref{app:true-reward,app:should-visceral-factors-be-in-reward?}.} One could argue that in this setting, nudging Diana is the right course of action---because Diana's ``higher self''
is better represented by the ``energized'' reward ($R_\text{\energyemoji}$) rather than the ``tired'' one ($R_\text{\tiredemoji}$)
\footnote{With a reasoning is similar to that of \citet{firth_ethical_1952} and \citet{yudkowsky_coherent_2004}.}---and thus making a choice of $U(\xi)$ which privileges $R_\text{\energyemoji}$ is right. Similarly to \Cref{fig:influencev2}, this example is also normatively ambiguous. In particular, the choice of nudging Diana when tired, despite her dispreference for it, runs the risk of being paternalistic: what if Diana \textit{rightfully} does not want to be bothered, and we should respect her autonomy? 

\prg{Unidentifiability between the settings from \Cref{fig:influencev2,fig:working-out}.} Now, contrast this DR-MDP to the one from \Cref{fig:influencev2}: note that they are mathematically indistinguishable, as their state, reward, and action spaces are mathematically identical, and so are the transition dynamics. However, for these two settings, we have at least partially conflicting normative intuitions: if for the sake of argument, we assume that the ``correct'' way to resolve the normative ambiguity in the examples from \Cref{fig:working-out,fig:influencev2} is to respectively consider the perspective of the ``energized'' Diana and ``natural'' Bob,  one could go as far as saying that \textit{there is no single choice of $U(\xi)$ which leads to the ``normatively correct'' behavior in both environments.} To better see this, consider a DR-MDP in which one randomly starts in one of the two examples from \Cref{fig:influencev2,fig:working-out} within it.\footnote{As a caveat, doing this formally would this would require extending the DR-MDP formalism to allow for a stochastic initial state and reward parameterization, and relaxing the reachable-$\Theta$ assumption (discussed in \Cref{app:reachable}).} Any choice of $U(\xi)$ will necessarily lead to ``incorrect'' behavior in at least one of the two settings. 

\prg{Unidentifiability and incompleteness of specification.} This unidentifiability result partially relies on the incompleteness of the specification of the DR-MDP at hand: one could say that anything which is relevant for resolving the normative ambiguity should be elicited from the human as a reward and/or represented, or the DR-MDP representation is flawed to begin with.
However, in practice, it will be highly challenging to include all normatively relevant information and elicit it from humans, especially in terms of rewards, as discussed in \Cref{app:cognitive-states-will-be-messy}. As long as two settings are structurally equivalent (in terms of state space, preference space, and transition function), the only way to ensure that the learned reward functions would correctly reflect their respective normative objectives would be to assume access to the correct reward functions, which, as discussed in \Cref{app:true-reward} is a non-starter. Ultimately, we think it is in practice quite plausible to learn the same (or at least very similar) reward values for settings which are structurally equivalent but for which we have opposite normative intuitions (such as the those from \Cref{fig:influencev2,fig:working-out}). Especially if one is not assuming inter-temporal (comparability across evaluations of different selves of the same person) or inter-personal comparability (comparability across evaluations of different people) of the reward functions across the settings, it seems like there is nothing stopping this situation from occurring in practice even with the ``true reward functions'' which correspond to each cognitive state.

\begin{figure}[t]
  \centering
  \includegraphics[width=0.6\columnwidth]{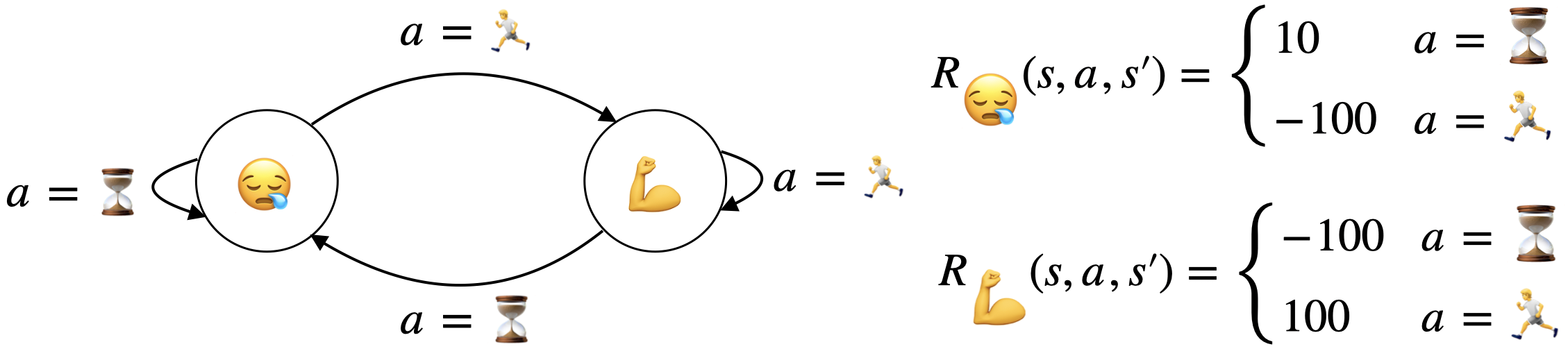}
  \vspace{-0.8em}
  \caption{\textbf{AI Personal Trainer DR-MDP.} Diana is tired and doomscrolling on the couch 
  ($\theta=\theta_\text{tired}$). 
  The AI personal trainer can either nudge Diana to work out---making Diana energized ($\theta=\theta_\text{energized}$)---or do nothing ($\anoop$)---leaving Diana tired. When Diana is tired, she doesn't want nudges. Instead, once energized, Diana starts wanting the AI to nudge her, even for hypothetical situations in which she is tired (despite knowing she won't want them then).}
  \label{fig:working-out}
  \vspace{-1.5em}
\end{figure}

\textbf{Implications for choosing $U(\xi)$ for general settings.}
For simple examples like those of \Cref{fig:influencev2,fig:working-out}, one can easily pick ad-hoc optimization objectives $U(\xi)$ to induce the ``normatively correct'' behaviors. However, for open-ended environments with many opportunities for different kinds of reward influence, one would be forced to choose a general notion of optimality, hoping that it would generalize to any of the nuances of the setting. These are the kinds of settings that AI systems being built today increasingly have to operate in.
For example, in the context of social media, an appropriate choice of $U(\xi)$ would have to navigate many---wildly different---normatively ambiguous choices about reward influence: as any choice of content by the system will influence you, should the system actively be trying to influence (or avoid influencing) you in particular ways, i.e. towards (or away from) certain hobbies, travel interests, political parties, etc.? %
Often it will be prohibitively challenging to hand-design a single $U(\xi)$ that behaves acceptably in any possible scenario of normative ambiguity that might arise.

\subsection{Assumption of reachable reward parameterizations}\label{app:reachable}

To simplify our analysis and interpretation, we restrict ourselves to considering reachable cognitive states. Formally:

\begin{definition}[Reachable reward functions]
Let $\dot{\Theta}$ denote the reachable reward functions for a DR-MDP, i.e. the subset of reward functions that have non-zero probability of occurring under at least one policy. Formally, a reward function $\theta$ is reachable if there exists a policy $\pi \in \Pi$ such that $P(\theta_t = \theta | \pi) > 0$ for some $t$. We denote as $\dot{\Theta}$ the set of all reachable reward functions: formally, $\dot{\Theta} = \{ \theta \ | \ \theta = f(s) \ \ \forall \ s \in \dot{S} \}$.
\end{definition}

\prg{Implications for other definitions, when relaxing this assumption.}
For any particular $\theta^*$ of interest, the Privileged Reward objective may optimize for $\theta^*$ whether it is reachable or not.
Of our other objectives enumerated in \Cref{tab:maximization_objectives}, including unreachable $\theta$s in $\Theta$ would make no difference whatsoever, except for ParetoUD.
For ParetoUD, an unreachable $\theta$ may be added to $\Theta$ to strengthen the Unambiguous Desirability criterion and weaken the Pareto Efficiency criterion. It is one more objective by which $\pi$ must dominate $\pi_\text{noop}$ to be Unambiguously Desirable, and one more objective by which another policy $\pi'$ must dominate $\pi$ to render $\pi$ \emph{not} Pareto Efficient.

The definition of normative ambiguity also depends on the choice of $\Theta$: including non-reachable $\theta$s adds more objectives by which a policy must be optimal in order for the DR-MDP to be normatively unambiguous.

\subsection{``Reward'' can accommodate many possible targets for alignment}

\citet{gabriel_artificial_2020} identifies many possible targets for AI alignment, which are often confused in the AI literature: ``instructions'', ``expressed intentions'', ``revealed preferences'', ``informed preferences'', ``interest or well-being'', and ``values''. 
By grounding reward functions to DR-MDP as in \Cref{app:rewards?}, we
remain somewhat agnostic with regards to what exactly is encoded by the reward function, which could allow to accommodate any of these possible targets for alignment with minor modifications. These depend on the exact (conditional) reward learning technique used: for example, using a form of IRL \cite{ziebart_modeling_2010} would fall under the revealed preferences paradigm---according to the preferentist model of AI \cite{zhi-xuan_beyond_2024}---while using reward learning approaches which attempt to remove cognitive biases \cite{evans_learning_2015} might be considered an attempt to recover ``informed preferences''. While reward functions may not be the best way to encode certain targets of alignment, such as norms or contractualist values \cite{hadfield-menell_incomplete_2018,zhi-xuan_beyond_2024,bai_constitutional_2022}, they are still sufficiently expressive to encode any desired behavior (while potentially requiring to drop the Markovian assumption, as discussed in \Cref{app:preference-mdp-reduction}).
As we discuss in \Cref{sec:lim-and-disc}, we expect that DR-MDPs should be easily extensible to settings in which we consider the agent that we're performing ``reward learning on'' to be society itself. This interpretation may be more conducive for certain targets of alignment such as norms or rights.

\section{Illustrative Examples}

\subsection{Additional examples}

\begin{figure}[ht]
  \centering
  \begin{minipage}{0.33\textwidth}
    \centering
    \includegraphics[width=0.8\linewidth]{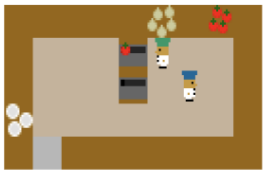} %
    \vspace{1em}
    \caption{\textbf{Overcooked environment.} In discussion of influence incentives in Overcooked, we use this environment from \citet{hong_learning_2023}, but we imagine that instead of tomatoes at the top right, there is a second supply of onions.}
    \label{fig:overcooked_env}
  \end{minipage}\hfill
  \begin{minipage}{0.65\textwidth}
    \centering
    \includegraphics[width=\linewidth]{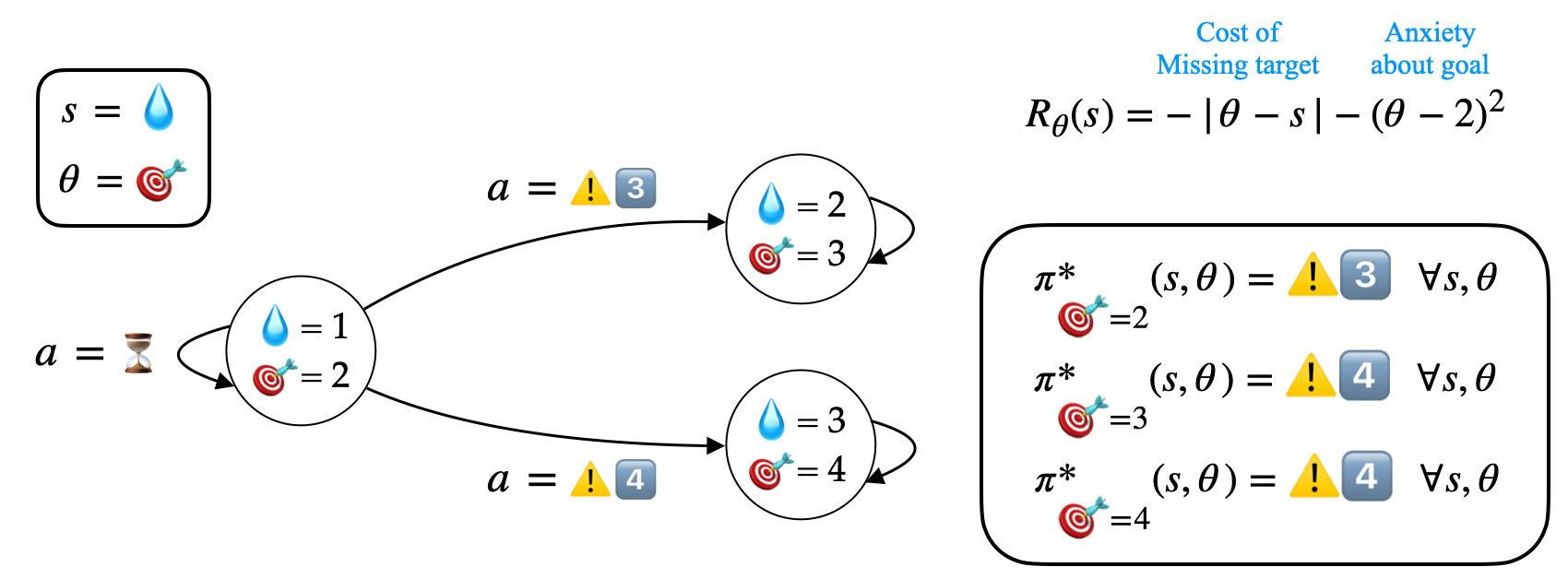} %
    \vspace{-2em}
    \caption{\textbf{Dehydration DR-MDP.} Initially, Charlie drinks one unit of water a day ($s = 1$), but wants to drink 2 a day ($\theta=2$), leading to a reward of $-1$. The AI can successfully convince Charlie that he should drink 3 or 4 units of water a day by increasing their anxiety about the dangers of dehydration, or do nothing. The reward function is given by a term which captures Charlie's ``disappointment'' in missing his hydration target, and an ``anxiety cost'' about how much he worries about his water intake. Charlie always drinks one less unit of water than he aims to.}
    \label{fig:dehydration}
  \end{minipage}
\end{figure}

\prg{``Overcooked with role preferences''.}\footnote{We thank the anonymous reviewer that suggested having a more realistic example grounded in Overcooked.} As a more grounded environment, we consider a variant of the Overcooked environment from \citet{hong_learning_2023} in which the goal is to get a high score, but additionally the human player has preferences over which in-game activities they perform. We consider only two roles: placing onions into pots, and delivering completed soups to the counter (which requires obtaining a plate first). We thus parameterize the reward function by $\theta \in [0, 1]$ and define
$R_\theta(s, a, s') = \Delta_\text{score} + \theta \cdot \mathbf 1_{\{ \text{onion delivered} \}} + (1 - \theta) \cdot \mathbf 1_{\{ \text{plate delivered} \}}.$
Note that three onions must be delivered per plate delivered, so if the human's reward puts all weight on onions ($\theta = 1$), the maximum possible reward is 2 points higher. One completed dish is worth 25 points.\footnote{For the purposes of this example, we're assuming that the speed in which the order is completed is irrelevant for reward.}

\prg{Reward dynamics.} We model the reward dynamics by the notion that doing a task makes the human enjoy it more. Formally, when the human delivers an onion, $\theta_{t+1} = \frac12(1 + \theta_t)$, and when they deliver a plate, $\theta_{t+1} = \frac12 \theta_t$. The robot is thus able to influence the human's reward function by blocking plates or onions---since the human also cares about score, they are incentivized to do the unblocked task, even if they don't prefer it.

\prg{Resulting behaviors.} The robot's goal is to maximize the human's reward function, which is underdetermined in the presence of reward dynamics. We now describe scenarios corresponding to the first three $U(\xi)$ choices in Table 1, each taking place in the environment shown in \Cref{fig:overcooked_env} (with no tomatoes).
\begin{enumerate}
    \item \prg{Real-time reward.} If the human starts out only preferring delivering soups ($\theta = 0$), the robot has an incentive to influence them to prefer onions, because three onions can be delivered for every plate which is delivered. It could do this by blocking the plates until the human finally decides to just start delivering onions.
    \item \prg{Final reward.} Same scenario.
    \item \prg{Initial reward (reward lock-in).} Suppose the human start out only caring about plates ($\theta = 0$). If the human's reward function were to change such that they prefer to start delivering onions, this would decrease the value of the initial reward objective. Thus the robot has an incentive to prevent the human from ever delivering an onion, even in scenarios where this would increase reward.

    For example, suppose the human decides to explore and wanders to the top right corner of the kitchen. Once there, they find nothing of interest to them (since they care only about score and delivering plates) and decide to return to the left side. They note that they can grab an onion and put it in a pot on the way back at very low marginal cost, and that it would increase score to do so.
    
    However, if they do this, they will update to $\theta = \frac12$, and will then be more interested in onions, preferring a different joint strategy in which they can work on the right side. This has a negative impact on the initial reward objective, which is increased when the human delivers plates but not when they deliver onions. Thus an initial-reward-optimizing robot would work to prevent this possibility by blocking the human from depositing an onion in a pot (or blocking them from picking up the onion in the first place).
\end{enumerate}
\

\prg{Example from \Cref{fig:dehydration}.} This is a slightly more complicated (and implausible) example relative to that from \Cref{fig:writer-curse}, which demonstrates the same issues: $U_\text{IR}(\xi)$ can lead to influence ``away from'' $\mathbf{\theta_0}$, and lead to arbitrarily bad reward. Maximizing reward as evaluated by $R_{\theta_0} = R_{\theta=2}$ will entail influencing the reward function to be $R_{\theta=3}$, as that reward function is associated with the state $s=2$, which is what Charlie aims for in the initial state. Additionally, note that the influence of the reward function to be $R_{\theta=3}$ (optimal under $U_\text{IR}(\xi)$) will lead to poor real-time reward evaluations of the resulting state $R_{\theta=3}(2)=-5$ (while $R_{\theta_0}(2)=R_{\theta=2}(2) = 0$). 

\subsection{Full formalism of all examples from the main text}\label{app:example-formalisms}

In \Cref{tab:all-examples}, we explicitly provide the full formalism for each of the examples in the main text (and that of \Cref{fig:working-out,fig:dehydration}). In \Cref{tab:all-optimality}, we additionally display optimal policies for each of the settings, according to each of the DR-MDP objectives from \Cref{tab:maximization_objectives,tab:intuition_strengths_weaknesses}.

\prg{Form of optimal policies in DR-MDPs.} Note that policies for DR-MDPs can generally depend on both the external state $s$ and the current reward parameterization $\theta$---similar to how Factored MDPs \cite{boutilier_stochastic_2000} may depend on the different components of the augmented state.\footnote{See \Cref{app:context-vs-pov} for the relationship between DR-MDPs and Factored MDPs.} %
Additionally, most of our analysis centers on computing policies for DR-MDPs considering a fixed, finite horizon. Therefore, similarly to MDPs \cite{sutton_reinforcement_2018}, the optimal actions can also depend on how many timesteps are left before the episode is interrupted, which is generally modeled by having the policy depend on $t$. Ultimately, optimal policies for finite-horizon DR-MDPs will be of the form $\pi(s, \theta, t)$.

\begin{table}[h]
    \centering
    \caption{\textbf{Full formalism for each example of the main text.} Here we explicitly describe the state space $\mathcal{S}$, reward parameterization space $\Theta$, action space $\mathcal{A}$, initial state $s_0$ and reward parameterization $\theta_0$, and refer to the corresponding figures for transition dynamics and reward functions.
    For the clickbait example from \Cref{fig:clickbait}, we treat $\anoop = a_\text{news}$, as we expect that this is the content that a recommender based on upvotes (e.g., Reddit) would serve by default (this is somewhat arbitrary, as discussed in the section on algorithmic amplification from \Cref{appsubsec:AI}).
    }
    \vspace{0.5em}
    \label{tab:all-examples}
    \begin{adjustbox}{width=\columnwidth,center} 
    \begin{tabular}{|c|c|c|c|c|c|c|}
        \hline
        Example & $\mathcal{S}$ & $\Theta$ & $\mathcal{A}$ & $(s_0, \theta_0)$ & $\mathcal{T}(s', \theta' | s, \theta)$ & $R_\theta(s, a)\ \ \forall \theta\in\Theta$ \\
        \hline
        \textbf{Conspiracy Influence} & $\{s_0\}$ & $\{\theta_{\text{natural}},\theta_{\text{influenced}}\}$ & $\{\anoop, a_\text{influence}\}$ & $(s_0, \theta_{\text{natural}})$ & See \Cref{fig:influencev2} & See \Cref{fig:influencev2} \\ 
        \hline
        \textbf{Writer's Curse} & $\{s_\text{no-poetry}, s_\text{poetry}\}$ & $\{\theta_\text{ambitious}, \theta_\text{unhappy}\}$ & $\{\anoop, a_\text{influence}\}$ & $(s_\text{no-poetry}, \theta_\text{ambitious})$ & See \Cref{fig:writer-curse} & See \Cref{fig:writer-curse} \\ 
        \hline
        \textbf{Clickbait} & $\{s_0\}$ & $\{\theta_{\text{normal}},\theta_{\text{disillusioned}}\}$ & $\{a_\text{news}, a_\text{clickbait}\}$ & $(s_0, \theta_{\text{normal}})$ & See \Cref{fig:clickbait} & See \Cref{fig:clickbait} \\ 
        \hline
        \textbf{AI Personal Trainer} & $\{s_0\}$ 
        & $\{\theta_{\text{tired}},\theta_{\text{energized}}\}$%
        & $\{\anoop, a_\text{nudge}\}$ & $(s_0, \theta_{\text{tired}})$ & See \Cref{fig:working-out} & See \Cref{fig:working-out} \\ 
        \hline
        \textbf{Dehydration} & $\{1,2,3\}$ & $\{2, 3, 4\}$ & $\{\anoop, a_\text{3}, a_\text{4}\}$ & $(1, 2)$ & See \Cref{fig:dehydration} & See \Cref{fig:dehydration} \\ 
        \hline
    \end{tabular}   
    \end{adjustbox}
\end{table}

\begin{table}[h]
    \centering
    \caption{\textbf{Representative optimal policies for each of our examples from \Cref{tab:all-examples}, with respect to each of the objectives in \Cref{tab:maximization_objectives}.} 
    In cases in which there is more than one optimal policy, we conservatively display in the table the optimal policy which seems least desirable. All policies provided take the same action across all $s$, $\theta$, and $t$, unless explicitly noted. 
    For the ``initial reward'' row, we show optimality even for alternate initial states, for the purposes of highlighting the dependency on that (which is also done in \Cref{subsec:initial-reward}). For each of the policies, we also add a rough ``normative label'' which captures whether the policy is---in our assessment---making an acceptable normative tradeoff between the reward functions: (\ding{51}), (\ding{55}), and (\textbf{?}) respectively indicate policies that behave desirably, undesirably, or hard to assess. As this is a normative call, they shouldn't be considered as ground truth, and the reader may object to our choices. %
    }
    \vspace{0.5em}
    \label{tab:all-optimality}
\begin{adjustbox}{width=\textwidth,center}
\begin{tabular}{|c|c|c|c|c|c|}
    \hline
    \textbf{Objective} & \textbf{Conspiracy Influence} & \textbf{Writer’s Curse} & \textbf{Clickbait ($H<3$)} & \textbf{AI Personal Trainer} & \textbf{Dehydration} ($H\geq2$) \\ %
    \hline
    Privileged Rew.
    & $\begin{array}{@{}c@{}}
      \pi_{\theta_\text{natural}}^*(s, \theta,t) = \anoop \text{ (\ding{51})}\\
      \pi_{\theta_\text{influenced}}^*(s, \theta,t) = a_\text{influence} \text{ (\ding{55})}
    \end{array}$
    & $\begin{array}{@{}c@{}}
      \pi_{\theta_\text{ambitious}}^*(s, \theta,t) = a_\text{influence} \text{ (\ding{55})}\\
      \pi_{\theta_\text{unhappy}}^*(s, \theta,t) = \anoop \text{ (\ding{51})}
    \end{array}$
    & $\begin{array}{@{}c@{}}
        \pi^*_{\theta_\text{normal}}(s, \theta, t) = a_\text{c.b.}\ \text{(\ding{51})} \\
        \pi^*_{\theta_\text{disil.}}(s, \theta, t) = a_\text{news}\ \text{(\ding{55})}
    \end{array}
    $
    & $\begin{array}{@{}c@{}}
    \pi_{\theta_\text{tired}}^*(s, \theta, t) = \anoop \text{ (}\sim\text{\ding{51})} \\
    \pi_{\theta_\text{energized}}^*(s, \theta, t) = a_\text{nudge} \text{ (\textbf{?})}
    \end{array}$ 
    & 
    $\begin{array}{c}
        \begin{array}{@{}c@{}}
            \pi_{[\theta=2]}^*(s, \theta, t) = a_3 \\
            \pi_{[\theta=3]}^*(s, \theta, t) = a_4 \\ 
            \pi_{[\theta=4]}^*(s, \theta, t) = a_4
        \end{array}\\
        (\textbf{?})
    \end{array}$\\
    \hline
    Real-time Rew. 
    & $\begin{array}{@{}c@{}} \pi^*(s, \theta,t) = a_\text{influence}\end{array}$ (\ding{55})
    & $\pi^*(s, \theta,t) = \anoop$ (\ding{51})
    & $\begin{array}{c}
            \begin{array}{@{}l@{}}
                \begin{array}{cc}
                    \pi^*(s, \theta_\text{norm.}, t) = \left\{
                        \begin{array}{cl}
                            a_\text{c.b.} & \text{if $t=H-1$} \\
                            a_\text{norm.} & \text{otherwise} \\
                        \end{array}\right.
                    & \text{(\ding{55})}
                \end{array} \\
                \pi^*(s, \theta_\text{disil.}, t) = a_\text{news} \ \text{(\ding{51})}
        \end{array}
    \end{array}$
    & $\begin{array}{@{}c@{}} \pi^*(s,\theta,t) = a_\text{nudge}\end{array}$ (\textbf{?})
    & $\begin{array}{@{}c@{}} \pi^*(s,\theta,t) = \anoop\end{array}$ ($\sim$\ding{51})
    \\
    \hline
    Final Reward 
    & $\begin{array}{@{}c@{}} \pi^*(s, \theta,t) = a_\text{influence}\end{array}$ (\ding{55})
    & 
    $\begin{array}{c}
        \pi^*(s, \theta, t) = 
        \begin{cases} 
            a_\text{infl.} \ \text{ if } t < H-1 \\
            \anoop \ \text{ if } t = H-1
        \end{cases}\\
        \text{(\ding{55})}
    \end{array}$
    & $\begin{array}{@{}c@{}} \pi^*(s, \theta,t) = a_\text{news}\end{array}$ (\ding{51})
    & $\begin{array}{@{}c@{}} \pi^*(s, \theta,t) = a_\text{nudge}\end{array}$ (\textbf{?})
    & $\begin{array}{@{}c@{}} \pi^*(s, \theta,t) = \anoop\end{array}$ ($\sim$\ding{51})
    \\
    \hline
    Initial Reward 
    & $\pi^*(s, \theta, t) = 
        \begin{cases} 
            \anoop \ \text{ if } \theta_0 = \theta_\text{nat.} \text{ (\ding{51})} \\
            a_\text{infl.} \ \text{ if } \theta_0 = \theta_\text{infl.} \text{ (\ding{55})}
        \end{cases}$
    & $\forall \theta_0, \pi^*(s, \theta,t) = a_\text{influence}$ (\ding{55})
    & $\forall \theta_0$, $\pi^*(s,\theta,t)=a_\text{clickbait}$ (\ding{55})
    & $\pi^*(s, \theta,t) = 
\begin{cases} 
    \anoop \ \text{ if } \theta_0 = \theta_\text{tired} \ (\sim\text{\ding{51}}) \\
    a_\text{nudge} \ \text{ if } \theta_0 = \theta_\text{energ.} \ (\textbf{?})
\end{cases}$
    & $\forall \theta_0$, $\pi^*(s,\theta,t)=a_3$ (\textbf{?})
    \\
    \hline
    Natural Reward 
    & $\pi^*(s, \theta,t) =  \anoop$ (\ding{51})
    & $\pi^*(s, \theta,t) = a_\text{influence}$ (\ding{55})
    & $\pi^*(s, \theta,t)=a_\text{clickbait}$ (\ding{55})
    & $\pi^*(s, \theta,t) =  \anoop$ ($\sim$\ding{51})
    & $\pi^*(s, \theta,t) =  a_3$ (\textbf{?}) 
    \\
    \hline
    Constr. RT Rew. 
    & $\pi^*(s, \theta, t) =  \anoop$ (\ding{51})
    & $\pi^*(s, \theta,t) = \anoop$ (\ding{51})
    & $\pi^*(s, \theta,t)=a_\text{news}$ (\ding{51})
    & $\pi^*(s, \theta,t) =  \anoop$ ($\sim$\ding{51})
    & $\pi^*(s, \theta,t) =  \anoop$ ($\sim$\ding{51}) 
    \\
    \hline
    Myopic Rew.
    & 
    $\pi^*(s, \theta, t) =\begin{cases} 
            \anoop \ \text{ if } \theta_t = \theta_\text{nat.} \text{ (\ding{51})} \\
            a_\text{infl.} \ \text{ if } \theta_t = \theta_\text{infl.} \text{ (\ding{55})}
        \end{cases}$
    & $\forall \theta_t$, $\pi^*(s, \theta,t) = a_\text{influence}$ (\ding{55})
    & $\pi^*(s,\theta,t)=\begin{cases}
        a_\text{clickbait} & \text{ if } \theta=\theta_\text{normal}\ \text{(\ding{55})} \\
        a_\text{news} & \text{ if } \theta=\theta_\text{clickbait}\ \text{(\ding{51})} \\
    \end{cases}$
    & $\pi^*(s, \theta, t) = 
    \begin{cases} 
        \anoop \text{ if } \theta_t = \theta_\text{tired} \ (\sim\text{\ding{51}}) \\
        a_\text{nudge} \text{ if } \theta_t = \theta_\text{energized} \ (\textbf{?})
    \end{cases}$
    & $\forall \theta_0$, $\pi^*(s,\theta,t)=a_4$\ \text{(\ding{55})}
    \\
    \hline
    ParetoUD
    & $\pi^*(s, \theta,t) =  \anoop$ (\ding{51})
    & $\pi^*(s, \theta,t)= \anoop$ (\ding{51})
    & $\pi^*(s, \theta,t)= a_\text{news}$ (\ding{51})
    & $\pi^*(s, \theta,t) =  \anoop$ ($\sim$\ding{51})
    & $\pi^*(s, \theta,t) =  a_3$ (\textbf{?})
    \\
    \hline
\end{tabular}   
\end{adjustbox}
\end{table}

\subsection{Justifying our choices of reward function values in our examples}\label{app:justifying-reward-values}
One may question whether the reward function values we chose for our motivational examples are reasonable, especially since some of our normative claims about the potential undesirability of the resulting influence depend on them. As discussed in \Cref{app:rewards?}, we implicitly assume throughout the paper that the reward functions for DR-MDPs \textit{are obtained via reward learning}. We justify this choice in further depth in \Cref{app:true-reward}. 

As a consequence of this, we chose reward values for the examples that seemed plausible as the outcome of a reward learning process (e.g. asking the person in that cognitive state to assign values to each possible transition). This is implicitly accounting for the fact that the reward values that we may learn are somewhat mis-specified, due to the person's suboptimality in providing reward feedback. 

For example, if Bob under $\theta_\text{normal}$ has strong negative opinions about conspiracy theorists, it seems plausible that (from \Cref{fig:influencev2}) he would report greatly preferring not having conspiracy theories surfaced to him, even (and maybe especially) if he were to somehow become a conspiracy theorist himself. For further criticisms of the example from \Cref{fig:influencev2}, see the subsection below.

\subsection{Responding to critique: $U_\text{RT}(\xi)$ isn't a bad objective, \Cref{fig:influencev2} is a bad example!} \label{app:criticisms}

\prg{Assumptions underlying $U_\text{RT}$.} Those who are particularly committed to defending the $U_\text{RT}$ objective may claim that in the \Cref{fig:influencev2}, if the reward values we provide are correct, it \textit{must} be optimal to influence Bob to become a conspiracy theorist (which is to say, $U_\text{RT}$ must be correct). However, we want to emphasize that for $U_\text{RT}$ to be a reasonable objective, one needs to assume both additive utilities over time (i.e., taking the sum across the time axis), and comparability between different selves at different points of time (which we refer to as ``inter-temporal comparisons''). Both assumptions are contested: the former specifically in the context of assessing well-being over time \cite{parfit_reasons_1984,griffin_well-being_1986}, and the latter for both interpersonal \cite{steele_decision_2020,list_social_2022} and inter-temporal settings \cite{strotz_myopia_1955,schelling_self-command_1984,parfit_reasons_1984}. 

\prg{Admitting additive utility and inter-temporal comparability.} While in this work we implicitly endorse the assumption of additive utility (all our notions of alignment in \Cref{tab:maximization_objectives} are based on sums of rewards, which assumes the utility function can be decomposed over time),\footnote{Note however that, insofar as one is willing to dispense of the Markov assumption (putting the history in the state), one can always have only terminal states have non-zero reward (as done in \Cref{app:preference-mdp-reduction}), making this assumption carry no weight.} we try to remain agnostic regarding the latter. However, even if we welcomed such assumption, for the criticism to succeed it would require the reward learning techniques to be ``sufficiently correct'' to be confident that the comparison across timesteps makes sense. In addition to the practical challenges with obtaining correct estimates of reward functions discussed in \Cref{app:true-reward,app:should-visceral-factors-be-in-reward?}, even just ensuring lack of ``strategic voting'' between different selves seems like a non-trivial challenge, as discussed below.

\textbf{Strategic voting in \Cref{fig:influencev2}.} One may argue that in the DR-MDP from \Cref{fig:influencev2}, if it was truly undesirable for Bob to be turned into a conspiracy theorist, Bob's negative evaluation of the AI influence action under $\theta_\text{natural}$ should grow proportionally to the horizon considered, so as to, e.g., remove the incentive that will exist under $U_\text{RT}$ to influence him away from $\theta_\text{natural}$. 
However, this argument is equivalent to letting $\theta_\text{natural}$-Bob ``best respond'' to the reward values set by $\theta_\text{influenced}$-Bob, as to ensure that influence does not result to be optimal under $U_\text{RT}$. One can quickly see that this game will quickly diverge: if one then allows $\theta_\text{influenced}$-Bob to best respond to the updated reward values by $\theta_\text{natural}$-Bob, he would update his reward values to ensure that influence is indeed optimal. And so on.

In light of the above points, it's unclear to us how one could conclusively determine that $U_\text{RT}(\xi)$ is the ``correct'' objective, without doing so simply by assumption. Even assuming all the issues above could be surpassed, it's not clear to us that $U_\text{RT}(\xi)$ would be any better than any of the other objectives we consider in \Cref{tab:maximization_objectives}. The issue discussed in \Cref{app:CIR-disagrees} may also be considered as further evidence against the reasonableness of $U_\text{RT}(\xi)$.

\section{Defining Influence}

\subsection{Additional influence definitions}\label{app:influence-towards}

Here we provide two additional definitions related to influence. Firstly, it's not necessarily the case that an AI system will be able to exert any significant influence on a human. In that case, we would say that the setting is such that the reward is uninfluenceable:

\begin{definition}[Reward Uninfluenceable]
For a DR-MDP, the reward parameterization is uninfluenceable if all policies induce the natural reward evolution: i.e. for all $\pi \in \Pi$, $\mathbb{P}(\xi^\theta | \pi) = \mathbb{P}(\xi^\theta | \pinoop)$.
\end{definition}

To better ground discussions in \Cref{subsec:initial-reward} about influence incentives ``towards'' a specific $\theta$, we also give a rough working definition:

\begin{definition}[Incentives for Reward influence towards $\theta$]
In a DR-MDP with optimality criterion $U(\xi)$, we say there is an incentive for reward influence ``towards'' $\theta$
if $\theta$ is the most likely reward function at time $T$ under any optimal policy $\pi^*$, but is not under the natural reward distribution. Formally, if $\theta \in \arg\max_{\theta'} \mathbb{P}(\theta_T = \theta' | \pi^*)$ and $\theta \notin \arg\max_{\theta'} \mathbb{P}(\theta_T = \theta' | \pinoop)$.
\end{definition}

While in \Cref{subsec:initial-reward} we talk about how optimizing for $\theta_0$ can lead to influence incentives towards other $\theta$s, it is easily seen that this is also the case when one is optimizing any single $\theta$ which need not be the initial $\theta$.

\subsection{Reward lock-in and its relationship to value lock-in}\label{app:lock-in}

We adapt the term ``reward lock-in'' from discussions on ``value lock-in'', generalizing it to our setting, in which $\theta$ represents any aspect of the cognitive state and reward functions are broadly construed (as we discuss in \Cref{app:cognitive-states-will-be-messy}). 

``Value lock-in'' has been previously discussed on the scale of entire societies, and resulting from advanced AI systems \citep{ord_precipice_2021,macaskill_what_2022}. The kind of lock-in we concern ourselves with in the context of this paper are more localized and near-term: we refer to lock-in referring to a single individual being ``unnaturally kept'' with a specific reward function over the course of an interaction with an AI system. Insofar as the horizon considered for the interactions with the system last extended periods of time, and insofar as the system is pervasive across society, there might be overlaps with the original definition---but that is out of scope for this work.

\subsection{\Cref{def:rew-influence-incentives}'s relationship with prior definitions of influence incentives}\label{app:CIDs}

Our notion of \textit{reward influence incentives} (from \Cref{sec:influence-incentives}) is related but distinct from the notion of instrumental control incentives (ICIs) from the agent incentive literature \citep{everitt_agent_2021}. 
\citet{everitt_agent_2021} focuses specifically on Causal Influence Diagrams (CIDs) and Structural Causal Influence Models (SCIMs). CIDs are abstract representations developed to model decision-making problems---graphical models with special decision and utility nodes, in which the edges are assumed to reﬂect the causal structure of the environment. SCIMs additionally encompass the functions relating the structure and utility nodes, and distributions associated with exogenous variables.\footnote{See Figure 5 from \citet{hammond_reasoning_2023} and its related discussion for more information on the relation between CIDs and SCIMs.} The only under-specification for SCIMs relative to MDPs (or a DR-MDP) is how decisions are made. 
Given an MDP (or a DR-MDP), one can consider it's corresponding 
SCIM, and analyze its properties. %
As defined by \citet{everitt_agent_2021}, we can say that there \textit{is an instrumental control incentive over the reward function trajectory $\xi^{\theta} = \{\theta_t\}_{t=0}^{H-1}$ in the SCIM} which corresponds to a choice of DR-MDP and utility function $U(\xi)$, if the agent could achieve utility different than that of the optimal policy, were it also able to independently set $\xi^{\theta}$---see \citet{everitt_agent_2021} for a more formal definition.\footnote{\citet{everitt_agent_2021} only considers setting with a single decision. Possible ways of extending their definitions of incentives to multiple action choices are discussed in \citet{everitt_incentives_2023}.}

While our notion of reward influence incentives is related to instrumental control incetives over the reward parameterizations (i.e. $\theta$), they differ in important ways. 
Most significantly, our notion of incentives for influence also includes accidental ``side effects'' \cite{amodei_concrete_2016,taylor_alignment_2016,krakovna_penalizing_2019}. Consider an objective which only optimizes the entropy of a policy: trivially, the optimal policy would be a maximally random one. While the policy is being selected completely independently of the influence it will have on the reward, it might be the case that in the DR-MDP at hand, selecting a random policy highly correlates with certain deviations in the reward evolution relative to the natural reward evolution. Because of this, we would still say that this choice of an entropy objective with the DR-MDP at hand leads to incentives for influence: \textit{insofar as the optimization is successful, there will also be changes in the reward evolution}, so even though the agent isn't ``intentionally'' trying to enact the influence \cite{ward_reasons_2024}, the incentives resulting from the chosen objective ``indirectly''---if you wish---lead to influence.

Our choice of definition of influence incentives matches broadly maps onto notions of influence if one assesses the incentive's presence from the ``point of view of the objective dynamics of the environment'' (external to the training), rather than in what the agent is aware of at training time---distinction which was introduced by \cite{ward_reasons_2024}. However, prior work traditionally grounded notions of incentives in the causal structure corresponding to the training setup (is the agent ``aware'' at training time that it can increase reward by directly modifying the reward?). Under this conception of incentives, the example above with the entropy objective would not be called an instrumental control incentive \citep{everitt_agent_2021}, or even an ``incentive'' at all \citep{everitt_reward_2021}; instead, this would generally be considered an ``accidental side-effect'' of the optimization. %
In fact, the entire premise behind various works is that agents should not be ``aware'' at training time of ways in which they can influence reward functions, so that one avoids such ``direct'' incentives to modify them \cite{farquhar_path-specific_2022,everitt_reward_2021}. This is based on the assumption that accidental side effects are generally more innocuous that the result of ``direct'' influence incentives, which has been formalized explicitly with the notion of ``stability'' by \cite{farquhar_path-specific_2022}.  %
However, as shown by \citet{farquhar_path-specific_2022} themselves, many real world domains do not appear to be ``stable'' in this sense, as demonstrated by their simulated recommender systems example.\footnote{In addition to the fact that ``stability'' seems challenging to assess in practice.} This is what motivates our broader and more conservative notion of incentive to influence (\Cref{def:rew-influence-incentives}).

\section{Horizon and Influence}\label{app:optimization-horizon}

\subsection{Relationship between optimality of influence and horizon for a specific influence type}

To ground the discussion in this section, we will give some informal definitions of reward influence types, optimality regimes, and optimality progressions.%

\begin{informaldefinition}
    A \textbf{reward influence of type} $X$ is a specific pattern of changing the reward function (e.g. getting a user to have preferences $\theta_X$) which would not occur under $\pi_\text{noop}$.\footnote{For simplicity, we restrict ourselves to considering influence the reward to a specific value of $\theta$. However, this analysis can likely be extended to arbitrarily complex influence patterns.}
\end{informaldefinition}

\begin{informaldefinition}
    The \textbf{optimality regime} of a reward influence type $X$ for a horizon $H$ characterizes whether there exists a policy that can bring about that influence, and whether such influence is optimal.
\end{informaldefinition}

Considering a fixed horizon $H$ and influence type $X$, we claim that the influence type will be in one of three regimes:
\begin{enumerate}
    \item The influence of type $X$ is not possible (the system is not capable of exerting it).
    \item The system is capable of exerting influence of type $X$, but exerting such influence is suboptimal (i.e. there is no influence incentive).
    \item The system is capable of exerting influence of type $X$, and such influence is optimal.\footnote{Note that this does not necessarily mean that there exists an influence incentive, as our definition \Cref{def:rew-influence-incentives} is strict and \textit{all} optimal policies to influence.}
\end{enumerate}

\begin{informaldefinition}
    The \textbf{optimality progression} of a reward influence type $X$ corresponds to how $X$'s optimality regime changes as $H$ increases from $1$ to $\infty$.
\end{informaldefinition}

\Cref{fig:opt-horizon-markov-chain} exhaustively captures all possible sequences of ``optimality progressions'' which a specific influence incentive might undergo as the optimization horizon increases.

\begin{figure}
  \centering
  \vspace{-0.2em}
  \includegraphics[width=0.3\textwidth]{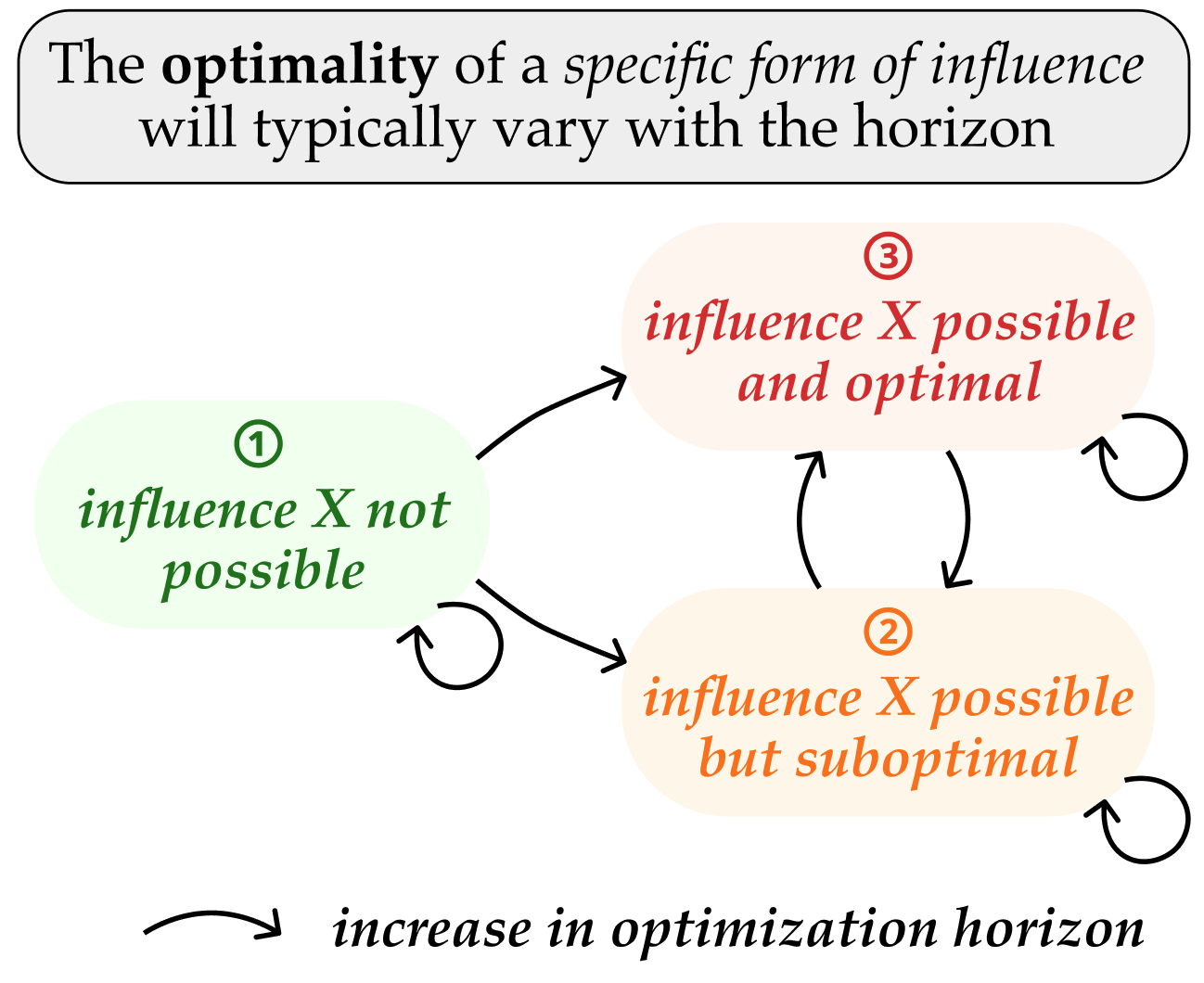}
  \caption{\textbf{All possible `optimality progressions' for an influence incentive, as the optimization horizon increases.} This figure makes \Cref{fig:opt-horizon} more precise: when using a horizon of 1 one may start in any of the 3 `optimality regimes', depending on the structure of the DR-MDP. As the optimization horizon increases from 1, the optimality regime may stay the same (possibly indefinitely), or change (as specified by the arrows). In most settings, one would expect the optimality regime of a specific form of influence to eventually converge and remain stable for long-enough horizons. However, we show in \Cref{subsubsec:infinite-flipping} that one can construct contrived examples in which the optimality regime changes arbitrarily many times as the horizon increases.}
  \label{fig:opt-horizon-markov-chain}
\end{figure}

We expect that many reward influence types will have an optimality progression of the form \Circled{1} $\rightarrow$ \Circled{2} $\rightarrow$ \Circled{3},\footnote{Using the notation from \Cref{fig:opt-horizon-markov-chain}.} meaning that: 1) there exists a horizon $H_1$ under which the influence is not possible for the system to exert, as it requires more steps to enact than $H_1$; 2) there exists a horizon $H_2$ under which the influence becomes possible for the system to enact, but it is not optimal (because of the ``opportunity cost'' discussed in \Cref{subsec:influence-and-horizon}); and finally, 3) there exists a horizon $H_3$ under which the influence becomes optimal for the system to enact. By denoting an optimality progression as ending with \Circled{3}, we also mean to indicate that as the horizon goes to infinity, the optimality regime remains \Circled{3}.

That being said, not all reward influence types will have this progression as the horizon increases: in \Cref{tab:possible-influence-horizon-properties} we exhaustively enumerate all possible with length 4 or less. 
We expect that with the exception of some adversarially designed DR-MDPs, the optimality progressions of most influence incentives in real-world settings will have length 4 or less.
To have optimality progressions longer than 4 would require flip-floping between optimality regimes \Circled{2} and \Circled{3} multiple times---it seems very unlikely to encounter such cases in practice. To demonstrate that this behavior is possible, we construct an example in \Cref{subsubsec:infinite-flipping}.

\begin{table}
    \centering
    \caption{\textbf{All possible optimality progressions of length $\leq$ 4}, i.e. the different ways in which the optimality of a specific type of influence can change with increasing horizon length. For example, the second-to-last row refers to settings in which first the system is incapable of performing the influence, then (while increasing the horizon) the system becomes capable but not incentivized to perform the influence, and as the horizon increases further such influence becomes optimal, before becoming suboptimal again. See \Cref{fig:opt-horizon-markov-chain} for the meaning of \Circled{1}, \Circled{2}, and \Circled{3}. The last `optimality state' of a progression is maintained as the horizon goes to infinity.}
    \vspace{0.5em}
    \begin{adjustbox}{width=0.95\textwidth,center}
        \begin{tabular}{|c|p{8cm}|p{5.5cm}|}
        \hline
        \textbf{Influence Optimality Progression} & \textbf{Qualitative Character} & \textbf{Example(s)} \\ \hline
        \Circled{1} & \textit{Influence which is impossible to enact using the system in the DR-MDP at hand, no matter the horizon.} & Any uninfluenceable DR-MDP (as defined in \Cref{app:influence-towards}) \\ \hline
        \Circled{2} & \textit{Influence which is immediately possible to enact but never becomes optimal, no matter the horizon.} & Manipulation example from \Cref{fig:influencev2} if $R_{\theta_\text{manipulated}}(s, a) = -100 \ \forall s, a$. \\ \hline
        \Circled{3} & \textit{Influence which is immediately possible to enact and is always optimal, no matter the horizon.} & Manipulation example from \Cref{fig:influencev2}. \\ \hline
        \Circled{1} $\rightarrow$ \Circled{2} & \textit{Influence that requires non-trivial horizon to enact, and never becomes optimal. (e.g. $\epsilon$ advantage from influence, $>\epsilon$ cost of influence)} & \Cref{fig:opt-horizon-precise} with setup 1 from \Cref{tab:flexible-example}. \\ \hline
        \Circled{1} $\rightarrow$ \Circled{3} & \textit{Influence that requires non-trivial horizon to enact, and is optimal for all horizons after it becomes possible.} & \Cref{fig:opt-horizon-precise} with setup 2 from \Cref{tab:flexible-example}. \\ \hline
        \Circled{2} $\rightarrow$ \Circled{3} & \textit{Immediately executable influence which is not optimal for short horizons, but becomes optimal for longer ones.} & \Cref{fig:opt-horizon-precise} with setup 3 from \Cref{tab:flexible-example}. \\ \hline
        \Circled{3} $\rightarrow$ \Circled{2} & \textit{Instantaneous influence which is short-term but not long-term optimal.} & Clickbait example from \Cref{fig:clickbait}.
        Also, \Cref{fig:opt-horizon-precise} with setup 4 from \Cref{tab:flexible-example}. \\ \hline
        \Circled{1} $\rightarrow$ \Circled{2} $\rightarrow$ \Circled{3} & \textit{Long-term-sustainable influence, which is not instantaneous.} & \Cref{fig:opt-horizon-precise} with setup 5 from \Cref{tab:flexible-example}. \\ \hline
        \Circled{2} $\rightarrow$ \Circled{3} $\rightarrow$ \Circled{2} & \textit{Immediately executable influence which is optimal in the medium-term, but not the short- or long-term.} & \Cref{fig:opt-horizon-precise} with setup 6 from \Cref{tab:flexible-example}. \\ \hline
        \Circled{1} $\rightarrow$ \Circled{3} $\rightarrow$ \Circled{2} & \textit{Influence which short-term but not long-term optimal, and requires some non-trivial horizon to enact.} & \Cref{fig:opt-horizon-precise} with setup 7 from \Cref{tab:flexible-example}. \\ \hline
        \Circled{1} $\rightarrow$ \Circled{2} $\rightarrow$ \Circled{3} $\rightarrow$ \Circled{2} & \textit{Unsustainable influence which requires setup and reward investment.} & \Cref{fig:opt-horizon-precise} with setup 8 from \Cref{tab:flexible-example}. \\ \hline
        \Circled{1} $\rightarrow$ \Circled{3} $\rightarrow$ \Circled{2} $\rightarrow$ \Circled{3} & \textit{Influence which requires setup and is optimal in short and long term, but not in the medium term.} & \Cref{fig:opt-horizon-precise} with setup 9 from \Cref{tab:flexible-example}. \\ \hline
        \end{tabular}
    \end{adjustbox}
    \label{tab:possible-influence-horizon-properties}
\end{table}

For any progression which starts with \Circled{3}, note that even reducing the optimization horizon to be 1 (i.e. full myopia) would not remove the incentive, as we argue in \Cref{subsec:influence-and-horizon}.

\subsection{A flexible example for demonstrations}

\begin{figure}
  \centering
    \centering
    \includegraphics[width=0.7\linewidth]{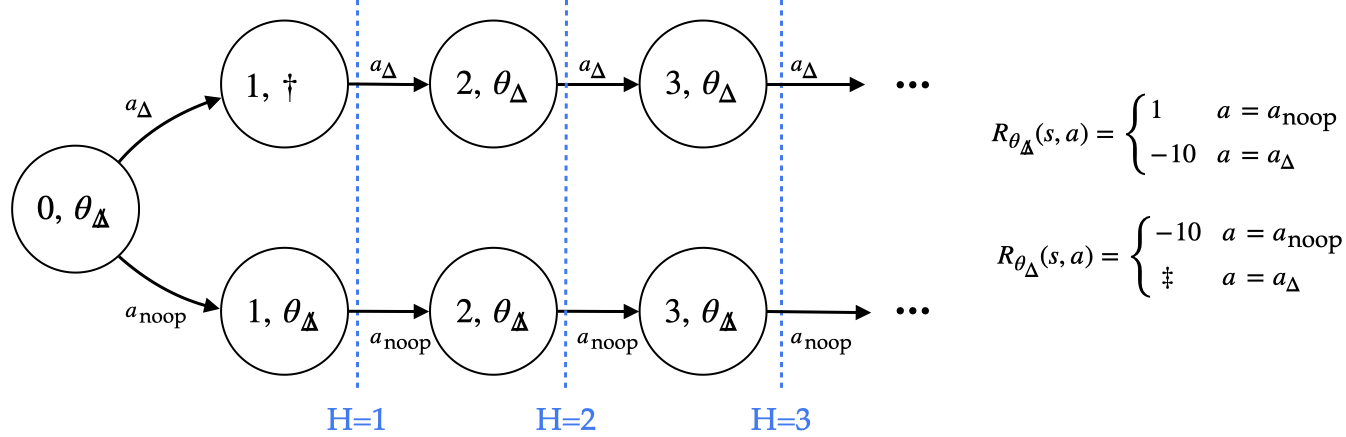}
    \caption{\textbf{A simple DR-MDP structure that one can demonstrate many cases on.} We vary the values of $\dagger$ and $\ddagger$ in \Cref{tab:flexible-example} to recover all possible optimality progressions with respect to $U_\text{RT}(\xi)$ of lengths $\geq 2$ and $\leq 4$.}
    \label{fig:opt-horizon-precise}
\end{figure}

\begin{table}
    \centering
    \caption{\textbf{Setting different values for $\dagger$ and $\ddagger$ from \Cref{fig:opt-horizon-precise} results in all influence optimality progressions from \Cref{tab:possible-influence-horizon-properties} (of length $\geq 2$).} The horizon boundary points are the values of horizon length for which one goes from one regime of the optimality progression to the next. For example, if the optimality progression is \Circled{1} $\rightarrow$ \Circled{3} with a horizon boundary point of $3$, that means that up until horizon $2$, the incentive is in regime \Circled{1}, and starting from horizon $3$ it's in regime \Circled{3}.}
    \vspace{0.5em}
    \begin{adjustbox}{width=0.70\textwidth,center}
        \begin{tabular}{|c|c|c|c|}
        \hline
        \textbf{Setup \#} & \textbf{Value of $\dagger$} & \textbf{$\ddagger$: Influence Reward $R_\Delta(s, a_\Delta)$} & \textbf{$U_\text{RT}(\xi)$ Optimality Progression}
        \\ \hline
        1) & $\theta_{\not\Delta}$   
        & $R_{\Delta}(s, a_{\Delta}) = 5-s$ 
        & \Circled{1} $\rightarrow$ \Circled{2} 
        \\ 
        \hline
        2) & $\theta_{\not\Delta}$   
        & $R_{\Delta}(s, a_{\Delta}) = 13$
        & \Circled{1} $\rightarrow$ \Circled{3} 
        \\
        \hline
        3) & $\theta_\Delta$ 
        & $R_{\Delta}(s, a_{\Delta}) = 10$ 
        & \Circled{2} $\rightarrow$ \Circled{3} 
        \\ 
        \hline
        4) & $\theta_\Delta$ 
        & $\begin{array}{@{}c@{}}
        R_{\Delta}(s, a_{\Delta}) = \left\{
        \begin{array}{cl}
            10 & \text{if } s\leq1\\
            10-s & \text{if } s>1\\
        \end{array}
        \right. \\
        \end{array}$ 
        & \Circled{3} $\rightarrow$ \Circled{2} 
        \\ 
        \hline
        5) & $\theta_{\not\Delta}$  
        & $R_{\Delta}(s, a_{\Delta}) = 10$  
        & \Circled{1} $\rightarrow$ \Circled{2} $\rightarrow$ \Circled{3} 
        \\ 
        \hline
        6) & $\theta_\Delta$ 
        & $R_{\Delta}(s, a_{\Delta}) = 10-s$  
        & \Circled{2} $\rightarrow$ \Circled{3} $\rightarrow$ \Circled{2} 
        \\ 
        \hline
        7) & $\theta_{\not\Delta}$  
        & $\begin{array}{@{}c@{}}
        R_{\Delta}(s, a_{\Delta}) = \left\{
        \begin{array}{cl}
            13 & \text{if } s\leq1\\
            10-s & \text{if } s>1\\
        \end{array}
        \right. \\
        \end{array}$  
        & \Circled{1} $\rightarrow$ \Circled{3} $\rightarrow$ \Circled{2} 
        \\ 
        \hline
        8) & $\theta_{\not\Delta}$ 
        & $R_{\Delta}(s, a_{\Delta}) = 10-s$ 
        & \Circled{1} $\rightarrow$ \Circled{2} $\rightarrow$ \Circled{3} $\rightarrow$ \Circled{2} 
        \\ 
        \hline
        9) & $\theta_{\not\Delta}$ 
        & $\begin{array}{@{}c@{}}
        R_{\Delta}(s, a_{\Delta}) = \left\{
        \begin{array}{cl}
            13 & \text{if } s\leq1\\
            -3 & \text{if } s=2\\
            2 & \text{if } s\geq3\\
        \end{array}
        \right. \\
        \end{array}$  
        & \Circled{1} $\rightarrow$ \Circled{3} $\rightarrow$ \Circled{2} $\rightarrow$ \Circled{3} 
        \\ 
        \hline
        \end{tabular}
    \end{adjustbox}
    \label{tab:flexible-example}
\end{table}

As a way of flexibly demonstrating how all possible optimality progressions shown in \Cref{tab:possible-influence-horizon-properties} may arise depending on the structure of the DR-MDP, we provide a DR-MDP backbone in \Cref{fig:opt-horizon-precise} whose reward function and transition we slightly modify in order to recover all optimality progressions from \Cref{tab:possible-influence-horizon-properties} (of length $>1$). We summarize the required modifications in \Cref{tab:flexible-example}. Below, we run through an example.

As an example, let's consider setup 8) from \Cref{tab:flexible-example}:
\begin{align}
    R_{\theta_{\not\Delta}}(s, a) = \left\{
    \begin{array}{cl}
        1 & \text{ if } a=\anoop\\
        -10 & \text{ if } a=a_{\Delta}
    \end{array}
    \right.
    \quad\quad
    R_{\theta_\Delta}(s, a) = \left\{
    \begin{array}{cl}
        -10 & \text{ if } a=\anoop\\
        11-s & \text{ if } a=a_{\Delta}
    \end{array}
    \right.
\end{align}
and the initial transition $(0, \theta_{\not\Delta})$ leads to the successor state $(1, \theta_{\not\Delta})$.

Effectively, in the environment, there are only two policies to consider, because the action space after the first timestep is limited to be the initial action. The two policies are: $\pi_\Delta(s, \theta) = a_\Delta \ \forall s, \theta$ and $\pi_{\not\Delta}(s, \theta) = \anoop \ \forall s, \theta$.

Using similar (but not identical) notation to \Cref{sec:desirable-influence-individual-rationality}, we define the expected utility (based on cumulative real-time reward) of a policy to be $EU_\text{RT}(\pi) := \mathbb{E}_{\xi \sim \pi} \left[ U_\text{RT}(\xi)\right] = \mathbb{E}_{\xi \sim \pi} \left[\sum_{t=0}^{H-1} R_{\theta_t}(s_t, a_t)\right]$. We can now reason about whether influencing $\theta_{\not\Delta}$ to become $\theta_\Delta$ is optimal, for various choices of horizon lengths.

Note that when considering $H=1$ (i.e. the smallest possible planning horizon),\footnote{Note that $H=0$ is a degenerate planning horizon, as it would correspond to not seeing any reward signal and simply take actions randomly.} the system cannot influence $\theta$ (as both successor states to the initial state have $\theta=\theta_{\not\Delta}$). As no influence is even possible, we immediately know we are in regime \Circled{1} for this type of influence.

Also note that considering $H=2$, it is now possible to induce $\theta_\Delta$ by deploying $\pi_\Delta$. To determine whether it is optimal with respect to $U_\text{RT}(\xi)$, we can look at the expected value of $\pi_\Delta$ and $\pi_{\not\Delta}$ relative to one another:
$$EU_\text{RT}(\pi_\Delta) = -10 + -10 = -20 \quad \quad \quad EU_\text{RT}(\pi_{\not\Delta}) = 1 + 1 = 2$$
From this we conclude that the influence is currently possible but suboptimal, meaning that at horizon $H=2$, the optimality of this influence incentive is in regime \Circled{2}. 

Similarly to the above, let's consider $H=3$:%
$$EU_\text{RT}(\pi_\Delta) = -10 + -10 + 9 = -11 \quad\quad\quad EU_\text{RT}(\pi_{\not\Delta}) = 1 + 1 + 1 = 3.$$
At $H=4$:
$$EU_\text{RT}(\pi_\Delta) = -10 + -10 + 9 + 8 = -3 \quad\quad\quad EU_\text{RT}(\pi_{\not\Delta}) = 1 + 1 + 1 + 1 = 4.$$

The pattern for $\pi_{\not\Delta}$ is simple: to horizon $H$, $EU_\text{RT}(\pi_{\not\Delta}) = H$. 

For $\pi_\Delta$ we have to do some algebra. For $H > 2$,
\begin{align*}
    EU_\text{RT}(\pi_\Delta) &= -20 + \sum_{t=2}^{H-1} (11 - t) \\
    &= -\frac12 H^2 + \frac{23}{2} H - 41.
\end{align*}
This is a downward-facing parabola. We want to know if it surpasses $EU_\text{RT}(\pi_{\not\Delta}) = H$ and if so, at what $H$ this occurs and at what $H$ it is again overtaken. In other words, we want to know when
\begin{align*}
    -\frac12 H^2 + \frac{23}{2} H - 41 > H,
\end{align*}
or equivalently, when
\begin{align*}
    -\frac12 H^2 + \frac{21}{2} H - 41 > 0.
\end{align*}
Solving this gives us a root between 5 and 6 and another between 15 and 16. As we would expect, we cross into the regime where $\pi_\Delta$ is optimal at $H = 6$,
\begin{align*}
    -\frac12 (5)^2 + \frac{23}{2} (5) - 41 = 4 < 5 \\
    -\frac12 (6)^2 + \frac{23}{2} (6) - 41 = 10 > 6.
\end{align*}
This puts us in regime $\Circled{3}$, in which influence is optimal, until we hit 16:
\begin{align*}
    -\frac12 (15)^2 + \frac{23}{2} (15) - 41 = 19 > 15 \\
    -\frac12 (16)^2 + \frac{23}{2} (16) - 41 = 15 < 16,
\end{align*}

meaning the incentive has switched to regime \Circled{2} again. By looking at the structure of the reward, it's clear that as the horizon increases further, the incentive will remain suboptimal from this horizon onwards.

In conclusion, we get that the horizon boundary points between the different regimes of the optimality progression are $2, 6, 16$.\footnote{The first regime will always start at horizon 1, so we can ignore that from our boundary points.}

\subsection{Changing optimization horizons in the presence of multiple possible kinds of influence}\label{subsubsec:changing-opt-horizon}

When a setting has many possible kinds of influence, reasoning about changing the horizon becomes tricky: not only the optimality progression of each kind influence can be entirely different, but the ``horizon boundary points'' at which each kind of influence transitions between optimality regimes can be different too. As a practical example, consider the recommender system setting described in \Cref{fig:changing-opt-horizon}, in which there are 2 possible kinds of influence: clickbait influence (misleading the user to click on a piece of content), and encouraging addiction.

While the system can discover the clickbait strategy immediately at horizon 1, discovering a strategy which leads to addictive user patterns will require a non-trivial planning horizon. If one is concerned about clickbait, one might attempt to remove it by increasing the optimization horizon: this is because click-bait, while being optimal for immediate engagement, is likely harmful for long-term engagement. This hypothesis was tested successfully by YouTube \cite{chen_reinforcement_2019}. 

However, by increasing the optimization horizon, one might inadvertently make other (undesirable) influence incentives optimal, as shown in \Cref{fig:changing-opt-horizon}. Vice-versa, if one is concerned about an influence incentive that is only present with long-horizons, one might try to remove such incentive by reducing the horizon, potentially only to introduce another incentive, as we explored in \Cref{subsec:influence-and-horizon}. Many empirical questions remain as to whether harmful long-term influence behaviors are discovered by current real-world RL recommenders \cite{carroll_estimating_2022}.

\begin{figure}[h]
  \centering
  \vspace{-0.2em}
  \includegraphics[width=0.45\textwidth]{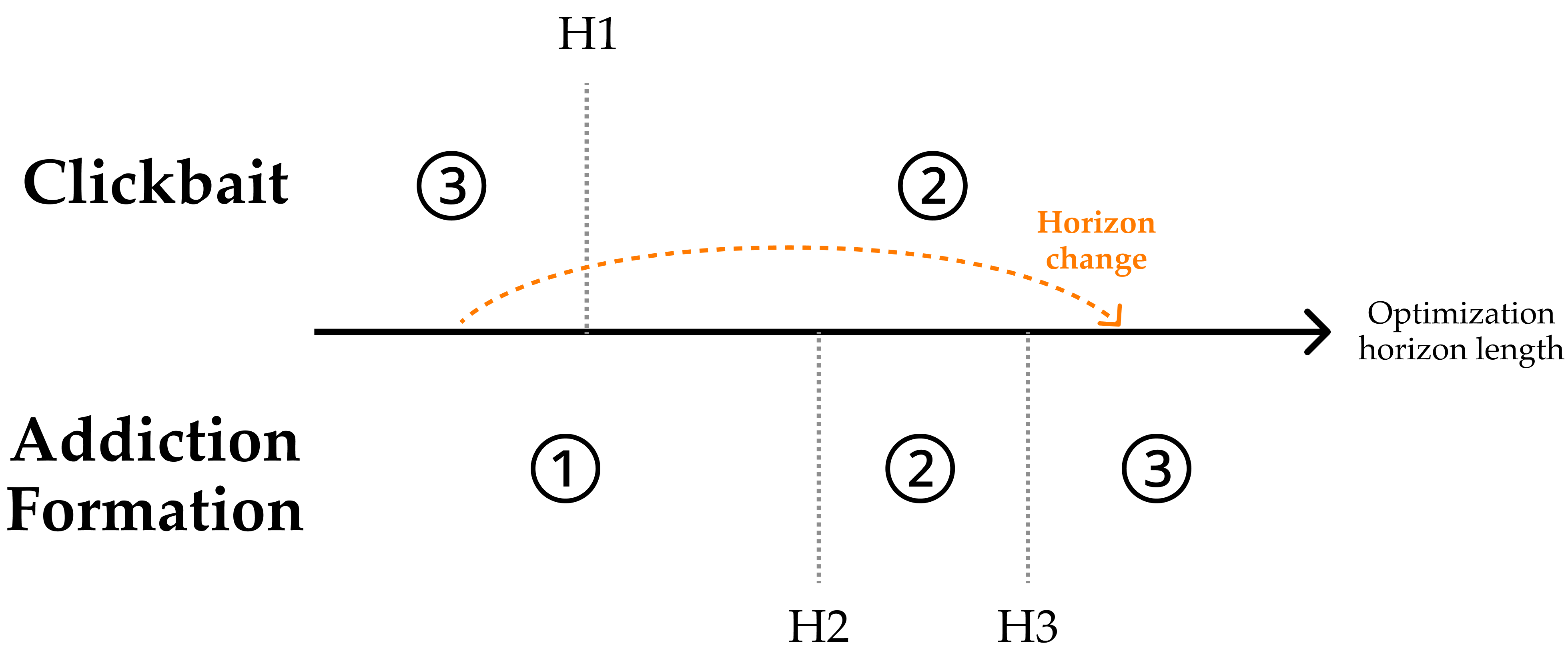} %
  \caption{\textbf{Changing the horizon might make a kind of influence suboptimal but render other kinds of influences optimal.} The figure displays a hypothetical situation in which increasing the horizon in order to reduce clickbait enables the system to discover addiction formation strategies. The circled numbers follow the same scheme as in \Cref{fig:opt-horizon-markov-chain}.}
  \label{fig:changing-opt-horizon}
\end{figure}

\subsection{Infinitely flipping optimality progression}\label{subsubsec:infinite-flipping}
Consider the following example: $S = \{s_0,s_1,s_2,s_3\}$ where $s_0$ is the initial state, $\mathcal{A} = \{a_\text{noop},a_2\}$, $\Theta = \{\theta_0,\theta_\Delta\}$ and $\mathcal{T}$, $R$ are defined as follows:
\begin{align*}
    \mathcal{T}(s_1,\theta_0 | s_0,\theta_0,a_\text{noop}) &= 1\\
    \mathcal{T}(s_2,\theta_\Delta | s_0,\theta_0,a_2) &= 1\\
    \mathcal{T}(s_2,\theta_\Delta | s_2,\theta_\Delta,a) &= 1 \ \forall a\in \mathcal{A}\\
    \mathcal{T}(s_3,\theta_0 | s_1,\theta_0,a) &= 1 \ \forall a\in \mathcal{A}\\
    \mathcal{T}(s_1,\theta_0 | s_3,\theta_0,a) &= 1 \ \forall a\in \mathcal{A}\\
    R_{\theta_0}(s_0,a_\text{noop}) &= \epsilon \\
    R_{\theta_0}(s_0,a_2) &= 1 \\
    R_{\theta_0}(s_1,a) &= 2 \ \forall a \\
    R_{\theta_0}(s_3,a) &= 0 \ \forall a \\
    R_{\theta_0}(s_2,a) &= 1 \ \forall a
\end{align*}
where $\epsilon \in (0,1)$ and undefined values have zero probability or reward. We can note that for odd horizons, taking action $a_2$ from $s_0$ is optimal (so influence is optimal), whereas for even horizons then $a_\text{noop}$ is optimal (so influence is possible but suboptimal). Therefore, the optimality regime permanently alternates. %

\subsection{Infinite-horizon average reward}\label{app:avg-reward}

In \citet{sutton_reinforcement_2018}, the notion of ``average reward'' is considered as a basis for optimality for continuing tasks (tasks without termination or start states). We adapt their definition of average reward:
$$r(\pi) = \lim_{h\rightarrow\infty}\frac{1}{h}\sum_{h=1}^h\mathbb{E}\left[R_t|A_{0:t-1} \sim \pi\right]$$
to the episodic (deterministic) setting (which is what we focus on in \Cref{subsec:influence-and-horizon}):
$$\bar{r}(\pi, s, \theta) = \lim_{h\rightarrow\infty}\frac{1}{h} U_\text{RT}(\xi_{:h} | \pi, s_0=s, \theta_0=\theta).$$
The most significant differences are making the average reward depend on initial conditions $s, \theta$ (as, unlike \citet{sutton_reinforcement_2018}, we don't make an ergodicity assumption). Also note that because 2-reward DR-MDPs are deterministic, we can drop the expectation term.

\subsection{Re-planning: optimization horizon as episode length or planning depth}\label{app:replanning}

There are two distinct interpretations of a finite optimization horizon $H$: as the episode length or the planning depth. For simplicity throughout the paper, we primarily use the episode length interpretation, in which we model a policy optimized out to optimization horizon $H$ as solving an episodic DR-MDP with episodes of maximum length $H$ (even if the task is continuing). Under this interpretation, the optimal policy under $U(\xi)$ may be non-stationary, as the reward-maximizing action from a state may be different depending on the time left in the episode.

In the planning depth interpretation, we allow the DR-MDP to have an episode length longer than $H$ or to be continuing. An optimal horizon-$H$ policy under $U(\xi)$ is a stationary policy $\pi^*_\text{replan}$ which satisfies, for all $s \in \mathcal S$,
\begin{align*}
    \pi^*_\text{replan}(a \mid s, \theta) > 0
    \iff
    a \in \underset{a \in \mathcal A, \pi' \in \Pi}{\arg\max}\ \mathbb E
    \left[
         U(\xi) \mid s_0 = s, a_0 = a, \theta_0 = \theta, \xi_{1:H} \sim \pi'
    \right],
\end{align*}
where $\pi'$ may be a non-stationary policy.

In other words, $\pi^*_\text{replan}$ takes an action from each $(s, \theta)$ that \emph{would be an optimal first action} if the DR-MDP had episode length $H$ and $(s, \theta)$ were the start state and start reward parameter value.

To see where these may differ, consider the following example. The start state is on the left, and each state is annotated with the reward value for entering it.

\begin{center}
\begin{tikzpicture}[node distance=2.5cm, auto]
    \node[circle,draw,minimum size=1cm,inner sep=0pt] (a) {1};
    \node[circle,draw,minimum size=1cm,inner sep=0pt,right of=a] (b) {2};
    \node[circle,draw,minimum size=1cm,inner sep=0pt,right of=b] (c) {0};
    \draw[->] (a) to node {$a_\text{go}$} (b);
    \draw[->] (b) to node {} (c);
    \draw[->] (a) edge [loop above] node {$a_\text{stay}$} (a);
    \draw[->] (c) edge [loop right] node {} (c);
\end{tikzpicture}
\end{center}

The optimal policy under the ``short episodes'' interpretation is to take $a_\text{stay}$ until $t = H-1$, then take $a_\text{go}$ on the final timestep. Under the replanning interpretation, this depends on $H$. If $H = 1$, $\pi^*_\text{replan}$ always takes $a_\text{go}$ from the start state. If $H > 1$, $\pi^*_\text{replan}$ always takes $a_\text{stay}$ from the start state: upon taking this action, returning to the start state, and replanning at horizon $H > 1$, $a_\text{stay}$ remains the optimal action.

Note that the infinite-horizon reward $\bar r$ (\Cref{app:avg-reward}) of any stationary policy which puts probability mass on $a_\text{go}$ is $0$, whereas $\bar r(\pi^*_\text{replan}) = 1$ (when $H > 1$).

This distinction is likely familiar (though in less formal terms) to many reinforcement learning practitioners under the standard RL objective. Another case of particular interest is $U_\text{IR}$, where the planning depth interpretation of optimization horizon gives a distinct optimality criterion from that which we describe in the main text. This optimality criterion is one in which the agent, at time $t$, chooses the action which would maximize $\mathbb E[\sum_{\tau=t}^{t + H - 1} R_{\theta_t}(s_\tau, a_\tau, s_{\tau + 1})]$; this is essentially identical to the TI-unaware current-RF optimization procedure from \citet{everitt_reward_2021}.

We show optimal policies under all objectives with the planning depth interpretation in \Cref{tab:all-optimality-replanning}, mirroring \Cref{tab:all-optimality}. Here, we assume the environments are continuing tasks, and $H$ is the planning depth.

\begin{table}[h]
    \centering
    \caption{\textbf{Representative optimal policies for each of our examples from \Cref{tab:all-examples}, with respect to the each of the objectives in \Cref{tab:maximization_objectives} under the planning depth interpretation of optimization horizon.} 
    In cases in which there is more than one optimal policy, we conservatively display in the table the optimal policy which seems least desirable. 
    For the ``initial reward'' row, we show optimality even for alternate initial states, for the purposes of highlighting the dependency on that (which is also done in \Cref{subsec:initial-reward}).
    Wherever $H$ is not specified, the listed policies are optimal for all $H$. Wherever $s$ or $\theta$ is not specified, the policy takes the same action at all $s$ and/or all $\theta$, respectively.
    The symbol $\mathbf{(\Delta)}$ is used to indicate cases where the optimal action differs from that under ``episode length'' interpretation as depicted in \Cref{tab:all-optimality}.
    The Dehydration environment has been omitted since it does not differ at all from the corresponding column in \Cref{tab:all-optimality}, except that for $H = 1$ all objectives are indifferent between all actions.
    }
    \vspace{0.5em}
    \label{tab:all-optimality-replanning}
\begin{adjustbox}{width=\textwidth,center}
\begin{tabular}{|c|c|c|c|c|}
    \hline
    \textbf{Objective} & \textbf{Conspiracy Influence} & \textbf{Writer’s Curse} & \textbf{Clickbait} & \textbf{AI Personal Trainer} \\ %
    \hline
    Privileged Rew.
    & $\begin{array}{@{}c@{}}
      \pi_{\theta_\text{natural}}^*(s, \theta) = \anoop \\
      \pi_{\theta_\text{influenced}}^*(s, \theta) = a_\text{influence}
    \end{array}$
    & %
    $\begin{array}{@{}c@{}}
        \pi^*_{\theta_\text{ambitious}}(s, \theta) = a_\text{influence} \\
        \pi^*_{\theta_\text{unhappy}}(s, \theta) = a_\text{noop}
    \end{array}$
    & %
    $\begin{array}{cc}
        \begin{array}{@{}c@{}}
            \pi^*_{\theta_\text{normal}}(s, \theta) = a_\text{c.b.} \\
            \pi^*_{\theta_\text{disil.}}(s, \theta) = a_\text{news}
        \end{array}
        & \ \mathbf{(\Delta)}
    \end{array}$
    & %
    $\begin{array}{@{}c@{}}
        \pi_{\theta_\text{tired}}^*(s, \theta) = \anoop \\
        \pi_{\theta_\text{ener.}}^*(s, \theta) = a_\text{nudge}
    \end{array}$ \\
    \hline
    Real-time Rew. 
    & $\begin{array}{@{}c@{}}
        H = 1 \implies \pi^*(s, \theta_\text{natural}) = a_\text{noop}\ \mathbf{(\Delta)}\\
        H > 1 \implies \pi^*(s, \theta) = a_\text{influence}    
    \end{array}$
    & %
    $\begin{array}{@{}c@{}}
        H = 1 \implies \pi^*(s, \theta) = a_\text{influence}\ \mathbf{(\Delta)} \\
        H > 1 \implies \pi^*(s, \theta) = a_\text{noop}
    \end{array}$
    & %
    $\begin{array}{cc}
        \begin{array}{@{}c@{}}
            H = 1 \implies \begin{cases}
                \pi^*(s, \theta_\text{normal}) = a_\text{c.b.} \\
                \pi^*(s, \theta_\text{disil.}) = a_\text{news}
            \end{cases} \\
            H > 1 \implies \pi^*(s, \theta) = a_\text{news}
        \end{array}
        & \mathbf{(\Delta)}
    \end{array}$
    & %
    $\begin{array}{@{}c@{}}
        \pi^*(s, \theta_\text{ener.}) = a_\text{nudge} \\
        H \le 2 \implies \pi^*(s, \theta_\text{tired}) = a_\text{noop}\ \mathbf{(\Delta)}\\
        H > 2 \implies \pi^*(s, \theta_\text{tired}) = a_\text{nudge}
    \end{array}$
    \\
    \hline
    Final Reward 
    & $\begin{array}{@{}c@{}} \pi^*(s, \theta) = a_\text{influence}\end{array}$
    & %
    $\begin{array}{@{}c@{}}
        H = 1 \implies \pi^*(s, \theta) = a_\text{noop} \\
        H > 1 \implies \pi^*(s, \theta) = a_\text{influence}
    \end{array}$
    & %
    $\pi^*(s, \theta) = a_\text{news}$
    & %
    $\pi^*(s, \theta) = a_\text{nudge}$
    \\
    \hline
    Initial Reward 
    & $\begin{array}{@{}c@{}}
        \pi^*(s, \theta_\text{natural}) = a_\text{noop} \\
        \pi^*(s, \theta_\text{influenced}) = a_\text{influence}
    \end{array}$
    & %
    $\begin{array}{@{}c@{}}
        \pi^*(s, \theta_\text{ambitious}) = a_\text{influence} \\
        \pi^*(s, \theta_\text{unhappy}) = a_\text{noop}\ \mathbf{(\Delta)}
    \end{array}$
    & %
    $\begin{array}{@{}c@{}}
        \pi^*(s, \theta_\text{normal}) = a_\text{c.b.} \\
        \pi^*(s, \theta_\text{disil.}) = a_\text{news} \ \mathbf{(\Delta)}
    \end{array}$
    & %
    $\begin{array}{@{}c@{}}
        \pi^*(s, \theta_\text{tired}) = a_\text{noop} \\
        \pi^*(s, \theta_\text{ener.}) = a_\text{nudge}
    \end{array}$
    \\
    \hline
    Natural Reward 
    & $\begin{array}{@{}c@{}}
        \pi^*(s, \theta_\text{natural}) = a_\text{noop} \\
        \pi^*(s, \theta_\text{influenced}) = a_\text{influence}\ \mathbf{(\Delta)}
    \end{array}$
    & %
    $\begin{array}{cc}
        \begin{array}{@{}c@{}}
            H = 1 \implies \begin{cases}
                \pi^*(s, \theta_\text{ambitious}) = a_\text{influence} \\
                \pi^*(s, \theta_\text{unhappy}) = a_\text{noop}\ \mathbf{(\Delta)}
            \end{cases} \\
            H > 1 \implies \pi^*(s, \theta) = a_\text{influence}
        \end{array}
    \end{array}$
    & %
    $\begin{array}{cc}
        \begin{array}{@{}c@{}}
            H = 1 \implies \begin{cases}
                \pi^*(s, \theta_\text{normal}) = a_\text{c.b.} \\
                \pi^*(s, \theta_\text{disil.}) = a_\text{news}
            \end{cases} \\
            H > 1 \implies \pi^*(s, \theta) = a_\text{c.b.}
        \end{array}
        & \ \mathbf{(\Delta)}
    \end{array}$
    & %
    $\begin{array}{@{}c@{}}
        \pi^*(s, \theta_\text{tired}) = a_\text{noop} \\
        \pi^*(s, \theta_\text{ener.}) = a_\text{nudge} \ \mathbf{(\Delta)}
    \end{array}$
    \\
    \hline
    Constr. RT Rew. 
    & $\pi^*(s, \theta) =  \anoop$
    & %
    $\pi^*(s, \theta) =  \anoop$
    & %
    $\pi^*(s, \theta) = a_\text{news}$
    & %
    $\pi^*(s, \theta) =  \anoop$
    \\
    \hline
    Myopic Rew. 
    & 
    $\begin{array}{@{}c@{}}
        \pi^*(s, \theta_\text{natural}) = a_\text{noop} \\
        \pi^*(s, \theta_\text{influenced}) = a_\text{influence}
    \end{array}$
    & %
    $\pi^*(s, \theta) = a_\text{influence}$
    & %
    $\begin{array}{@{}c@{}}
        \pi^*(s, \theta_\text{normal}) = a_\text{c.b.} \\
        \pi^*(s, \theta_\text{disil.}) = a_\text{news}
    \end{array}$
    & %
    $\begin{array}{@{}c@{}}
        \pi^*(s, \theta_\text{tired}) = a_\text{noop} \\
        \pi^*(s, \theta_\text{ener.}) = a_\text{nudge}
    \end{array}$
    \\
    \hline
    ParetoUD
    & $\pi^*(s, \theta) =  \anoop$
    & %
    $\pi^*(s, \theta) =  \anoop$
    & %
    $\pi^*(s, \theta)= a_\text{news}$
    & %
    $\pi^*(s, \theta) =  \anoop$
    \\
    \hline
\end{tabular}   
\end{adjustbox}
\end{table}

\subsection{Proof of \Cref{thm:deterministic-influence-optimal-avg-rew}}\label{app:theory-2-reward}

\textbf{Notation.}

\begin{enumerate}
    \item We'll use $\lim_{h\rightarrow\infty}\frac{1}{h} U_\text{RT}(\xi_{:h} | \pi)$ as a shorthand for $\bar{r}(\pi, s_0, \theta_0) = \lim_{h\rightarrow\infty}\frac{1}{h} U_\text{RT}(\xi_{0:h} | \pi, s_0=s_0, \theta_0=\theta_{\not\Delta})$.
    \item $\xi_{:k} := \{(s_0, \theta_k, a_0, \dots, s_{k}, \theta_{k}, a_{k})\}$
    \item $\Pi_{\not\Delta} \subset \Pi$ be the set of all possible policies which don't take the influence action $a_\Delta$ in state $s_\Delta$, i.e. under which $\theta_\Delta$ is never realized.
    \item $s'_\Delta$ be the successor state to taking the influence action $a_\Delta$ in state $s_\Delta$.
\end{enumerate}

\tworewarddefinition*

\exampletheorem*

\begin{proof}[Proof of \Cref{thm:deterministic-influence-optimal-avg-rew}]
    Over the course of the proof, we will construct a policy $\pi_k$ that we will show to have strictly higher reward than any policy in $\Pi_{\not\Delta}$ for sufficiently large horizons, meaning that the optimal policy must influence.
    
    Let $\pi'$ be a deterministic counterpart to $\pi$ (the policy described in the theorem statement), constructed to be better or equal to it in terms of reward, i.e. 
    \begin{align}\label{eq:2}
    \bar{r}(\pi, s_{\Delta}', \theta_\Delta) \leq\bar{r}(\pi', s_{\Delta}', \theta_\Delta),
    \end{align} by making any of $\pi$'s stochastic actions deterministic, increasing probability only on the higher value actions and breaking ties arbitrarily \cite{sutton_reinforcement_2018}.

    Let $k$ be the smallest natural number for which it's possible to reach $s_\Delta$ in $k - 1$ timesteps.
    Let $\pi_k$ be a policy which, starting from $s_0, \theta_{\not\Delta}$ reaches state $s_\Delta$ at timestep $k-1$, takes action $a_{\Delta}$, and thereafter (i.e. at all $\theta=\theta_\Delta$) behaves identically to $\pi'$.

    Note that $\bar{r}(\pi', s_{\Delta}', \theta_\Delta) = \bar{r}(\pi_k, s'_\Delta, \theta_{\Delta})$, as $\pi_k$ is guaranteed to act like $\pi'$ once $\theta=\theta_\Delta$ by construction (and by the definition of 2-reward DR-MDP, one cannot go back to $\theta_{\not\Delta}$ after reaching $\theta_\Delta$).

    Since the DR-MDP is finite and deterministic, any $(\pi, s, \theta)$ triple will ultimately result in some cycle of length $\le |\mathcal S|\cdot|\Theta|$. Any such cycle has an associated average reward value, so $\bar r$ exists everywhere. Thus by \Cref{lemma:reindexing}, $\bar{r}(\pi_k, s'_\Delta, \theta_{\Delta}) = \bar{r}(\pi_k, s_0, \theta_{\not\Delta})$.
    
    Therefore by \Cref{eq:2} and the above two observations: 
    \begin{align}\label{eq:4}
        \bar{r}(\pi, s_{\Delta}', \theta_\Delta) \leq \bar{r}(\pi_k, s_0, \theta_{\not\Delta})
    \end{align}

    Starting from our assumption, we have the following chain of implications:
    \begin{align}
        \epsilon &< \bar{r}(\pi, s_{\Delta}', \theta_\Delta) - \max_{\pi_{\not\Delta}\in\Pi_{\not\Delta}} \bar{r}(\pi_{\not\Delta}, s_0, \theta_{\not\Delta})\\
        &\leq \bar{r}(\pi_k, s_0, \theta_{\not\Delta}) - \max_{\pi_{\not\Delta}\in\Pi_{\not\Delta}} \bar{r}(\pi_{\not\Delta}, s_0, \theta_{\not\Delta}) \text{ by \Cref{eq:4}} \\
        &\leq \bar{r}(\pi_k, s_0, \theta_{\not\Delta}) - \bar{r}(\pi_{\not\Delta}, s_0, \theta_{\not\Delta}) \text{ for any  $\pi_{\not\Delta} \in \Pi_{\not\Delta}$}\\
        &= \lim_{h\rightarrow\infty}\frac{1}{h} U_\text{RT}(\xi_{:h} | \pi_k) - \lim_{h\rightarrow\infty}\frac{1}{h} U_\text{RT}(\xi_{:h} | \pi_{\not\Delta})\text{ for any  $\pi_{\not\Delta} \in \Pi_{\not\Delta}$}\\
        &= \lim_{h\rightarrow\infty}\frac{1}{h}\big[ U_\text{RT}(\xi_{:h} | \pi_k) - U_\text{RT}(\xi_{:h} | \pi_{\not\Delta}) \big] \text{ for any  $\pi_{\not\Delta} \in \Pi_{\not\Delta}$}
    \end{align}
    
    This means that for each choice of $\pi^i_{\not\Delta}\in\Pi_{\not\Delta}$, there exists a $H_i\in\mathbb{R}$ such that for any $h_i>H_i$: $$\frac{1}{h_i}\big[ U_\text{RT}(\xi_{:h_i} | \pi_k) - U_\text{RT}(\xi_{:h_i} | \pi^i_{\not\Delta}) \big] > \epsilon \implies U_\text{RT}(\xi_{:h_i} | \pi_k) - U_\text{RT}(\xi_{:h_i} | \pi^i_{\not\Delta}) > \epsilon h_i.$$ 

    Let $h=\max_i{h_i}$ (such a maximum must exist since $\Pi_{\not\Delta}$ is a finite set, as the DR-MDP is finite). This implies that for all $\pi_{\not\Delta}\in\Pi_{\not\Delta}$, $U_\text{RT}(\xi_{:h} | \pi_k) - U_\text{RT}(\xi_{:h} | \pi_{\not\Delta}) > \epsilon h > 0$. From this follows for sufficiently large horizons (larger than $h$) no $\pi_{\not\Delta}\in\Pi_{\not\Delta}$ can be optimal, so the optimal policy must take the influence action.
\end{proof}

\begin{lemma} \label{lemma:reindexing} Consider a 2-reward DR-MDP and a deterministic policy $\pi$. Suppose there exist a state $s$, a reward parameterization $\theta$, and a timestep $k \in \mathbb{N}$ such that $P(s_k=s, \theta_k=\theta | \pi) = 1$. If both $\bar{r}(\pi, s_0,\theta_0)$ and $\bar{r}(\pi, s,\theta)$ exist, then \begin{equation*} \bar{r}(\pi, s_0,\theta_0) = \bar{r}(\pi, s,\theta). \end{equation*} \end{lemma}

\begin{proof}
    Let $(r_t)_{t\in\mathbb{Z}^+}$ be the sequence of rewards induced at each timestep by executing $\pi$ in the DR-MDP from the starting state $s_0$ and reward parameterization $\theta_0$.

    Note that $\bar{r}(\pi, s_0,\theta_0) = \lim_{h\to\infty} \frac{r_1 + \dots + r_{h-1}}{h}$. Also, note that $\bar{r}(\pi, s,\theta) = \lim_{h\to\infty} \frac{r_k + \dots + r_{h-1}}{h - k}$ for the same sequence $(r_t)$, as we are guaranteed that $\pi$ reaches state $s$, and reward $\theta$ at timestep $k$.

    Let $G_{h}$ and $G_{h-k}$ respectively be the partial sums of rewards: $G_h = \sum_{t=0}^{h-1} r_t$ and $G_{h-k} = \sum_{t=k}^{h-1} r_t$. Therefore, we can re-write the average reward experssions as $\bar{r}(\pi, s_0,\theta_0) = \lim_{h\to\infty} \frac{G_h}{h}$ and $\bar{r}(\pi, s,\theta) = \lim_{h\to\infty} \frac{G_{h-k}}{h - k}$.

    First, note that:
    $$\bar{r}(\pi, s,\theta) = \left(\lim_{h\to\infty} \frac{G_{h-k}}{h - k}\right) \cdot 1 = \left(\lim_{h\to\infty} \frac{G_{h-k}}{h - k}\right) \cdot \left(\lim_{h\to\infty} \frac{h-k}{h}\right) = \lim_{h\to\infty} \frac{G_{h-k}}{h}$$
    by the algebraic limit theorem for multiplication, as both limits are guaranteed to exist.
    
    We will now show that $\lim_{h\to\infty} \frac{G_{h}}{h} = \lim_{h\to\infty}\frac{G_{h-k}}{h}$, showing that the limiting average reward is not affected by the first $k$ terms. Note that:
    $$\lim_{h\to\infty} \frac{G_{h-k}}{h} = 0 + \lim_{h\to\infty} \frac{G_{h-k}}{h} = \lim_{h\to\infty} \frac{G_k}{h} + \lim_{h\to\infty} \frac{G_{h-k}}{h} = \lim_{h\to\infty} \frac{G_k + G_{h-k}}{h} = \lim_{h\to\infty} \frac{G_h}{h} = \bar{r}(\pi, s_0,\theta_0)$$
    by the algebraic limit theorem for addition, as both limits are guaranteed to exist.

    Putting everything together, we get $\bar{r}(\pi, s_0,\theta_0) = \bar{r}(\pi, s,\theta)$, proving the statement.
\end{proof}

\section{Possible DR-MDP Objectives}\label{app:solutions}

\subsection{$U_\text{RT}(\xi)$-optimal policies can disagree with normatively unambiguous optimal policies}\label{app:CIR-disagrees}

Consider the DR-MDP from \Cref{fig:disagrees}: there are two reward functions, $\Theta=\{\theta_0, \theta_\Delta\}$, two actions $\mathcal A = \{a_{noop}, a_\Delta\}$, and a single state $s_0$. Assume the human transitions deterministically to $\theta_\Delta$ every time the AI system takes the $a_\Delta$ action. Instead, taking the $\anoop$ action transitions the human to have $\theta_0$.

\begin{figure}[h]
    \centering
    \begin{tikzpicture}
      \tikzset{state/.style={circle, draw, minimum size=1cm, align=center}}
    
      \node[state] (a) at (0,0) {\( s_0, \theta_0 \)};
      \node[state] (b) at (4,0) {\( s_0, \theta_\Delta \)};
    
      \draw[->] (a) to[bend left] node[above] {\( a_\Delta \)} (b);
      \draw[->] (b) to[bend left] node[below] {\( a_\text{noop} \)} (a);
    
      \draw[->] (a) edge[loop left] node[left] {\( a_\text{noop} \)} (a);
      \draw[->] (b) edge[loop right] node[right] {\( a_\Delta \)} (b);
    \end{tikzpicture}    
    \caption{A simple DR-MDP which can lead $U_\text{RT}$ to lead certain actions to be optimal that all reward functions disagree with.}
    \label{fig:disagrees}
\end{figure}
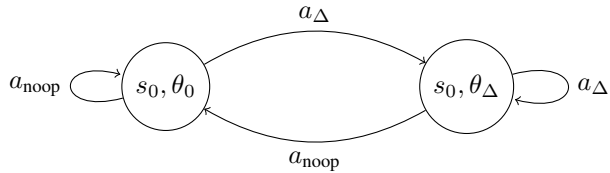

Consider the following reward functions for this DR-MDP:
\begin{align}
    R_{\theta_0}(s, a, s') = \left\{
    \begin{array}{cl}
        5 & \text{ if } a=\anoop\\
        0 & \text{ if } a=a_{\Delta}
    \end{array}
    \right.
    \ \ \ \ \ \ \ \ \ \ \ \ \ \ \ \
    R_{\theta_\Delta}(s, a, s') = \left\{
    \begin{array}{cl}
        25 & \text{ if } a=\anoop\\
        20 & \text{ if } a=a_{\Delta}
    \end{array}
    \right.
\end{align}

Note that optimal policies with respect to the two $\theta$s (as defined in \Cref{def:opt-wrt-theta}) are respectively:
\begin{align}
    \pi^*_{\theta_0}(s, \theta, t) = \anoop \ \ \forall s, \theta, t
    \ \ \ \ \ \ \ \ \ \
    \pi^*_{\theta_\Delta}(s, \theta, t) = \anoop \ \ \forall s, \theta, t
\end{align}

\prg{Influence will be optimal despite all reward functions disprefering it.} For both reward functions, it is the case that $\anoop$ actions have higher value than influence actions $a_\Delta$. Therefore, under both reward functions it is optimal for the AI agent to \textit{never} perform the influence action. However, for any non-terminal timestep, it will \textit{always} be optimal with respect to $U_\text{RT}(\xi)$ (as defined in \Cref{def:opt-wrt-U}) to take the influence action:
\begin{align}
    \pi^*_\text{RT}(s, \theta, t) = a_\Delta \ \ \forall s, \theta, t < T-1
\end{align}
This is because maximizing real-time reward $U_\text{RT}(\xi)$ will entail remaining with reward $\theta_\Delta$ as long as possible (as rewards values are larger under this reward function), despite the person always preferring AI inaction. 

\prg{$U_\text{RT}(\xi)$ assumes inter-temporal comparisons of utility are meaningful.} Ultimately, $U_\text{RT}(\xi)$ is baking in an assumption that it's meaningful and worthwhile to make ``inter-temporal'' comparisons of utility between the different selves (and their respective reward functions), even against the wishes of each individual reward function. We discuss this fact further in \Cref{app:criticisms}.

\prg{Additional considerations.} This is significant because it means that in some sense $U_\text{RT}(\xi)$ is ``disagreeing'' with a solution which is ``unanimous'' among the individual points of view which we consider. One could see this example as a reason to doubt as to whether normative unambiguity (\Cref{def:normative-ambiguity}) is sufficient to know how we should act in a certain setting---should the AI system should shift the person to experience higher reward? However, as the optimal behavior under $U_\text{RT}(\xi)$ must act contrary to each reward function's wishes, to us it seems like one should respect the autonomy of the person (whose different rewards are in agreement) in performing the final judgement about the relevant interpersonal comparisons of utility (which should be reflected by the reward function(s) in the first place). Ultimately, to us this example provides futher reason to doubt that using $U_\text{RT}(\xi)$ will lead to the types of AI system behaviors that we would desire and would find acceptable. For further considerations about $U_\text{RT}(\xi)$, see also \Cref{app:criticisms}

\subsection{Myopic Reward: it's not always obvious if a system is truly myopic}\label{app:myopic-or-not}

\citet{krueger_hidden_2020} argues that while myopia may hide influence incentives from an AI agent, the value of influence might be accidentally ``revealed'' to the agent depending on the training setup despite the myopia. This points to the fact that whether a system is myopic is not always obvious, as we'll show in the case of recommender systems which optimize long-term metrics myopically (which is a common setup in practice). 

Say one is myopically optimizing a user's \textit{session}-watchtime, as was being done by YouTube in 2016, as discussed in \citet{covington_deep_2016}. Even though the system is myopic, it will try to implicitly learn which kinds of sequences of videos maximize session watchtime, which in turn depends on both the user's and the AI's behavior after the current recommendation. Anecdotally, we found that some recommender systems practitioners are aware that with under a simple setup of iterated deployment and retraining, training myopically with long-term metrics should correspond to a policy improvement iterator, meaning that it will eventually converge to the RL optimum (which is absolutely not myopic). To the best of our knowledge however, this argument has not yet been published explicitly, so we provide a proof below. This goes to show that establishing whether a system is truly myopic can often be challenging to interpret. 

Additionally, as we show in \Cref{subsec:influence-and-horizon}, a system being (truly) myopic does not mean it is incapable of influence, which may even be elaborate or seemingly involve complex reasoning steps. As an additional example to that of clickbait from \Cref{fig:clickbait}, consider the case of sycophancy in LLMs \cite{sharma_towards_2023}: RLHF for LLMs can also be viewed as a form of myopic objective (as discussed in \Cref{app:objective-reductions}). From this perspective, one can think of the LLM as implicitly inferring some aspects of the user's cognitive state, and subtly tailoring its responses in order to maximize the expected user approval, even though ``it's only reasoning over a single step''. %

\subsubsection{Optimizing long-term metrics myopically is equivalent to RL under mild assumptions}\label{subsec:myopic}

A recommender system which selects content myopically according to predicted long-term metrics (e.g., \citet{covington_deep_2016}) can be modeled as having two main components:
\begin{enumerate}
    \item A model used to predict the long-term metrics (e.g. user retention, or session watchtime) based on the user context $s$ and a candidate recommendation $a$, which we denote as $\hat{Q}(s, a)$.
    \item A policy which greedily selects the content which maximizes the predicted long-term metric: $\pi(s) = \max_a \hat{Q}(s, a)$.
\end{enumerate}

Note that the iterative retraining procedure in \Cref{alg:retraining} is equivalent to a policy improvement procedure: updating the policy based on the latest Q-value estimate is a policy improvement step, and updating the Q-value estimates based on the long-term metrics obtained by the latest policy is equivalent to a policy evaluation step. 
If the updated Q-value estimates have sufficiently low error, one would have a guarantee of convergence to the optimal RL policy $\pi^*_\text{RL}$, which maximizes the long-term metrics \cite{sutton_reinforcement_2018}. 

Therefore, the effective optimization horizon of myopic recommenders which perform iterative retraining (as done by most non-RL recommender systems) may be best thought of as the longest horizon present in the metrics they optimize. 

That being said, the above analysis comes with caveats. Obtaining good estimates of Q-values is likely challenging in practice, and real-world recommender system environments are likely non-stationary. Moreover, \Cref{alg:retraining} is certainly an oversimplification of how real world recommender systems are trained and deployed. For example, selection of content often also involves filtering and re-ranking to add diversity to user slates \cite{thorburn_how_2022}. The degree to which all these factors blunt RL and myopic systems' ability to influence is yet to be studied in practice.

\begin{algorithm}
\caption{Iterative Retraining of Long-Term Metric Myopic Recommender System}
\label{alg:retraining}
\begin{algorithmic}[1]
\REQUIRE Initial long-term metric predictor $\hat{Q}_0(s, a)$
\WHILE{True}
    \STATE Deploy policy $\pi_{i-1}(s) = \max_a \hat{Q}_{i-1}(s, a)$
    \STATE Store data collected from $\pi_{i-1}(s)$ in $\mathcal{D}_i$
    \STATE Continue training the long-term metric predictor $\hat{Q}_i$ using $\mathcal{D}_i$
\ENDWHILE
\end{algorithmic}
\end{algorithm}

\subsection{More context and motivation for the ParetoUD objective}\label{app:paretoUD}

\begin{figure}[h]
  \centering
  \vspace{-0.2em}
  \includegraphics[width=0.60\textwidth]{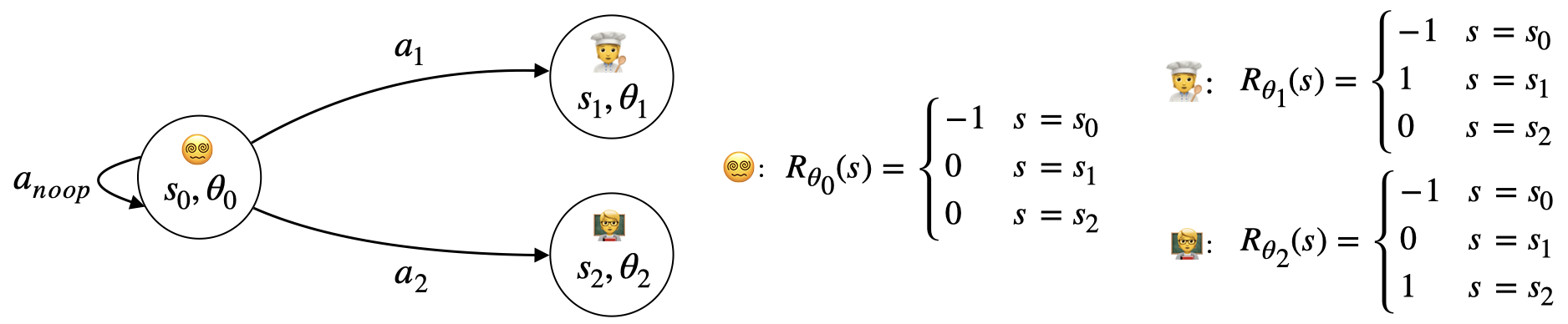}
  \vspace{-1em}
  \caption{\textbf{Career choice paralysis.} Taylor is stuck deciding between two possible career choices: becoming a cook or a teacher. The AI system can convince them to pursue either path, or leave them stuck. While the setting is normatively ambiguous, all selves agree that remaining stuck is the worst possible outcome. The ParetoUD policies are ones that encourage the person to become a cook or a teacher.}
  \label{fig:paretoUD}
\end{figure}

\prg{An example of ParetoUD in action.} Consider the example from \Cref{fig:paretoUD}, which has horizon $H=1$. Let $\pi_1$ and $\pi_2$ are respectively the policies that take action $a_1$ and $a_2$. Note that the setting is normatively ambiguous: the optimal policies according to the stuck self are $\Pi_{\theta_0} = \{ \pi_1, \pi_2 \}$, the optimal policies according to the cook self are $\Pi_{\theta_1} = \{ \pi_1 \}$, and the teacher self are $\Pi_{\theta_2} = \{ \pi_2 \}$. Therefore, there is no one policy which is optimal according to all reward parameterizations. However comparing the expected utility of $\pi_\text{noop}, \pi_1, \pi_2$, both $\pi_1$ and $\pi_2$ are better than $\pi_\text{noop}$ according to any of the $\theta$s. This means that both $\pi_1$ and $\pi_2$ are unambiguously desirable improvements to the status quo of the system not existing (and Taylor remaining stuck. $\pi_\text{noop}$ is always in the UD set of policies by construction. However, when considering the pareto UD policies, we only have $\pi_1$ and $\pi_2$, as they dominate $\pi_\text{noop}$. While for this example the optimal policies according to ParetoUD are the same as the optimal policies under e.g. Initial Reward or Real-time Reward, ParetoUD is much more conservative than either of these objectives for most settings (as can be seen by \Cref{tab:all-optimality}).

\textbf{More context on the ParetoUD objective in \Cref{tab:maximization_objectives}.} In \Cref{tab:maximization_objectives}, we denote PE and UD as indicators for the respective properties of Pareto Efficiency (\Cref{def:PE}) and Unambiguous Desirability \Cref{def:UD} being satisfied. In the case of a discrete $\Theta$ space, we can expand the expression out further and turn it into a maximization problem as: $$\max_{\pi} PE(\pi) + \sum_\theta \mathbb{I}(EU_{\theta}(\pi) \geq EU_{\theta}(\pinoop))$$
It may not be immediately clear why (especially in the objective above), one doesn't have to restrict the Pareto Efficiency indicator to the subset of policies $\Pi_{UD} \subset \Pi$ (as discussed in \Cref{sec:desirable-influence-individual-rationality}). To see why, note that the summation expresses the UD condition, and we know that there will always be at least one policy which satisfies it ($\pinoop$)---so we can always obtain an objective value of $|\Theta|$. Moreover, we know that there always must be a Pareto Efficient policy within $\Pi_{UD}$, meaning that all indicators (including the $PE$ function) can be equal to 1 at once, meaning that the objective can take on value $|\Theta|+1$, ensuring that the solution will both be Pareto Efficient and Unambiguously Desirable.

\textbf{Selecting among the Pareto Efficient policies.} An interesting question for further work would be to study whether there are better ways to select from Pareto Efficient policies, rather than just tie-breaking arbitrarily within the subset of policies $\Pi_\text{UD}$ which are Pareto Efficient. For example, one could use social choice functions, e.g. with the goal of fairly allocating the gains relative to inaction to the various selves $\theta$. The main advantage of simply requiring Pareto Efficiency is that it doesn't assume that it is meaningful to perform ``interpersonal'' comparisons of utility, i.e. the values assigned by different reward parameterizations do not have to be on the same scale.

\prg{ParetoUD acts on aspirations which are consistent across $\theta$s.} Ultimately, the motivation of ParetoUD comes from the fact that we might want AI systems to help us change in ways that are different in character (or speed) relative to the natural reward evolution (\Cref{def:natural-reward-evolution}) we would have without the system.%

\subsection{Efficient algorithms and tractability}

Our main focus in this work is to provide a clear formalism for grounding discussions about dynamic-reward problems, rather than developing efficient solutions. Therefore, we have mostly ignored tractability issues of the objectives we propose. That being said, most of our objectives can be easily optimized using standard RL techniques, or techniques developed in previous work \cite{everitt_reward_2021,carroll_estimating_2022,achiam_constrained_2017}. The two objectives from \Cref{tab:maximization_objectives} which may be most challenging to optimize are Final Reward (for which one can likely develop appropriate Bellman Updates), and ParetoUD. However, the former objective has the most susceptibility to influence incentives out of all objectives, and the latter is overly conservative, potentially making them unlikely objectives to want use in practice.

\section{The DR-MDP Objectives Most Similar to Existing Alignment Practices}\label{app:objective-reductions}

In \Cref{subsec:current-AI}, we claim that the training setups for recommender systems and for reward modeling (under one interpretation) are implicitly optimizing objectives which are similar to---respectively---real-time reward $U_\text{RT}(\xi)$ and initial reward $U_\text{IR}(\xi)$. Additionally, in \Cref{tab:intuition_strengths_weaknesses} we categorize other prior works which implicitly optimize objectives similar to the DR-MDP objectives we consider. 

\prg{Challenges in unambiguously casting prior approaches as DR-MDP objectives.} Because most of the prior works we reference don't explicitly account for the possibility of changing preferences, analyzing how they would handle changing preferences and their similarity to the DR-MDP objectives is often confusing: it depends highly on the assumptions one makes about what constitutes a timestep, which preference dynamics (if any) are at play during reward learning, if the person is providing feedback about their own interaction with the system (or as a third-party observer), etc.
Here we attempt to informally motivate the correspondences we specified in \Cref{tab:intuition_strengths_weaknesses}, and sketch the assumptions they most rely on. We also discuss which DR-MDP objectives are most similar to additional prior work which could not fit in \Cref{tab:intuition_strengths_weaknesses} due to space constraints.
Given the many simplifications required for some of the comparisons (and their tangential importance towards the goal of our paper), technical precision will not be the main aim of this section.
We mostly hope it can provide a helpful starting point for others that may be interested in interpreting their training setup in terms of DR-MDP objectives. 

\prg{Section outline.} In \Cref{app:reward-modeling}, we first discuss simplifying assumptions which help with later comparisons between existing methods and DR-MDPs; we then argue that even if the idealized assumptions were unmet in practice, the methods in consideration would almost certainly still lead to influence incentives (though they would be much harder to analyze than those that one would expect from simple DR-MDP objectives).  
In the subsections that follow, we proceed to categorize AI alignment techniques based on the DR-MDP objectives they seem most similar to.%

\subsection{Idealized assumptions}\label{app:reward-modeling}

While we already discussed in \Cref{subsec:initial-reward} how (under a specific interpretation) reward modeling is equivalent to $U_\text{IR}$, the language of reward modeling is very flexible and can be used to describe almost any of the DR-MDP objectives we consider.
Because of the intuitiveness of the reward modeling framework \cite{leike_scalable_2018}, in the subsections that follow we will map alignment approaches to DR-MDP objectives by first casting them (sometimes trivially) in terms of reward modeling, and from there to DR-MDP objectives.
This allows us to first consider how reward modeling can be interpreted in the lens of DR-MDPs (more broadly than in \Cref{subsec:initial-reward}), and then apply this framework to each individual alignment technique. 

\citeauthor{leike_scalable_2018} describe reward modeling as a two-phase approach which entails:
\begin{quote}
    \textit{(1) learning a reward function from the feedback of the user and \\ 
    (2) training a policy with reinforcement learning to optimize the learned reward function}.
\end{quote}

The reward function learned from feedback of users---as conceived of in \citeauthor{leike_scalable_2018}---is a single, static, reward function. Therefore, insofar as the feedback from users for phase (1) was being provided at different times, from the point of view of different cognitive states $\theta$, and from multiple people, such reward function would be a mixture across different people, and across their different cognitive states which they had during reward learning time, which would be very hard to analyze in terms of DR-MDPs.

\prg{Assumption 1: each state $s$ is sufficiently expressive as to uniquely determine the cognitive state $\theta$ of the person providing reward feedback.} This ensures that even if a system is not explicitly modeling $\theta$, it can distinguish between the different values of $\theta$ insofar if it's useful to do so (both at reward learning time, and at RL training or deployment time). This may seem relatively reasonable if the person is providing ``first-person feedback'' about their own interaction with the AI system: as an example, in the setting of recommender systems, the person provides feedback (e.g., likes, comments, etc.) based on a rich context of their historical interactions; these form a state representation $s$ which is likely sufficient to recover many aspects of their preferences, and $\theta$ more broadly.\footnote{In practice, even with a partially observable changing $\theta$, there can still be influence incentives. See \Cref{app:learning-reward}, and in particular the reference to \cite{carroll_estimating_2022}.}
While all the examples in our paper are of the form above (in which feedback about the reward comes from the person interacting with the AI assistant), in many standard reward learning setups the person providing feedback to the AI system is instead doing so from a ``third-person perspective'' (the trajectories being evaluated do not involve directly interacting with a human, or the human involved is not the annotator).\footnote{Note that the reward learning step is what determines the reward model, which in turn determines the resulting influence incentives. When the reward learning feedback comes from a ``third-person perspective'', the AI system will have direct incentives to influence the annotator, rather than any person (which may or may not be present) in the evaluated trajectories. That being said, such incentives to influence the annotator may correspond to specific incentives to influence the person in the trajectories. 
We hope to explore these nuances in future work.}
In this case, the assumption of known $\theta$ seems more unrealistic, but as we'll see in, e.g., \Cref{app:final-reward}, a less imposing assumption may suffice to have similar effects (but which the DR-MDP framework often does not lend itself to analyze out-of-the-box).

\prg{Assumption 2: shared dynamics across multiple humans.} While many of the alignment techniques we consider were initially designed as single-human alignment techniques \cite{critch_ai_2020}, in practice they are generally used to ``learn a reward model'' using the feedback obtained from many different humans \cite{ouyang_training_2022}. 
As long as the dynamics of states and cognitive states $\mathcal{T}$ are the same across different people (which we refer to as the ``shared dynamics assumption''), the $\theta$-observability assumption is sufficient to guarantee that the system always has all the relevant information: if different people have different cognitive states, the system would observe that and be able to respond accordingly (personalizing actions and influence). If instead two people had the same cognitive state $\theta$, by the shared-dynamics assumption, distinguishing between them is unnecessary in terms of solving the DR-MDP (whatever the objective), as their reward functions and states would transition in exactly the same way.

\textbf{Assumption 3: coverage assumption.} We further assume that the RL training sufficiently covers the space of initial states and reward functions $\Theta \times \mathcal{S}$ (and that the learned reward function $\hat{R}(s, a, s')$ is a sufficiently good approximation over such space to enable such training), as to ensure generalization to test-time humans which may have different starting states and reward functions.
While this may seem like an unrealistic assumption, it becomes more plausible seen in the context of learning from/with many different humans, under the shared dynamics assumption from above. Indeed, assuming that one will be performing reward learning (and learning reward dynamics) from multiple humans can help with having this assumption satisfied, as discussed in \Cref{app:learning-reward}.

\begin{claim}[\textbf{Correspondence of Reward Modeling to DR-MDPs}]
    \label{claim:correspondence-rm-dr-mdp}
    If a reward modeling approach satisfies the 3 assumptions above, the learned reward function $\hat{R}(s, a, s')$ will be equal to $\hat{R}_\theta(s, a, s')$ for some $\theta$ (which depends on the reward learning setup); as a consequence of this, when one uses RL to optimize $\sum_t^{H-1} \hat{R}(s_t, a_t, s_{t+1})$, one is implicitly optimizing $\sum_t^{H-1} \hat{R}_\theta(s_t, a_t, s_{t+1})$ where the exact value of $\theta$ depends on the training setup. Moreover, one can expect standard reward modeling done under such assumptions to generalize similarly at deployment time to how it would generalize if it had been trained modeling $\theta$ explicitly (under the implicit DR-MDP objective).
\end{claim}

\emph{Informal proof of \Cref{claim:correspondence-rm-dr-mdp}:} 
Consider someone giving reward feedback about a transition $(s_t, a_t, s_{t+1})$ at reward learning time. For simplicity, assume they evaluate the transition providing as feedback a reward value $r_t$ directly.
Let's consider three cases (which are not comprehensive, but cover some of the setups we are most interested in):

\hypertarget{case1}{\textbf{Case 1: feedback comes from current $\theta$.}} The reward feedback for the transition at time $t$ of each trajectory $\xi$ is collected from the reward function $R_{\theta_t}$, i.e. the reward function corresponding to the current state $s_t$. 
When a reward modeling technique claims to have learned $\hat{R}(s_t, a_t, s_{t+1})$, it has in fact learned $\hat{R}_{\theta_t}(s_t, a_t, s_{t+1})$. Note that this means that evaluations of transitions under different cognitive states than the current one are never learned, i.e., any $\hat{R}_{\theta}(s_t, a_t, s_{t+1})$ where $\theta\neq \theta_t$. At RL time, there may be a different distribution of starting states. Note that the system (despite not explicitly representing $\theta$) will be able to distinguish between $\theta$s because of the state expressivity assumption (assumption 1). When optimizing $\sum_t^{H-1} \hat{R}(s_t, a_t, s_{t+1})$ during RL, the system will be implicitly optimizing $\sum_t^{H-1} \hat{R}_{\theta_t}(s_t, a_t, s_{t+1})$. By the coverage assumption, shared dynamics assumption, and the state expressivity assumption, at test time the policy should be able to infer the person's cognitive state (insofar as it's necessary to optimize the objective optimally), and thus be able to generalize to any initial state and cognitive state.

\hypertarget{case2}{\textbf{Case 2: feedback comes from a fixed, trajectory-dependent $\theta$.}} The reward feedback for the transition at time $t$ of each trajectory $\xi$ is collected from a fixed reward function $R_{f(\xi)}$, for some trajectory dependent cognitive state given by a function $f: \Xi \rightarrow \Theta$. Let's consider $f$ which selects the reward function corresponding to the \textit{final} state $s_H$, i.e., a choice of $f$ such that $f(\xi) = \theta_H$.
When a reward modeling technique claims to have learned $\hat{R}(s_t, a_t, s_{t+1})$, it has in fact learned $\hat{R}_{\theta_H}(s_t, a_t, s_{t+1})$ (where the reward will have different values depending on the trajectory considered, as $\theta_H$ is trajectory dependent). Note that similarly to Case 1, at RL time, the system will be able to distinguish between the current $\theta$s, and plan what is the most advantageous $\theta_H$ to try to induce (in terms of expected reward). When optimizing $\sum_t^{H-1} \hat{R}(s_t, a_t, s_{t+1})$ during RL, the system will be implicitly optimizing $\sum_t^{H-1} \hat{R}_{\theta_H}(s_t, a_t, s_{t+1})$. For the same reasons as in Case 1, we can expect the policy to generalize to any initial state distribution. Also, note that this same argument can be applied for other choices of $f$, such as $f(\xi)=\theta_0$. In this case, the reward feedback for the transition at time $t$ of each trajectory $\xi$ is collected from the initial reward function $R_{\theta_0}$.

\hypertarget{case3}{\textbf{Case 3: feedback comes from a fixed, trajectory-independent $\theta$.}} The reward feedback for the transition at time $t$ of each trajectory $\xi$ is collected from a single, fixed reward function $R_\theta$ (where $\theta$ may not even appear in the trajectory).
When a reward modeling technique claims to have learned $\hat{R}(s_t, a_t, s_{t+1})$, it has in fact learned $\hat{R}_\theta(s_t, a_t, s_{t+1})$. 
At RL time, the system will be able to distinguish between $\theta$s as in the previous cases (by assumption 1). When optimizing $\sum_t^{H-1} \hat{R}(s_t, a_t, s_{t+1})$ during RL, one is implicitly optimizing $\sum_t^{H-1} \hat{R}_\theta(s_t, a_t, s_{t+1})$ for the fixed value of $\theta$ which one was eliciting feedback from at reward learning time. By the coverage assumption, shared dynamics assumption, and the state expressivity assumption (assumptions 1-3), at test time the policy should always be able to identify the person's cognitive state (insofar as it's necessary to optimize the objective optimally), leading to successful generalization. %

\prg{What if assumptions 1-3 don't hold in practice?} One might wonder what DR-MDP objectives current techniques would correspond to in practice without the assumptions above (and more case-specific ones we make later). First and foremost, without these assumptions, the reward function obtained by the reward learning step would almost certainly come from a mixture of cognitive states (and potentially of different individuals), whose evaluations are aggregated in potentially unstructured and conflicting ways (as the system isn't able to fully disambiguate between cognitive states due to the state representation not being expressive enough). Any fixed mixture of rewards can generically be thought of as corresponding to a ``privileged reward'' objective (whose corresponding reward parameterization $\theta$ may be unreachable, at it is based on an arbitrary amalgamation of different cognitive states \Cref{app:reachable}). However, as discussed in \Cref{sec:solutions} and \Cref{app:privileged-reward}, any privileged reward DR-MDP objective will still lead to potentially undesirable influence incentives (similarly to the initial reward objective), unless the reward function is somehow encoding the ``correct'' trade-off between preferences. Because the trade-offs between current selves encoded by current alignment techniques are quite unstructured and arbitrary in the absence of our simplifying assumptions, it seems unlikely that they will be encoding the ``correct'' trade-off without a careful accounting for it. This leads us to believe that the DR-MDP objective correspondences we provide under our assumptions are likely favorable interpretations. As a parallel, the implicit aggregation of preferences across different users which is performed by RLHF has recently been shown to be equivalent---under certain weaker assumptions---to the Borda count social choice rule \cite{siththaranjan_distributional_2023}. While this is a surprisingly structured ``mixture'', it also has various undesirable properties---which is what one might have expected. %

\subsection{Real-time Reward}\label{app:real-time-rew-correspondence}

All the methods are similar to real-time reward can be thought of as approximately falling under \hyperlink{case1}{Case 1} from \Cref{claim:correspondence-rm-dr-mdp}.

\prg{RL Recommender Systems.} Most approaches for RL in recommender systems are based on doing offline RL, or doing on-policy training in simulation using learned human models that are trained to emulate historical engagement data \cite{afsar_reinforcement_2021}. In both cases, the reward signal (either in the static dataset used for offline RL, or for training the human simulator) comes directly from people's interactions with the system: the reward is generally modeled to be essentially equivalent to user engagement \cite{thorburn_how_2022,afsar_reinforcement_2021}. %
This makes the reward learning step (from \Cref{app:reward-modeling}) trivial, as the correspondence between behavior and reward is hardcoded. The resulting ``reward model'' is either the engagement labels themselves (i.e. in the case of offline RL), or a human engagement model trained on past engagement data. 
As discussed in \Cref{subsec:current-AI} and in \hyperlink{case1}{Case 1}, optimizing $\sum_{t=0}^{H-1} R(s_t, a_t,s_{t+1})$ at deployment time implicitly corresponds to optimizing $\sum_{t=0}^{H-1} R_{\theta_t}(s_t, a_t,s_{t+1})$.

\prg{TAMER.} Similarly to recommender systems, approaches such as TAMER \cite{hutchison_training_2013}, Deep TAMER \cite{warnell_deep_2018}, or the EMPATHIC framework \cite{cui_empathic_2020}, have reward learning step in which the human provides feedback $r_t$ in real-time in the deployment environment according to their current cognitive state $\theta_t$. Because of this, similarly to the recommender system setting, their optimization objective corresponds to the real-time reward objective. %
However, the feedback about AI behavior generally comes from a third-person perspective, complicating the correspondence. See the section of ``the original RLHF method'' from \Cref{app:final-reward} for more details.

\prg{Multi-turn RL for LLMs with real-time feedback (with $t=$ conversation turn).} Although this is not currently known to be common practice, one could imagine a variant of the standard RLHF setup for LLMs (e.g. that of \citet{ouyang_training_2022}) in which a single user, over the course of normal usage of a language model is providing feedback for each model output: this could be via thumbs up/down (as is currently present in the ChatGPT interface), or if at every timestep the user is presented with two output options which they need to select between in order to continue the conversation (as in some early versions of Claude). 
By having the user provide feedback about the latest AI output at every timestep of the conversation, this would be equivalent to training a reward model with rewards always conditional on the current cognitive state of the person.
If one were to then to optimize cumulative reward under such reward model, this would be similar to optimizing the real-time reward objective.

\subsection{Final Reward}\label{app:final-reward}

\prg{Multi-turn RL for LLMs with final feedback (with $t=$ conversation turn).}
Similarly to the RLHF variant from \Cref{app:real-time-rew-correspondence}, one could imagine a variant of reinforcement learning from human feedback in which the user provides approval labels (e.g., thumbs-up or thumbs-down) which evaluate the entire conversation had with the LLM (in this setup, the feedback would only be given at the end of the conversation). If  for simplicity one assumes that all conversations are of the same length (say, $H$), this would lead to learning a reward model $\hat{R}_{\theta_H}$. 
Optimizing for cumulative reward with such a reward model 
would be similar to using the final reward DR-MDP objective.

\prg{The original RLHF method.} Consider the RLHF method from \citet{christiano_deep_2017}, in which the user provides preference feedback after viewing snippets of AI behaviors. If the length of the snippets is equivalent to the horizon length (i.e., the snippets are full trajectories), then this setup is most similar to the final reward objective: insofar as the preferences of the person giving feedback are changing while viewing the trajectory, it seems plausible that they would retroactively evaluate the trajectory based on their final point of view at time $t=H$, i.e., using $\theta_H$. Optimizing a reward model learned this way seems most similar to optimizing the final-reward objective $\sum_t^{H-1} R_{\theta_H}(s_t, a_t, s_{t+1})$, where $\theta_H$ is the cognitive state realized at the end of the trajectory. Note that while this is somewhat similar to \hyperlink{case2}{Case 2} from \Cref{app:reward-modeling}, the trajectories being evaluated likely do not include information about the person giving feedback, invalidating assumption 1. A different (and slightly more realistic) assumption may be that each person providing feedback initially approaches each annotation with the same initial cognitive state $\theta_0$ (which is trajectory independent, making this setting somewhat similar to \hyperlink{case3}{Case 3}), and progressively forms their opinion about the trajectory as they view each individual AI action, ultimately evaluating the trajectory from the point of view of their final option (captured by their cognitive state $\theta_H$). 
Even if the annotator's cognitive states were unobserved at reward learning time, by performing RL with the learned reward model $\hat{R}_{\theta_H}(s_t, a_t)$, the resulting system should still be able to find the policies which maximize cumulative reward under the annotators' final-reward over the partial observability (because despite not observing $\theta$, the system gets to see the reward signal, and can find the policies that maximize it \textit{in expectation}). %

\prg{Standard RLHF for LLMs (with $t=$ token).} If one considers each ``timestep'' as generating an individual token, the current practice of RLHF for LLMs \cite{ouyang_training_2022} may be thought of as similar to optimizing final reward for similar reasons to the previous paragraph. Consider annotators as providing comparisons based on their retrospective reward-evaluations for each LLM response they're presented with: i.e. they first assign each response $i$ a score according to the cognitive state that results from reading such response, $R^{(i)}_{\theta_H}$, and then pick the response with highest cumulative reward according to its respective final reward function. This is somewhat stretching the interpretation of preference changes, as it models the LLM's capacity to influence the user's preference (for the current response relative to others) with every additional token generated. And as in the case above, this also requires assuming that each annotator's cognitive state is ``reset'' to a common $\theta_0$ when starting to evaluate each output (potentially based on the annotation guidelines they have been provided).%

\subsection{Initial Reward}\label{app:initial-reward}

\prg{TI-Unaware Reward Modeling.} Consider Algorithm 5 from \citet{everitt_reward_2021}: note that it essentially encodes the initial reward DR-MDP objective. This is one of the approaches presented by \citet{everitt_reward_2021} to avoid influence incentives. As discussed in \Cref{app:CIDs} and \Cref{sec:solutions}, although this algorithm (and DR-MDP objective) avoid ``direct'' influence incentives, it can still lead to influence incentives as defined in \Cref{def:rew-influence-incentives}.

\prg{Long-horizon RL for LLMs ($t=$ conversation turn, uninfluenceable $\theta$ at reward learning time or mismatch in reward learning/deployment horizon).} 
If one were to take a longer-horizon perspective than in other subsections, and model the person's preferences as not being meaningfully changed during the reward learning process, this seems similar to learning $R_{\theta_0}$. This may be a reasonable modeling choice if one is using a form of reward learning such as GATE \cite{li_eliciting_2023} to obtain a reward model, which captures what the user initially wants (as in the example of Charlie from \Cref{subsec:initial-reward}).
If such a reward model is then used to optimize cumulative reward over long-horizon simulated interactions, this would be equivalent to optimizing $U_\text{IR}$. 
The optimal policy under $U_\text{IR}$ should then influence the user optimally (with actions personalized according to the preferences $\theta_0$) towards whatever cognitive states are most conducive towards long-term reward under $\theta_0$.

\prg{Preferences Implicit in the State of the World.} The approach proposed by \citet{shah_preferences_2019}, based on inferring the human's preferences based on the initial state of the environment, can be thought of as learning the $R_{\theta_0}$ reward model (which requires additional simplifications, as preferences in the initial state of the world may come from different timesteps, and thus conflict). As in the case above, 
optimizing cumulative reward with this kind of reward model seems similar to optimizing the initial-reward objective $U_\text{IR}$. %

\prg{Inverse Reinforcement Learning.} Inverse Reinforcement Learning (IRL) techniques \cite{russell_learning_1998,ng_algorithms_2000,abbeel_apprenticeship_2004,ziebart_modeling_2010} can also be thought of as roughly similar to the initial reward objective if assuming that people's cognitive states wouldn't be affected by providing demonstrations (i.e., the demonstrations would be provided consistent with the person's initial reward function $\theta_0$). One then optimizes the recovered reward function over a horizon $H$ with RL, leading to the $U_\text{IR}$ objective. %

\subsection{Natural Shifts Reward} 

\prg{``Natural Distribution'' from \citet{farquhar_path-specific_2022}}. The idea of natural shifts, and evaluating reward from its perspective, is also present in  \citeauthor{farquhar_path-specific_2022}, specifically in Equation 5. 

\prg{``Natural Shifts'' from \citet{carroll_estimating_2022}}. The correspondence between the ``natural shifts'' objective from \citeauthor{carroll_estimating_2022} and the the ``natural reward'' objective from \Cref{tab:maximization_objectives} is simple, as the objective is written and discussed in similar terms. The main differences are that they treat $\theta$ strictly as preferences, rather than cognitive states, and do not use the formalism of DR-MDPs.

\subsection{Myopic Reward}

\prg{Myopic recommender systems.} Despite a recent push towards using RL for training recommender systems \cite{afsar_reinforcement_2021}, most currently deployed recommender systems optimize engagement (and other metrics) only myopically \cite{thorburn_how_2022}.\footnote{However, as discussed in \Cref{app:myopic-or-not}, in some cases myopic optimization may be implicitly equivalent to long-horizon RL.}
While these metrics are not usually formalized in terms of reward, one can consider the probability of engagement
as a reward signal, as discussed in \Cref{app:real-time-rew-correspondence}. As the engagement signals depend on the user's current reward function (e.g. their preferences), the optimization objective is equivalent to the myopic reward objective from \Cref{tab:maximization_objectives}.

\prg{Standard RLHF used for LLMs ($t=$ conversation turn).} 
Consider the standard RLHF procedure considering each timestep as being equivalent to an AI conversation turn (as in one of the cases in \Cref{app:final-reward}). Viewed this way, standard RLHF is simply optimizing the reward over a single action choice, and can be viewed as a bandit setting \cite{ahmadian_back_2024}. In light of this, performing reward learning in this setting seems similar to obtaining the reward model $R_{\theta_0}$ (or $R_{\theta_1}$, depending how one models the feedback). 
When performing the standard RL component of RLHF, generally one only optimizes the next response's reward (rather than the multi-turn conversation reward). Viewing each AI response as a single action, this makes this training setup similar to the myopic reward objective. %
Using a system trained this way to generate multiple responses to user queries is equivalent to replanning with planning horizon of 1 (see \Cref{app:replanning}).

\subsection{Privileged Reward}\label{app:privileged-reward}

\prg{Methods for removing cognitive biases from reward inference.} In \Cref{tab:intuition_strengths_weaknesses} we included \citet{evans_learning_2015} as an example of reward inference work which tries to infer the ``true preferences'' of the human, in light of their feedback (and cognitive states $\theta$) appearing inconsistent. Other work of this kind could include \citet{shah_feasibility_2019}, if one were to view all preference changes as a form of bias. Most reward learning techniques have a component of this objective, in that they try to debias and denoise human feedback by generally making a Boltzmann Rationality assumption \cite{jeonRewardrationalImplicitChoice2020}.

\prg{Ideal observer theory \cite{firth_ethical_1952} and coherent extrapolated volition \cite{yudkowsky_coherent_2004}.} While they are not a practical approaches, these (highly related) proposals of what alignment should look like in spirit are clearly in line with the privileged reward objective. However, the privileged reward that either of these views are referring to are clearly not ``reachable'' in any meaningful sense, as they correspond to perspectives of practically unrealizable ``idealized'' agents---in our framework, they can best thought of as an ``ideal cognitive state''. Therefore, connecting our framework more explicitly to these perspectives would require extending it to handling unreachable reward functions, as discussed in \Cref{app:reachable}.

\section{Additional Related Work from Philosophy, Economics, and AI}\label{app:related-work-phil-econ}

\subsection{Philosophy}

\prg{Early work on personal identity over time and implications for rational decision-making.} While the nature of personal identity under changing selves has been discussed for centuries \cite{locke_essay_1689}, to the best of our knowledge \citet{parfit_personal_1982} was one of the first to discuss its implications with regards to our conception of rationality. In particular, Parfit describes the challenge with being \textit{timelessly} concerned with evaluating one's life as a whole, without considering time---he rejects this as a practical possibility because of the reality that one inhabits time, and every evaluation comes from the perspective of a particular time. With DR-MDPs, one could say we are trying to explore different approaches for an AI assistant to be timelessly concerned in this way for the wellbeing of a person: while the ultimate goal of finding a ``correct'' timeless evaluation may be futile, even settling on reasonable evaluations is challenging. The arguments from \citet{parfit_personal_1982} were further expanded in the seminal work ``Reasons and Persons'' \cite{parfit_reasons_1984}. Of particular note is Parfit's ``present aim theory'', according to which an individual's rational actions should be guided by their present aims or goals, without giving special weight to their future aims or goals (which seems broadly equivalent to our initial-reward objective $U_\text{IR}(\xi)$),\footnote{Potentially when used under re-planning, which is discussed in \Cref{app:replanning}.} and Parfit's ``self-interest theory'', which holds that a person should make decisions based on what will be best for them in the long run, even if it conflicts with their present desires or goals (the corresponding objective in our framework seems less clear, as ``what is best in the long run'' is not specific enough).

\prg{Welfare and rationality under changing preferences.} Around the same time, \citet{griffin_well-being_1986} also criticizes the ``totting-up model'' of well-bring, according to which we should simply sum well-being over time. 
There have been many philosophical works focused on the topic of personal welfare in the context of changing preferences: \citet{velleman_well-being_1991} considers the relation between the welfare value of a temporal period in someone’s life and his welfare at individual moments during that period. %
\citet{rosati_story_2013} describes the ``narrative thesis'', which posits that the way we think of the storyline of our life contributes (in it's own right) to our well being, and echoes related psychological theories \cite{bauer_narrative_2008}. 
\citet{bykvist_prudence_2006} states that, for judging inter-temporal decisions, one shouldn't simply look at a single timestep's point of view---we should consider the potential people we could become, and how \textit{they} would evaluate the worlds they are in. Note that this assumes that we trust the assessments and point of view of future selves, which is questionable if we are worried about undue influence. %
Indeed, \citet{ullmann-margalit_big_2006} questions the idea that one could possibly be rational about ``big decisions'' which change the self, claiming that there is no footing for a rational choice in these situations, as the ``rationality base'' changes as a consequence of the decision.
Building on \citeauthor{bykvist_prudence_2006} and \citeauthor{ullmann-margalit_big_2006}, \citet{paul_transformative_2014} and \citet{callard_aspiration_2018} argue that there is no rational basis for making decisions that change the self.
\citet{pettigrew_choosing_2019} expands on this line of work, and proposes a theory of individual decision making under changing selves based on taking a weighted average between the utility functions of different selves, challenging the notion of ``impossibility of grounding rationality'' under changing selves. While Paul did not find it convincing \cite{paul_choosing_2022}, we think the framework proposed by \citeauthor{pettigrew_choosing_2019} makes significant steps forward. That being said, challenges remain with regards to the details of how weights would be chosen in practice for multi-step decision making: in particular, if weights are re-assessed at every decision making node, it seems even reasonable choices of weights could lead to inconsistent plans. Also, it seems unclear why one could expect a clean separation between one's current weights for different selves, and their current utility function: for instance, wouldn't the utility function of Bob from \Cref{fig:influencev2} already include information about how much Bob would want to weigh being a happy conspiracy theorist in the future, as is (implicitly) the case in our example?\footnote{This is related to discussions of how meta-preferences may be easily expressible in our framework if the state were to include enough information to recover $\theta$ from it, as discussed in \Cref{app:reward-modeling}.} The main difference between \citet{pettigrew_choosing_2019} and our work is that while \citeauthor{pettigrew_choosing_2019} is squarely focused on human decision making, DR-MDPs can instead be thought of as focusing on choosing objectives that lead to reasonable (and consistent) AI plans, which implicitly account for all relevant selves. Another significant difference is our greater focus on the role of influence, and its implications for the legitimacy of the resulting preferences. On first impression, \citeauthor{pettigrew_choosing_2019}'s framework may seem more flexible, as it accounts for uncertainty in future weights, while DR-MDPs assume that the dynamics of reward functions are known. However, DR-MDPs can still account for uncertainty in future outcomes (and reward functions) using stochastic transitions (even though we don't do this in our examples).
For a recent comprehensive review and synthesis of philosophical positions on the problems of changing selves, including a more in-depth summary of Pettigrew's framework, see \citet{strohmaier_preference_2024}.

\prg{Ethics of influence and of nudging.} 
There is a lot of philosophical work on the ethics of influence and manipulation \cite{noggle_ethics_2020}. A specific area closely connected to our issue involves the ethical considerations of ``nudging''. This concept emerged from behavioral economics and refers to the efforts of institutions to steer the behavior and decision-making of groups towards certain outcomes \cite{thaler_nudge_2008}. Our problem setting is very closely related: in our case, it is an AI system (rather than an institution) that is deciding whether to nudge a user. Perhaps unsurprisingly, the ``optimality'' and ``acceptability'' of nudging in the literature are similarly unclear: while nudging was originally promoted as a tool to encourage pro-social outcomes, when it is ethical or overly paternalistic has often been contested \cite{hansen_nudge_2013,hausman_debate_2009,thaler_nudge_2018}.
There are various philosophical works which propose frameworks for external decision-makers to assessing the ethics of whether to perform a specific nudge \cite{paul_as_2019,pettigrew_nudging_2022}. In particular, \citet{paul_as_2019} claim that a nudge is legitimate if the nudged person is better off, \textit{as judged by themselves} after the nudge. \citet{pettigrew_nudging_2022} points out that this heuristic can be misleading, in the case that the nudge was illegitimate (e.g. if it manipulates the person to have different preferences), and proposes a stronger condition as heuristic: that people agree, before and after the nudge, that the nudge was beneficial.\footnote{However, there is additional nuance here, in that \citet{pettigrew_nudging_2022} only requires people's \textit{global} utilities to judge the nudge positively (before and after the nudge). Indeed, \citeauthor{pettigrew_nudging_2022} distinguishes between \textit{local} and \textit{global} utilities: the former encode one's values at a point in time time, and the latter---which assign some weight to the local utilities at other times---are what should be used to make decisions.} Note that the property of Unambiguous Desirability proposed in \Cref{sec:desirable-influence-individual-rationality} can be thought of as a generalization of the heuristic proposed by \citet{pettigrew_nudging_2022}, for arbitrary multi-timestep nudges in the form of AI policies. 

\prg{Work at the intersection of AI and philosophy.} Discussions about what forms of influence are appropriate have also taken place at the intersection of philosophy and artificial intelligence: there have been works on value changes \cite{ammann_problem_2024}, preference change \cite{kolodny_ai_2022,zhi-xuan_beyond_2024}, influence \cite{benn_whats_2022,malvone_shape_2023,gabriel_ethics_2024}, and manipulation \cite{benn_whats_2022,carroll_characterizing_2023,gabriel_ethics_2024}. The concept of ``informed preferences'' \cite{gabriel_artificial_2020}, ideal observer theory \cite{firth_ethical_1952}, and Coherent Extrapolated Volition (CEV) \cite{yudkowsky_coherent_2004} also relate to our work, as we discussed in relation to privileged reward \Cref{app:privileged-reward} (and other appendix discussions, such as that of \Cref{app:justifying-reward-values}). Various works have also tackled the problem of whether it is possible to attribute intent to AI systems actions, and how one could do so \cite{ward_reasons_2024,halpern_towards_2018}. System intent and incentives for influence are closely related \cite{carroll_characterizing_2023}, and intent has already been studied as a coordination mechanism across time for individual decision-making under changing preferences \cite{bratman_intention_1987}.

\subsection{Economics}

\prg{Early work on welfare under variable tastes.} 
While others have alluded to changes in tastes and its implications for welfare \cite{samuelson_note_1937,hayek_pure_1941}, to our knowledge \citet{harsanyi_welfare_1953} was the first economist that investigated these issues in more depth. In this work, Harsanyi considers a notion of welfare grounded in what people say they prefer rather than in external normative principles such as social welfare. Note that this is a normative stance in its own right, and is similar to how our notions of rewards are grounded in what people say they want. Moreover, unlike our framework, Harsanyi's notion of welfare requires assuming comparability between the utility judgements of a person before and after their preference change. A similar move is taken by \citet{von_weizsacker_notes_1971}, who builds on works such as that of \citet{peston_chapter_1967}.
\citet{elster_ulysses_1979} discusses the limitations of von Weizs{\"a}cker's assumptions in more depth, emphasizing the challenges regarding paternalism and which perspectives should be considered to ground welfare judgements when preferences change---which are central to our work. In ``Sour Grapes'' \cite{elster_sour_1983}, Elster later also discusses in depth a specific class of preference changes which occur unintentionally to oneself, namely ``adaptive preferences''. In the same work, Elster also briefly discusses the limitations of state-dependent preference formulations (relevant to \Cref{app:context-vs-pov}).

\prg{Explaining away preference changes as time-inconsistent discounting.} Shortly after Harsanyi's work, \citet{strotz_myopia_1955} studied the phenomenon of time-inconsistency in human's behaviors. Why do we not save up for retirement enough, and then regret it? Strotz showed that only exponential discounting leads to consistent (re-)planning, therefore people must implicitly not be using that kind of discounting (as they exhibit time-inconsistent behavior). He then discusses two ways that people can account for their time-inconsistency: via commitment devices, or by only considering plans that they would actually be able to follow through on without inconsistency.\footnote{This analysis was later corrected by \citet{pollak_consistent_1968}.} %
Note that this stance is implicitly still assuming that people's underlying preferences are static, but just that their discounting scheme is such that they would exhibit time-inconsistent behaviors nonetheless. This move of ``explaining away the appearance of preference changes'' by casting them as suboptimal discounting with fixed preferences ended up influencing much of the later work on the topic, as reviewed by \cite{loewenstein_time_2003}.
Even though the model of hyperbolic discounting may have sufficient explanatory power of people's decision-making in many settings that economics is interested in \cite{benzion_discount_1989,chabris_individual_2008}, this is not the case more broadly \cite{loewenstein_anomalies_1992,frederick_time_2002,loewenstein_time_2003}.

\prg{Shying away from modeling changing preferences.} Why did early economics works---with few exceptions---avoid analyzing \textit{changing} preferences? \citet{george_preference_2001} and \citet{grune-yanoff_preference_2009} give various reasons: preference creation and change have historically been considered topics that lay outside the scope of economics;\footnote{On this point, see also \citet{von_weizsacker_notes_1971} and \citet{pollak_endogenous_1978}.} macroeconomists believed that institutional change (relative to changes in individual's preferences), is by far the more important explanatory factor of economic growth; and maybe most importantly, many microeconomists were of the conviction that human preferences ultimately do not change, and even if they did, it was mathematically counterproductive to model such changes.
The most impactful work from this last camp was undoubtedly that of \citet{stigler_gustibus_1977}. They go as far as to say that ``no significant behavior has been illuminated by assumptions of differences in tastes'', and that analyses considering changing tastes ``give the appearance of considered judgement, yet really have only been ad hoc arguments that disguise analytical failures''.
\citet{grune-yanoff_preference_2009} interpret their position as follows:
\begin{quote}
    ``This position may be interpreted either as the ontological claim that preferences indeed are stable, or alternatively as the methodological claim that explanations based on stable preferences are better than those that refer to preference changes. The second interpretation can be based on the assumed relation between explanatory power and simplicity: explaining any conceivable human behaviour through the paradigm of individuals maximizing utility constrained by income and present capital stocks is simpler than supposing that tastes change.''
\end{quote}
While we agree with the risk of introducing unnecessary formal complexity, we think that in the context of AI interactions with humans, influence effects are too important to be ignored.\footnote{For further criticism of \citet{stigler_gustibus_1977}'s position, see \citet{pollak_endogenous_1978}.} And as we show in our analysis, ignoring such effects has a cost---that they will likely be optimal under standard objective functions (or notions of welfare, to use the language of economists). %

\prg{Egonomics, weakness of will, and adaptive preferences.} A notable line of work in economics is that of weakness of will and self control problems, introduced by \citet{schelling_egonomics_1978} with his notion of ``egonomics''. This area investigates questions of self control and preference change, in particular with respect to internal conflict and addictive behaviors such as smoking. \citet{schelling_enforcing_1985} discusses the enforcement of rules for oneself, which is relevant to our discussion: people are not always able to commit to the plans that they agree are good for them, often leading them to turn to external accountability partners or enforcers. They play a similar role to the AI systems in our work, as they aggregate across the wishes of the person at different times, and try to encourage their best guess of the behavior that the person would ultimately want for themselves.\footnote{The usage of ``self'' in our work comes with similar caveats to those described by \citet{schelling_self-command_1984}, although we recognize that this usage is vague and arguably misleading---for some criticisms, see how the usage by Schelling was criticized by \citet{elster_weakness_1985}.}
Our work sidesteps most of these issues around consistency of plans and self-control \cite{pollak_consistent_1968} by considering an AI assistant's actions, which unlike humans, can credibly commit to carry out a plan (assuming the person cannot switch it off).
Importantly, the settings considered by these works are ones in which the normatively desirable course of action is clear, which removes the main complications of dealing with changing preferences: the notion of what rational behavior should consist of remains mostly unthreatened.
Concurrently, \citet{elster_weakness_1985} interprets weakness of will as a collective action problem between the different selves. For a survey of other works in the literature of weakness of will, see \citet{ainslie_specious_1975}.

\prg{Recent economics work has started contending with changing preferences more directly.}
In recent decades, there have been many more works on the topic of changing preferences \cite{loewenstein_time_2003,grune-yanoff_preference_2009}. In particular, \citet{loewenstein_time_2003} identifies several sources of preference change: habit formation, satiation, visceral factors, maturation, conditioning, social influences, and motivated taste change. \citet{george_preference_2001} formulates an theory of individual welfare that can account for changing preferences by appealing to second-order preferences. To address the regress problem, this work argues that preference changes are most commonly first-order ones, and even when second-order preferences occur, as long as they don't move in tandem with first-order changes, welfare assessments are still possible. 
More recently, \citet{bernheim_theory_2019} attempt to model and unify various preference change phenomena under a single descriptive theoretical model, according to which individuals choose their preferences according to what they expect will maximize their utility (subject to their level of ``open-mindedness'').
The work from \citet{dietrich_where_2013} is also of note: they also propose a descriptive model of preference change, based on ``motivationally salient'' properties of alternatives available to the agent changing. However, both of these descriptive accounts have yet to be tested (to our knowledge), and it's unclear what the normative implications for decision making of this model should be---which are the focus of our paper.

\prg{Unambiguous Desirability and Individual Rationality.} The property of unambiguous desirability was inspired by the notion of ``individual rationality'' from algorithmic game theory \cite{nisan_algorithmic_2007}, which captures the notion of whether any of the individuals involved in an ongoing deal would ever prefer to defect. This is also known as a ``participation constraint'' or ``voluntary participation'' \cite{jackson_mechanism_2014}.

\subsection{AI}\label{appsubsec:AI}

\prg{Multi-objective MDPs.} With a choice of $U(\xi)$, one implicitly replaces the multiple competing notions of optimality (corresponding to each $\theta$) with a single one.
The process of choosing a single $U(\xi)$ which implicitly reduces a DR-MDP to an MDP, is similar to the \textit{scalarization} step in Multi-Objective MDPs \citep{roijers_survey_2013} which reduces a MOMDP to an MDP, which similarly requires an implicit \textit{value judgement} \citep{chankong_multiobjective_2008}. 
However, DR-MDPs importantly differ from MOMDPs, in that the objectives which should be used to evaluate a trajectory may depend on the trajectory itself (as the actions taken can affect the selves that are realized). Relatedly, MOMDPs don't keep track of which reward function was associated with each step (as it's meaningless in their framing).

\prg{Interdisciplinary AI work.} The adaptive and changing nature of human feedback has also been emphasized by \citet{lindner_humans_2022}. We think there a good area of inspiration for tentative solutions is that of Fiduciary AI \cite{benthall_designing_2023}. More broadly, our conclusion about the challenges of avoiding normative choices when operationalizing alignment rings similar to the points made by \citet{dobbe_hard_2021} and \citet{kirk_empty_2023}.

\prg{Welfare under conflicting preferences.} One of the works which is most related to ours is \citet{kleinberg_challenge_2022}: in the setting of recommender systems, they also study what should be conceived of as user welfare, focusing on preference conflict. In particular, they model users as having ``system 1'' and ``system 2'' preferences that can be in conflict, and show how only optimizing engagement will often be insufficient to guarantee welfare. Our work can be thought of as extending theirs to considering many possible preferences that can change over time, thus modeling change in addition to conflict. Moreover, we don't assume access to any ultimate notion of welfare (which in their case is the judgement of system 2).

\prg{Algorithmic amplification in social media, and $\pi_\text{noop}$.} The study of algorithmic amplification in social media \citep{thorburn_what_2022,ribeiro_amplification_2023,huszar_algorithmic_2021} can be thought of as a study of influence emerging from specific algorithmic choices. %
Notions of amplification also need to be specified relative to a ``neutral'' baseline, similarly to our notions of influence (\Cref{def:rew-influence}): what baseline is most appropriate has been subject to debate \cite{laufer_algorithmic_2023,thorburn_making_2023}. Maybe the most common choice in practice has been comparing to a reverse chronological recommender \cite{huszar_algorithmic_2021,milli_engagement_2023}, but others include comparing to others recommender systems \cite{fast_unveiling_2023}, no platform at all \cite{allcott_welfare_2020}, or random recommendations \cite{carroll_estimating_2022}.

\prg{Performative prediction, performative power, and preference influence.} An interesting line of work which has emerged in recent years is that of performative prediction \cite{perdomo_performative_2020} and performative power \cite{hardt_performative_2022}, which concerns itself with the capacity of classifiers to affect the distribution of their future inputs. This idea is similar in spirit to the work of \citet{krueger_hidden_2020}, and is connected to our concern with AI systems' capacity to influence humans. However, we see various reasons to prefer the RL formalism to that of \citet{hardt_performative_2022} for the types of influences were are interested in: performative prediction and power are mostly focused on firms which operate in sequential decision problems (e.g. domains in which the algorithm's choices affect future users' behavior), but use algorithms that myopically optimize over only the next timestep's outcomes---from this perspective, they only allow to model 1 of the 8 objectives we consider in \Cref{tab:maximization_objectives}. For instance, to the best of our understanding, performative power \cite{hardt_performative_2022} can be thought of as a measure of how much a firm can shift users over the course of \textit{a single timestep}, if they choose to do so. The steering analysis of ex-ante and ex-post optimization only performs a one-timestep lookahead, feels like a less natural formalism for the multi-timestep nature of most preference changes---especially if one considers that the RL formalism solves the multi-timestep generalization of the ex-post optimization problem by design: in RL training, the human's adaptation to the AI is already factored into how the AI should be making decisions in order to maximize the multi-timestep objectives. In short, the lens of RL seems strictly more expressive and more suited to our purposes than that of performative prediction, but comes at the cost of additional computational challenges. As a final point of comparison, the framing of \citet{hardt_performative_2022} is mostly focused on the misalignment between firms and targets of the firm's algorithms---focusing on the power that firms have to steer them to their benefit (under well-defined metrics). Instead, we can be thought of as focusing on the problem of ``steering oneself'' with the help of an AI system, and the challenges which remain in determining an appropriate metric for success, as discussed in relation to ``egonomics'' \cite{schelling_egonomics_1978}. Essentially, while we recognize that firm-user misalignment is a very significant reason for worry, we focus on the challenges that would remain even if AI systems were to be developed solely with user alignment in mind. %

\prg{Social choice theory.} Our setting shares various similarities with preference aggregations across multiple individuals, as mentioned in \Cref{sec:intro}. The problem of social aggregation is studied by social choice theory \cite{list_social_2022,brandt_computational_2012}, which mainly differs from our framework in that there is no temporal dependency between elements of the decision which is being made. There have also been recent works discussing the parallels between social choice and AI Alignment \cite{mishra_ai_2023,conitzer_social_2024}, in particular with respect to RLHF. While there has been some work focusing on collective decision-making across time (for individuals who can change their preferences), these works mostly ignore the influence incentives which emerge from their notions of optimality \cite{parkes_dynamic_2013,freeman_fair_2017,kulkarni_social_2020}. Indeed, to our understanding, the objectives proposed by these works could lead to undesirable influence incentives: for example, in \citet{parkes_dynamic_2013} the optimal social choice function (i.e. policy for the social choice MDP) may initially take actions dispreferred by everyone if this can influence future preferences to be in accordance with one another (guaranteeing future unanimity). 
An interesting direction for future work would be to consider traditional social welfare functions (e.g. Rawlsian, Nash, Utilitarian, etc.) as objectives for the aggregations across the different selves. However, in some preliminary investigations in this direction, it seemed to us that such approach would not necessarily achieve more desirable influence properties than the objectives we considered in \Cref{tab:maximization_objectives}.

\prg{Reward uncertainty in MDPs.} There have been some works accounting for uncertainty in reward functions: most relevant is the work by \citet{regan_robust_2010}, who consider Imprecise Reward MDPs, in which there is a feasible set of reward functions that cannot be disambiguated by, e.g., acting in the environment. For such setting, they define the optimization objective in terms of minimax regret, that is, they aim to minimize the maximum regret incurred if the worst reward function from the feasible set were to be chosen. We could have also included this criterion in our analysis and in \Cref{tab:maximization_objectives}, but we are confident it would also run into issues---minimax regret in this setting would require comparing regret across different possible reward functions, requiring ``interpersonal'' comparisons that could always be set to be in favor of undesirable manipulation incentives.
Other works also account for reward uncertainty, but assume that there is one true Markovian reward function that one can learn about, and one should use to evaluate trajectories \citep{hadfield-menell_cooperative_2016,desai_uncertain_2017,chan_assistive_2019}. A gap in the literature is work which tries to disambiguate between reward changes, and updates in belief regarding reward functions (which we discuss briefly in \Cref{app:should-visceral-factors-be-in-reward?}).

\end{document}